\documentclass[twoside,11pt]{article}

\usepackage{thm-restate,enumitem}
\usepackage{jmlr2e_arxiv}

\usepackage[disable]{todonotes}

\usepackage{wrapfig}
\usepackage{hyperref}
\usepackage{url}
\usepackage{mathrsfs}
\usepackage[mathletters]{ucs} % allows some more UniCode math symbols $ω=∇×u$
\usepackage[utf8x]{inputenc} % interprets many UniCode characters correctly $∃\alpha∀β:\alpha→β,��∈ℝ,∫_ε^θ≤∞,∑_γ^δ,����←λ��,x∈A⊂ℝ$,↔,⇔
\usepackage{booktabs}
\usepackage{amsmath,amssymb}
\usepackage{cleveref}

\usepackage[ruled,vlined]{algorithm2e}
\usepackage{xspace}
\usepackage{pgfplots}

\DeclareMathOperator{\Avg}{Avg}

\def\arrvline{\hfil\kern\arraycolsep\vline\kern-\arraycolsep\hfilneg}
\def\fr#1#2{{\textstyle\frac{#1}{#2}}} % textstyle fraction

\newcommand{\cC}{\mathcal{C}}
\newcommand{\cD}{\mathcal{D}}
\newcommand{\cE}{\mathcal{E}}
\newcommand{\cF}{\mathcal{F}}
\newcommand{\cG}{\mathcal{G}}
\newcommand{\cL}{\mathcal{L}}
\newcommand{\cN}{\mathcal{N}}
\newcommand{\cP}{\mathcal{P}}
\newcommand{\cU}{\mathcal{U}}
\newcommand{\cX}{\mathcal{X}}
\newcommand{\cY}{\mathcal{Y}}
\newcommand{\cZ}{\mathcal{Z}}
\newcommand{\cW}{\mathcal{W}}
\newcommand{\cH}{\mathcal{H}}

\newcommand{\E}{\mathbb{E}}
\newcommand{\R}{\mathbb{R}}

\newcommand{\bI}{\mathbb{I}}
\newcommand{\Var}{\mathrm{Var}}

\DeclareMathOperator*{\argmin}{arg\,min}

\newcommand{\tx}{\tilde{x}}

\newcommand{\tS}{\tilde{S}}

\newcommand{\tcG}{\tilde{\cG}}
\newcommand{\tP}{\tilde{P}}
\newcommand{\tcP}{\tilde{\cP}}

\newcommand{\tig}{\tilde{g}}
\newcommand{\tih}{\tilde{h}}
\newcommand{\hnnS}[1]{h_{f,#1}^{\text{\tiny NCC}}}
\newcommand{\hnn}{\hnnS{S}}

\newcommand{\hS}{\hat{S}}
\newcommand{\hU}{\hat{U}}
\newcommand{\hcU}{\hat{\mathcal{U}}}

\newcommand{\bz}{{\bf z}}

\newcommand{\threepartdef}[6]
{
	\left\{
		\begin{array}{lll}
			#1 & \mbox{if } #2 \\
			#3 & \mbox{if } #4 \\
			#5 & \mbox{if } #6
		\end{array}
	\right.
}

\jmlrheading{1}{2023}{1-48}{4/00}{10/00}{galanto23a}{Tomer Galanti, Andr\'as Gy\"orgy, and Marcus Hutter}

\ShortHeadings{}{Galanti, Gy\"orgy and Hutter}
\firstpageno{1}

\begin{document}

\title{Generalization Bounds for Few-Shot Transfer Learning with Pretrained Classifiers}

\author{\name Tomer Galanti \email galanti@mit.edu \\
       \addr McGovern Institute \\
       Massachusetts Institute of Technology\\
       Cambridge, MA, USA
       \AND
       \name Andr\'as Gy\"orgy \email agyorgy@google.com \\
       \addr Google DeepMind \\
       London, UK
       \AND
       \name Marcus Hutter \email mhutter@google.com \\
       \addr Google DeepMind \\
       London, UK}

\editor{}

\maketitle

\begin{abstract}
We study the ability of foundation models to learn representations for classification that are transferable to new, unseen classes. Recent results in the literature show that representations learned by a single classifier over many classes are competitive on few-shot learning problems with representations learned by special-purpose algorithms designed for such problems. We offer a theoretical explanation for this behavior based on the recently discovered phenomenon of class-feature-variability collapse, that is, that during the training of deep classification networks the feature embeddings of samples belonging to the same class tend to concentrate around their class means. More specifically, we show that the few-shot error of the learned feature map on new classes (defined as the classification error of the nearest class-center classifier using centers learned from a small number of random samples from each new class) is small in case of class-feature-variability collapse, under the assumption that the classes are selected independently from a fixed distribution. This suggests that foundation models can provide feature maps that are transferable to new downstream tasks, even with very few samples; to our knowledge, this is the first performance bound for transfer-learning that is non-vacuous in the few-shot setting. 
\end{abstract}


%%%%%%%%%%%%%%%%%%%%%%%%%%%%%%%%%%%%%%%%%%%%%%%%%%%%%%%%%%%%%%%
\section{Introduction}
%%%%%%%%%%%%%%%%%%%%%%%%%%%%%%%%%%%%%%%%%%%%%%%%%%%%%%%%%%%%%%%

In a variety of machine learning applications, we have access to a limited amount of data from the task that we would like to solve, as labeled data is oftentimes scarce and/or expensive. In such scenarios, training directly on the available data is unlikely to produce a hypothesis that performs well on new, unseen test samples. A prominent solution to this problem is to apply transfer learning (see, e.g., \citealp{NIPS1994_0f840be9,pmlr-v27-bengio12a,NIPS2014_375c7134}). In transfer learning, we are typically given a large-scale source task (e.g., ImageNet ILSVRC, \citealp{ILSVRC15}) and a target task from which we encounter only a limited amount of data. While there are multiple approaches to transfer knowledge between tasks, a popular approach suggests to train a large neural network on a source classification task with a wide range of classes (such as ResNet50, \citealp{7780459}, MobileNet, \citealp{howard2017mobilenets} or the VGG network, \citealp{DBLP:journals/corr/SimonyanZ14a}), and then to train a relatively smaller network (e.g., a linear classifier or a shallow MLP) on top of the penultimate (also known as feature) layer of the pretrained model, using the data available in the target task.

Due to the effectiveness of this approach, transfer learning has become a central element in the machine learning toolbox. For instance, using pretrained feature maps is common practice in a variety of applications, including fine-grained classification \citep{Chen2019ThisLL,huang2020interpretable,yang2018learning}, object detection, \citep{Redmon_2016_CVPR,NIPS2015_14bfa6bb,He_2017_ICCV}, semantic segmentation \citep{7298965,conf/eccv/ChenZPSA18}, or medical imaging \citep{DBLP:journals/corr/abs-1904-00625}. 
In fact, due to the cumulative success of transfer learning, large pretrained models that can be effectively adapted to a wide variety of tasks~\citep{NEURIPS2020_1457c0d6,DBLP:journals/corr/abs-2102-12092} have recently been characterized as foundation models~\citep{DBLP:journals/corr/abs-2108-07258}, emphasizing their central role in solving various learning tasks.

Foundation models are typically pretrained on a source task at a time when the concrete details of the target task(s) are not -- or only partially -- available to the practitioner. Therefore, the ability to precondition or design the training regime of the foundation model to match the target task is limited. As a result, there may be a mismatch between the source and target tasks, such as a different number of classes or a different number of available samples. This suggests that a foundation model should be generic and applicable to a wide range of tasks.

On the other hand, when some specifics of the target tasks are known, often special-purpose algorithms are designed to utilize this information. Such an example is the problem of few-shot learning, when it is known in advance that the target problems come with a very small training set \citep{NIPS2016_90e13578,Ravi2017OptimizationAA,pmlr-v70-finn17a,lee2019meta}. While these specialized algorithms have significantly improved the state of the art, the recent work of \citet{DBLP:conf/eccv/TianWKTI20,Dhillon2020A} demonstrated that predictors trained on top of foundation models can also achieve state-of-the-art performance on few-shot learning benchmarks.

\paragraph{Contributions.} 

In previous work~\citep{galanti2022on}, we provided an explanation for the success of transfer learning between classification problems based on the phenomenon of neural collapse~\citep{Papyan24652}. Informally, neural collapse (see Section~\ref{sec:nc}) describes the training dynamics of deep networks for standard classification tasks, where multiple structural properties arise. A central property, known as class-features variability collapse, asserts that the features (the output of the penultimate layer of the deep network) associated with training samples belonging to the same class tend to concentrate around their class feature mean. We argued that this property generalizes to new data points and new classes (e.g., the target classes) when the model is trained on several classes with many samples each. Additionally, we showed that in the presence of neural collapse, training a linear classifier on top of the learned feature map can generalize well even with few samples. However, our analysis relied on assumptions about the feature maps that may be difficult to justify in practical situations.

In this paper we provide a stronger theoretical analysis for this problem without relying on this kind of assumptions. Instead of claiming that feature-variability collapse generalizes to new samples and classes, we argue that weaker notions of clustering generalize to new samples and classes. Specifically, we study the few-shot error of the learned feature map, which is the classification error of nearest class-center classifiers on top of the feature map with centers learned using a small number of random samples. We show that the few-shot error generalizes from training to test data (Lemma~\ref{prop:marginBoundA}), and upper bound the expected few-shot error between new classes using the average few-shot error between the source classes (Lemma~\ref{prop:marginBoundB}). Additionally, we show that the few-shot error on the training data can be upper bounded using the degree of feature-variability collapse (Lemmas~\ref{prop:cdnv_error}--\ref{prop:cdnv_errorG}). Based on this analysis, we prove that large neural networks that exhibit feature-variability collapse when trained on many classes and samples can be easily adapted to new downstream tasks with very few samples (Theorems~\ref{prop:marginBoundA} and~\ref{prop:marginBoundB}). While our results provide a proof of concept for the ability of pre-trained classifiers to easily adapt to new tasks with very few samples, our bounds are typically loose, akin to traditional generalization bounds.

\subsection{Related Work}

\paragraph{Generalization bounds for transfer and meta-learning.} Several theoretical frameworks have been proposed to study various multi-task learning settings \citep{baxter,10.5555/2946645.3007034,pmlr-v32-pentina14,Galanti2016ATF,NEURIPS2019_f4aa0dd9,du2021fewshot}. However, these frameworks all consider settings in which the learning algorithm is provided with a sequence of similar learning problems (e.g., problems selected independently from the same distribution over tasks), and therefore cannot be used to study the case where only a single source task is available, which is quite different from the target task.

In contrast, we propose a theoretical framework in which the learning algorithm is provided with one classification task and is required to learn a representation of the data that can be adapted to solve new unseen tasks using a limited amount of data. This framework is more closely aligned with the common practice of transfer learning using deep neural networks.

Furthermore, while the aforementioned papers indeed show that having a collection of similar source tasks are useful in learning feature vectors (essentially, as if the number of data points was the number of points in all the source tasks), they require that the final target task have enough samples to learn a classifier in the feature space. This is in stark contrast to our results as our bounds remain meaningful (and are especially derived for this case) when the number of samples in the target class is a small constant.

\paragraph{Feature clustering and transfer learning.} The connection between clustering in the feature space and few-shot transfer learning has been explored empirically by \citet{pmlr-v119-goldblum20a}: They demonstrated on multiple benchmark tasks that the more the feature embeddings of the classes are clustered the better the few-shot performance is. Motivated by this observation, they designed modified training algorithms that directly encourage clusterability, which led to better few-shot transfer performance for both in meta-learning and in classical training of foundation models.
As their notion of feature-clusterability is essentially the same as the class-features variability collapse we consider in this paper (although the mathematical details are somewhat different, with our measure being more local than theirs, see the discussion in Section~\ref{sec:nc}), our results also give a theoretical explanation to their empirical observations. 

\subsection{Notation} Throughout the paper we use the following notation. $\|\cdot\|$ denotes the Euclidean norm for vectors and the spectral norm for matrices. For matrices we also use $\|\cdot\|_F$ to denote the Frobenius norm. For a distribution $Q$ over $\cX \subset \R^d$ and $u: \cX \to \R^p$,  the mean and variance of $u(x)$ for $x\sim Q$ are denoted by $\mu_u(Q) := \E_{x \sim Q}[u(x)]$ and by $\Var_u(Q) := \E_{x \sim Q}[\|u(x)-\mu_u(Q)\|^2]$, where $\E_{x \sim Q}[\cdot]$ denotes the expectation operator according to $x$ with distribution $Q$ (we may omit $Q$ when it is clear from the context). For a finite set $A=\{a_i\}^{n}_{i=1}$, we denote $\Avg^{n}_{i=1}[a_i] := \frac{1}{n} \sum^{n}_{i=1}a_i$ and by $U[A]$ the uniform distribution over $A$. We write $x \sim A$ to denote $x \sim U[A]$ for a finite set $A$, and to simplify notation, in some similar cases we use $A$ in place of $U[A]$ where it does not cause confusion. For a positive integer $k$, we use $[k]:=\{1,\ldots,k\}$; to simplify the notation, we write $U[k]$ in place of $U[[k]]$.
For an event $E$, $\bI[E]$ denotes its indicator function  (equals 1 if the event holds and 0 otherwise). $\mathbf{0}$ denotes the zero-vector of appropriate dimension.
For a class of functions $\cH \subset \{h: \cW \to \cZ\}$ for some sets $\cW$ and $\cZ$ and a (multi-)set $W \subset \cW$, we denote $h(W):=\{h(w): w \in \cW\}$ for any $h \in \cH$ (i.e., $h$ is applied elementwise), and we also define $\cH(W):=\{h(w): w \in W, h \in \cH\}$. We use similar notation when $W$ is a vector. We use $\log$ to denote the natural logarithm.

%%%%%%%%%%%%%%%%%%%%%%%%%%%%%%%%%%%%%%%%%%%%%%%%%%%%%%%%%%%%%%%
\section{Problem Setup}\label{sec:setup}
%%%%%%%%%%%%%%%%%%%%%%%%%%%%%%%%%%%%%%%%%%%%%%%%%%%%%%%%%%%%%%%
We study a transfer learning setting in which a model is pretrained on a source task and later adapted to solve downstream tasks. To model this setup, we assume that the final downstream task is a balanced $k$-class classification problem, called the \emph{target} problem, and the auxiliary task, which is used to train the feature representation, is a balanced $l$-class classification problem, called the \emph{source} problem. Formally, a target task is defined by a distribution $P$ over samples $(x,y) \in \cX \times \cY_k$, where $\cX \subset \R^d$ is the instance space and $\cY_k := [k]$ is the label space. For a pair $(x,y)$ with distribution $P$, we denote by $P_c$ the class-conditional distribution of $x$ given $y=c$ (i.e., $P_c(\boldsymbol{\cdot}) = \Pr[x \in \boldsymbol{\cdot} \mid y=c]$).

Our goal is to learn a classifier $h:\cX\to \cY_k$ that minimizes the target test error
\begin{equation}
L_P(h) ~:=~ \E_{(x,y)\sim P}\big[\bI[h(x)\neq y]\big].
\end{equation}
The algorithm is provided with a labeled training dataset $S := \cup^{k}_{c=1}S_c$, where the sets $S_c:=\{x_{ci}\}^{n}_{i=1}$ are drawn from the class-conditional distributions $P_c$, but the true distributions $P_c$ (and $P$) are unknown. We assume that $P$ is a balanced distribution (i.e., $P[y=c]=1/k$) with class-conditional distributions $\{P_c\}^{k}_{c=1}$ and that the training data $S$ is also balanced over the $k$ classes (as described above).

When the number of training samples $n$ (for each class) is small, it may be difficult to train a classifier on $S$ that achieves good performance on the target task. To address this issue, we aim to learn a classifier of the form $h=g \circ f$, where $f \in \cF \subset \{f':\R^d \to \R^p\}$ is a feature map and $g \in \cG\subset \{g':\R^p \to \R^k\}$ is a classifier used on the feature space $\R^p$. The idea is to learn the feature map $f$ based on a different problem where more data is available, and then train the classifier $g$ to solve the hopefully simpler classification problem of predicting $y$ based on $f(x)$ instead of $x$. That is, $g$ is actually a function of the modified labeled training data $\cup^{k}_{c=1}\{f(x_{ci})\}_{i=1}^{n}$. To emphasize this dependence, we use the notation $g_{f,S}$ for the trained classifier based on the features, and $h_{f,S} = g_{f,S} \circ f$ for the full classifier.

We assume that the source task used to learn the feature map $f$ is a single $l$-class classification problem over the same sample space $\cX$, given by a distribution $\tP$. We are interested in finding a classifier $\tih:\cX \to \R^l$ of the form $\tih=\tig \circ f$, where $\tig \in \tcG \subset \{g':\R^p \to \R^l\}$ is a classifier over the feature space $f(\cX) := \{f(x):x \in \cX\}$. Given a training dataset $\tS:=\cup^{l}_{c=1}\{(\tx_{ci})\}_{i=1}^{m}$, we train the classifier $\tih$ on $\tS$ with the goal of minimizing the source test error, $L_{\tP}(\tih_{\tS})$.

Similarly to the target task, we assume that $\tP$ is balanced (i.e., $\tP[y=c]=1/l$) with class-conditionals $\{\tP_c\}^{l}_{c=1}$, and that the dataset $\tS=\cup^{l}_{c=1}\{\tx_{ci}\}^{m}_{i=1}$ is also balanced with $\tS_c := \{\tx_{ci}\}^{m}_{i=1} \sim \tP^{m}_c$. 
Finally, we assume that the \emph{classes in the source and target tasks are selected randomly}: that is, the elements of the sets of class-conditionals, $\tcP:=\{\tP_c\}^{l}_{c=1}$ and $\cP:=\{P_c\}^{k}_{c=1}$ for the source and target tasks, respectively, are independent and identically distributed (i.i.d.), drawn from a distribution $\cD$ over a set of class-conditional distributions $\cE$.
For simplicity, we could assume that the classification problems are all noiseless, that is, the supports of the class-conditional distributions in $\cE$ are disjoint. While this is not needed for our theorems to hold, it is needed for neural collapse to occur, which is essential so that our theoretical results provide an explanation for the success of few-shot learning. However, as our results are still meaningful under approximate neural collapse, which may still occur under some limited label noise, we keep our definition general.

In a typical transfer learning setting, the classifier $\tih$ is a deep neural network, and $f$ is the representation in the last internal layer of the network, a.k.a.\ the \emph{penultimate} or \emph{feature embedding} layer. The last layer of the network, $\tig$, is often a linear map, and the same is true for $g$ in the target problem. The learned feature map $f$ is called a \emph{foundation model} when it can be effectively used in a wide range of target tasks.

The effectiveness of a foundation model $f$ in dealing with downstream tasks can be measured by the transfer error of the trained classifier $h_{f,S}$, 
defined as the expected error $\E_{P\sim\cD}\E_{S: S_c \sim P_c^n}[L_P(h_{f,S})]$ taken over both the training and the evaluation on the random target task.
For concreteness, we focus on the simple \emph{nearest class-center (NCC)} classification rule
\begin{equation}
\hnn(x) ~:=~ \argmin_{c\in [C]} \|f(x) - \mu_{f}(S_c)\|,   \label{eq:ncc-classifier}
\end{equation}
that attributes $x$ to a class $c$ if its embedding $f(x)$ is closest to the empirical class center $\mu_f(S_c)$ amongst all centers $\{\mu_f(S_c)\}^{C}_{c=1}$.
Then the \emph{transfer error} of $f$ is defined as
\begin{equation}\label{eq:transfer_err}
\cL_{\cD}(f)~:=~\E_{P\sim\cD}\E_{S: S_c\sim P_c^n}[L_P(\hnn)],
\end{equation}
(the expectation is taken over the random choice of target tasks $P$ and the training data $S$ for the target class). Note that while the feature map $f$ is evaluated on the distribution of the target task $P$ determined by $\cD$, the training of $f$ in a foundation model is completely independent of this target. As the main contribution of the paper, we analyze the above setting and provide an explanation why transfer learning (or more precisely NCC classification) is effective through the concept of neural collapse.

\paragraph{Network architectures.} For concreteness, we focus on the case where $\cF$, the class of feature maps, is a set of depth-$q$ ReLU neural networks of the form $f(\boldsymbol{\cdot}) = W^q \sigma (W^{q-1} \dots \sigma(W^1 \boldsymbol{\cdot} )):\R^{d} \to \R^{p}$, where $\sigma(x):=\max\{0,x\}$ is the element-wise ReLU function, and $W^i \in \R^{d_{i+1}\times d_{i}}$ for $i \in [q]$, where $d_1=d$ and $d_{q+1}=p$. Our bounds depend on the norm of a network $f$, which is defined as 
\begin{equation}
\cC(f) ~:=~ 
\prod^{q}_{i=1} \|W^i\|_F,    
\end{equation}
where $\|\cdot \|_F$ is the Frobenius norm. This quantity upper bounds the Lipschitz constant of $f$ and is similar in fashion to other (slightly different) norms of neural networks \citep{golowich,10.5555/3295222.3295372}. 

\paragraph{}
Finally, we assume that the sample space $\cX$ is compact and define $B := \sup_{x\in \cX} \|x\|<\infty$. Combining this with the facts that $\cC(f)$ bounds the Lipschitz constant of a ReLU network $f\in \cF$ and that $f(\mathbf{0})=\mathbf{0}$, it follows that for any $f \in \cF$ and $x \in \cX$, 
\[
\|f(x)\| \le \cC(f)\cdot \|x\| \le \cC(f) B.
\]

%%%%%%%%%%%%%%%%%%%%%%%%%%%%%%%%%%%%%%%%%%%%%%%%%%%%%%%%%%%%%%%
\section{Neural Collapse}\label{sec:nc}
%%%%%%%%%%%%%%%%%%%%%%%%%%%%%%%%%%%%%%%%%%%%%%%%%%%%%%%%%%%%%%%

Neural collapse (NC) is a recently discovered phenomenon in deep learning \citep{Papyan24652}. It occurs in learning settings where, during the training of deep networks for standard classification tasks with deterministic labels, the feature embeddings (the output of the penultimate, or embedding, layer) associated with training samples belonging to the same class tend to concentrate around their means, or -- in other words -- the classes are clustered. Specifically, \citet{Papyan24652} observed that the ratio of the within-class variances and the distances between the class means converge to zero (NC1).\footnote{Note that condition NC1 of \citet{Papyan24652}, called \emph{variability collapse}, only requires that the average within-class variance tends to zero, and the condition on the normalized (or relative) variance we consider here only follows from having properly normalized and maximally distant class means as dictated by NC2. In the paper we refer to our more general and more meaningful notion as \emph{class-features variability collapse}.} They also noticed that, asymptotically, the class means are not only linearly separable, but are also maximally distant (and hence orthogonal) with the same norm when normalized properly (NC2). Furthermore, the weights of the last (linear) layer converge to the class means (NC3), and hence the behavior of the last-layer classifier (operating on the features) converges to that of the NCC decision rule (NC4). In other words, neural collapse is a phenomenon in which deep learning models exhibit a particular behavior during training, resulting in a concentration of feature embeddings around their mean for each class in the training set.

In this paper, we focus on two properties related to neural collapse: class-features variability collapse and the nearest class-center separability, both defined below. We will explore how these properties affect the performance of transfer learning and foundation models.

\paragraph{Class-features variability collapse.} In this work we use a simplified version of normalized class-features variability collapse, introduced in our initial paper \citep{galanti2022on}: For a feature map $f:\R^d \to \R^p$ and two distributions $Q_i,Q_j$ over $\cX \subset \R^d$, we define their \emph{class-distance normalized variance} (CDNV) as
\[
V_f(Q_i,Q_j) ~:=~ \frac{\Var_f(Q_i)}{\|\mu_f(Q_i)-\mu_f(Q_j)\|^2}. 
\]
Furthermore, the definition can be extended to finite sets $S_1, S_2 \subset \cX$ by defining 
\[
V_f(S_1,S_2)~:=~V_f(U[S_1],U[S_2]).
\]
This quantity essentially measures to what extent the feature vectors of samples from $Q_i$ and $Q_j$ are clustered in space. Suppose we have a learning algorithm that, given a training dataset $\tS$, returns a classifier $h^{(t)} = g^{(t)} \circ f^{(t)}$ after $t$ iterations, we say that it satisfies \emph{class-features variability collapse} if
\begin{equation}
\label{eq:nc1}
\lim_{t\to \infty} \Avg_{i\neq j \in [l]} [V_{f^{(t)}}(\tS_i,\tS_j)] = 0,    
\end{equation}
where $f^{(t)}$ is the penultimate layer of the neural network $h^{(t)} = g^{(t)} \circ f^{(t)}$ produced by the learning algorithm after $t$ iterations. Similarly to the measure proposed by \citet{Papyan24652}, in practice exact convergence to zero rarely occurs, but we show experimentally in Section~\ref{sec:experiments} that the left-hand side of the equation above does indeed decrease to some small value during training. \citet{https://doi.org/10.48550/arxiv.2202.09028} demonstrated empirically that the limit decreases significantly when the depth of the network $f$ is increased.

It is interesting to note that while our average CDNV measures pairwise separability of the classes (as we shall see), related measures of class-features clustering considered in the literature are much more global: In the context of measuring class-features clustering for few-shot learning,
\citet{pmlr-v119-goldblum20a} considered $\frac{l \cdot \Avg_{i \in [l]}[\Var_f(Q_i)]}{\Avg_{i\in [l]} [\|\mu_f(Q_i)-\mu_G\|^2]}$ where $\mu_G=\Avg_{i \in [l]} [\mu_f(Q_i)]$ is the global mean of the data, while \citet{Papyan24652} used the trace of $\Sigma_W \Sigma_B^\dagger/l$ to measure neural collapse where $\Sigma_W=\Avg_{i \in [l]} \E_{x \sim Q_i}[(x-\mu_f(Q_i))(x-\mu_f(Q_i))^\top]$ is the average within-class covariance matrix and
$\Sigma_B^\dagger$ is the Moore-Pensrose inverse of the average between-class covariance matrix $\Sigma_B=\Avg_{i \in [l]}[(\mu_f(Q_i)-\mu_G)(\mu_f(Q_i)-\mu_G)^\top]$.
Introducing a pairwise measure of separability instead of a global one greatly facilitates our analysis.

\paragraph{Nearest class-center separability.} A weaker notion of neural collapse, known as NCC separability, asserts that the features of samples from the same class can be accurately classified using the nearest class-center classification rule. 
To this end, the \emph{expected $\Delta$-margin NCC classification error} of a feature map $f:\R^{d} \to \R^{p}$ for two classes with class-conditional distributions $Q_i$ and $Q_j$ is defined as
\begin{equation}
\label{eq:EDelta}
E_{\Delta}(f;Q_i,Q_j) ~:=~ \E_{x_i \sim Q_i}\Big[\bI\big[\|f(x_i) - \mu_{f}(Q_j)\| \leq \|f(x_i) - \mu_{f}(Q_i)\| + \Delta\big]\Big],
\end{equation}
where $\Delta > 0$ is a margin of our choice. By setting $\Delta=0$, we simply obtain the expected classification error of the NCC classifier for the two-class classification problem of separating samples from $Q_i$ and $Q_j$ (not considering ties for the NCC classifier for simplicity).
In particular, given a dataset $\tS = \cup^{C}_{c=1} \tS_c$ and a mapping $f$, we can measure the degree of NCC separability by
$\Avg_{i\neq j \in [l]} E_{\Delta}(f;\tS_i,\tS_j)$ (where the average is taken over all ordered pairs of classes), which in turn is an upper bound on the average training loss. 

Suppose we have a learning algorithm that, given a training dataset $\tS$, returns a neural network $h^{(t)} = g^{(t)} \circ f^{(t)}$ after $t$ iterations, we say that it satisfies \emph{NCC separability} with margin $\Delta$ if the aforementioned upper bound on the classification error of the NCC classifier approaches zero as $t$ goes to infinity:
\begin{equation}
\label{eq:nc4}
\lim_{t\to \infty} \Avg_{i\neq j \in [l]}[E_\Delta(f^{(t)};\tS_i,\tS_j)] = 0.
\end{equation}
Here, again, $f^{(t)}$ is the penultimate layer of the neural network $h^{(t)} = g^{(t)} \circ f^{(t)}$.

As we show in Lemma~\ref{prop:cdnv_errorE}, NCC separability is a weaker notion of collapse than variability collapse, as the NCC error rate can be upper bounded in terms of the CDNV. However, the NCC error can be zero in cases where the CDNV is larger than zero (see the example after Lemma~\ref{prop:cdnv_errorE}).

\section{Generalization Bounds for Transfer Learning}\label{sec:analysis}

In our previous paper \citep{galanti2022on}, we introduced generalization bounds for transfer learning in the setting of Section~\ref{sec:setup}. Our analysis was divided into three main steps. First, we showed that if the CDNV on the training data, $\Avg_{i\neq j \in [l]}[V_f(\tS_i,\tS_j)]$, is small at the end of training, then we can expect that the CDNV over unseen samples, $\Avg_{i\neq j \in [l]}[V_f(\tP_i,\tP_j)]$, would also be small for the same classes when $m$ is sufficiently large (see Lemma~1 of \citealp{galanti2022on}). As a second step, we bounded the expected CDNV between two new target classes, $\E_{P_c\neq P_{c'}}[V_f(P_c,P_{c'})]$ (here the expectation is taken with respect to the random choice of the target class conditionals $P_c$ and $P_{c'}$ under the condition that they are different), using the CDNV between the source classes, $\Avg_{i\neq j \in [l]}[V_f(\tP_i,\tP_j)]$. Finally, we proved that $\cL_{\cD}(f) = \E_{P}\E_{S}[L_P(\hnn)] \leq C(k+\fr{k}{n}) \cdot \E_{P_c\neq P_{c'}}[V_f(P_c,P_{c'})]$, where $C>0$ is some constant and $\hnn$ is the NCC classifier defined in \eqref{eq:ncc-classifier}.

However, there are several limitations to this argument. The first issue arises from the fact that $\E_{P_c\neq P_{c'}}[V_f(P_c,P_{c'})]$ may be very large even when $\cL_{\cD}(f)$ is very small. For example, if the support $\cE$ of $\cD$ is infinite and $f$ is bounded, then the set ${\mu_f(P_c)}_{P_c \in \cE} \subset \R^{p}$ has at least one limit point, and therefore, $\inf_{P_{1}\neq P_{2} \in \cE} \| \mu_f(P_1)-\mu_f(P_2)\|=0$ and $\E_{P_c\neq P_{c'}}[V_f(P_c,P_{c'})]=\infty$. In addition, the upper bound on $\E_{P_c\neq P_{c'}}[V_f(P_c,P_{c'})]$ scales with the inverse of $\inf_{f \in \cF} \inf_{P_1 \neq P_2 \in \cE}\|\mu_f(P_1)-\mu_f(P_2)\|$, which may be very large even if $\cE$ is finite (for example, if $\cF$ contains a constant function).  

In this paper we aim to circumvent these issues and provide an upper bound on the transfer error $\cL_{\cD}(f)$ of $f$ which does not suffer from these limitations. Instead of bounding the term $\E_{P_c\neq P_{c'}}[V_f(P_c,P_{c'})]$, we bound $\cL_{\cD}(f)$ with the average margin error of the few-shot NCC classifier between pairs of source classes $i\neq j$ (Lemma~\ref{prop:marginBoundA}). As a next step, we show that each of these error terms can be bounded using the average margin error of few-shot NCC classifiers over the training data of classes $i$ and $j$ (Lemma~\ref{prop:marginBoundB}). Finally, we provide various upper bounds to these terms based on the CDNV between the two samples, $V_f(\tS_i,\tS_j)$ (Lemmas~\ref{prop:cdnv_error}-\ref{prop:cdnv_errorG}).
Combining these bounds we obtain the main results of this paper: upper bounds on the transfer error of $f$ in terms of the average CDNV between pairs of source classes, $\Avg_{i\neq j}[V_f(\tS_i,\tS_j)]$, which typically yields small values, as demonstrated empirically in Section~\ref{sec:experiments}.

Our first results, Theorem~\ref{thm:full_bound} covers the general case, while Theorem~\ref{thm:full_bound2} provides somewhat better results when the feature embeddings satisfy certain symmetry properties. The main proof components (i.e., the aforementioned lemmas) are presented in the next section, while the full proofs of the theorems are given in Section~\ref{app:fullbound}.

\begin{restatable}{theorem}{fullBound}\label{thm:full_bound} Let $\delta \in (0,1)$, $l \geq 2$, $1 \leq n \leq \sqrt{m}$ and let $\cF$ be a class of ReLU neural networks of depth $q$. Let $\Lambda := \min_{i\neq j \in [l]}\|\mu_f(\tS_i)-\mu_f(\tS_j)\|$ and define $\kappa:=\frac{m}{m-n}$. Then, with probability at least $1-\delta$ over the selection of the source class-conditionals $\tcP = \{\tP_c\}^{l}_{c=1}$ and the corresponding source training data $\tS_1\dots,\tS_l$, for any $f \in \cF$, we have
\begin{equation*}
\begin{aligned}
\mathcal{L}_{\cD}(f) &\le C_1 (k-1) \kappa \Avg_{i\neq j \in [l]}[V_f(\tS_i,\tS_j)]
+ (k-1)\kappa \delta^2 \\
&\quad+ 
C_2 \frac{(k-1)\lceil \cC(f) \rceil B \kappa n^2}{\Lambda\sqrt{m-n}}
\sqrt{(q+\log(p))\log(m)} \\
&\quad+ 
C_3 \frac{(k-1)\lceil \cC(f) \rceil B \kappa}{\Lambda\sqrt{m-n}}\sqrt{\log\left(\fr{l^2(\lceil \cC(f) \rceil+1)^2 \left(\left|\log \left(\fr{\Lambda}{40B}\right)\right|+2\right)^2}{\delta}\right)} \\
&\quad +
C_4 \frac{(k-1)\lceil \cC(f) \rceil B}{\Lambda \sqrt{l}}
\left(\!\sqrt{\log(l)(q + \log(p))} + \sqrt{\log\left(\fr{(\lceil \cC(f) \rceil+1)\left(\left|\log \left(\fr{\Lambda}{40B}\right)\right|+2\right)}{\delta}\right)}\ \right)\!,
\end{aligned}
\end{equation*}
where $C_1,C_2,C_3,C_4>0$ are appropriate constants.
\end{restatable}

In the above theorem, we bound the transfer error $\mathcal{L}_{\mathcal{D}}(f)$, which is the quantity that measures the performance of $f$ in adapting to new tasks with very few samples. For simplicity, we stated the bound without explicitly specifying the constants, which are given in \eqref{eq:thm6_consts} in Section~\ref{app:fullbound}.

To interpret the bound, note that we are interested in the regime where $n$ is a constant (or when $n =o(m^{1/4})$), hence $\kappa=\fr{m}{m-n} \approx 1+\mathcal{O}(\fr{1}{m})$ (or $1+o(m^{-3/4})$). The bound is composed of multiple terms. The first term is proportional to $k$ (the number of target classes) times the average CDNV over the source training data, $\Avg_{i\neq j \in [l]} [V_{f}(\tS_i,\tS_j)]$, which is expected to be small in the presence of neural collapse. The second term scales as $\mathcal{O}\left(\kappa k \delta^2\right)$ that depends on a free parameter $\delta$ that can be selected to be very small, as the other terms scale (poly-)logarithmically with respect to $1/\delta$. The third term scales as $\tilde{\mathcal{O}}\left(\fr{n^2k\sqrt{q}\lceil \mathcal{C}(f)\rceil B}{\Lambda\sqrt{m}}\right)$ dominating the fourth term which is a factor $1/n^2$ smaller, while the last term scales as $\tilde{\mathcal{O}}\left(\fr{k\sqrt{q}\lceil\mathcal{C}(f)\rceil B}{\Lambda\sqrt{l}}\right)$. 

Thus, selecting, e.g., $\delta=1/m$, the terms other than the CDNV in the bound are of order
$\fr{k\sqrt{q}\lceil\mathcal{C}(f)\rceil B}{\Lambda}\cdot\left(\fr{n^2}{\sqrt{m}}+\fr{1}{\sqrt{l}}\right)$ up to logarithmic factors, which goes to zero with $m$ and $l$ increasing as desired (with the square-root rate expected from standard concentration arguments), while in the numerator we have typical complexity terms. The two unusual terms are $n^2$ in the numerator and $\Lambda$ in the denominator. 
The presence of $n$ is somewhat undesired and is probably the artifact of our proof technique as one would expect better performance as the number of samples in the target problems increases; however, since we are interested in the regime where $n$ is very small (a small constant, typically below $10$), this is only a slight problem.
$\Lambda$ is a data-dependent term, measuring the minimum pairwise distance of the class-means. According to previous work on neural collapse \citep{Papyan24652,mixon2020neural,zhu2021geometric,4880,pmlr-v162-tirer22a}, under the usual conditions required for neural collapse (such as, training with weight decay, the features are properly normalized and $p \geq l$), the function $f$ converges to a solution such that $\{\mu_f(\tilde{S}_i)\}^{l}_{i=1}$ forms a simplex equiangular tight frame, meaning that the class centers $\mu_f(\tilde{S}_i)$ are of equal length and the distances $\|\mu_f(\tilde{S}_i)-\mu_f(\tilde{S}_j)\|$ are also equal and maximized for all $i\neq j$. This property, which is typically referred to as NC2 in the literature, implies that $\Lambda=\sqrt{2}
\|\mu_f(\tS_i)\|$. To empirically validate that $\Lambda$ is indeed maximized, we evaluated $\Lambda$ during our experiments (reported in Section~\ref{sec:distances}), and we can observe that it significantly increases during training. Therefore, if $n=\mathcal{O}(1)$ and $l$ and $m$ are large, and if neural collapse is present in the source data (i.e., $\Avg_{i\neq j \in [l]} [V{f}(\tS_i,\tS_j)]$ is small) and $\cC(f)$ is bounded (as a function of $m$ and $l$, i.e., if the size of the neural network $f$ only mildly increases with the amount of the input data, as per the practitioner's choice), then the few-shot transfer error $\mathcal{L}_{\mathcal{D}}(f)$ is also small.

Next we present two improvements for Theorem~\ref{thm:full_bound}, both aiming to improve the dependence on the CDNV, $\Avg_{i\neq j \in [l]} [V{f}(\tS_i,\tS_j)]$. The first idea is to measure neural collapse with the weaker notion of the average $\Delta$-margin error $\Avg_{i\neq j \in [l]} [E_{8\Delta}(\tS_i,\tS_j)]$ for some specific choice of $\Delta$. In fact, if the feature vectors are sufficiently clustered during training, then the $\Delta$-margin error can easily vanish while the CDNV remains positive. 

\begin{restatable}{theorem}{fullBound1}\label{thm:full_bound1} Let $\delta \in (0,1)$, $l \geq 2$, $1 \leq n \leq \sqrt{m}$ and let $\cF$ be a class of ReLU neural networks of depth $q$. Furthermore, let $\Lambda_{ij} := \|\mu_f(\tS_i)-\mu_f(\tS_j)\|$ for $i,j\in[l]$, $\Lambda := \min_{i\neq j \in [l]} \Lambda_{ij}$ as before,  and define $\kappa:=\frac{m}{m-n}$. Then, with probability at least $1-\delta$ over the selection of the source class-conditionals $\tcP = \{\tP_c\}^{l}_{c=1}$ and the corresponding source training data $\tS_1\dots,\tS_l$, for any $f \in \cF$ and $\phi>0$, we have
\begin{equation*}
\begin{aligned}
\mathcal{L}_{\cD}(f) &\le C_1 (k-1) \kappa 
\left(\Avg_{i\neq j \in [l]}[E_{\phi\Lambda}(f;\tS_i,\tS_j)]
+\frac{C_2}{\phi^2 n^2}
\Avg_{i\neq j \in [l]}\left[\frac{\Lambda_{ij}^2}{\Lambda^2}V_f(\tS_i,\tS_j)\right]\right) \\
&\quad + (k-1)\kappa \delta^2
+ 
C_3 \frac{(k-1)\lceil \cC(f) \rceil B \kappa n^2}{\phi \Lambda\sqrt{m-n}}
\sqrt{(q+\log(p))\log(m)} \\
&\quad+ 
C_4 \frac{(k-1)\lceil \cC(f) \rceil B \kappa}{\phi \Lambda\sqrt{m-n}}\sqrt{\log\left(\fr{l^2(\lceil \cC(f) \rceil+1)^2 \left(\left|\log \left(\fr{\phi \Lambda}{40B}\right)\right|+2\right)^2}{\delta}\right)} \\
&\quad +
C_5 \frac{(k-1)\lceil \cC(f) \rceil B}{\phi \Lambda \sqrt{l}}
\left(\!\sqrt{\log(l)(q + \log(p))} + \sqrt{\log\left(\fr{(\lceil \cC(f) \rceil+1)\left(\left|\log \left(\fr{\phi \Lambda}{40B}\right)\right|+2\right)}{\delta}\right)}\ \right)\!,
\end{aligned}
\end{equation*}
where $C_1,C_2,C_3,C_4,C_5>0$ are appropriate constants.
\end{restatable}

If the features are sufficiently clustered during training, the $E_{\phi\Lambda}$ terms vanish (or become very small) for $\phi$ sufficiently small. Furthermore, when neural collapse happens, the class-mean distances are typically of the same order (i.e., $\Lambda_{ij} \approx \Lambda$ for all $i\neq j \in [l]$), and so the second term involving the $V_f$ terms becomes $1/n^2$-times the CDNV, significantly improving the dependence on the bound on the CDNV compared to Theorem~\ref{thm:full_bound}. This comes at the price of multiplying the rest of the terms with a $1/\phi>1$ factor.

A further improvement is possible if the distribution of the embeddings of the training samples are spherically symmetric. In case they are, the CDNV term in the bound can be multiplied with an additional factor $(1/p + 1/n)$ (compared to Theorem~\ref{thm:full_bound}), again assuming that the class-mean distances are of the same order (i.e., $\Lambda_{ij} \approx \Lambda$ for all $i\neq j \in [l]$). This can be a substantial improvement when the number of samples $n$ from the target classes is moderate, as the embedding dimension $p$ is typically much larger (at least a few hundred in practical situations).

\begin{restatable}{theorem}{fullBoundSec}\label{thm:full_bound2} Let $\delta \in (0,1)$, $l \geq 2$, $1 \leq n \leq \sqrt{m}$ and let $\cF$ be a class of ReLU neural networks of depth $q$. Furthermore, let $\Lambda_{ij} := \|\mu_f(\tS_i)-\mu_f(\tS_j)\|$ for $i,j\in[l]$, $\Lambda := \min_{i\neq j \in [l]} \Lambda_{ij}$ as before, and $s:=p$ if the set $\{f \circ \tS_c\}$ is spherically symmetric for all $c \in [l]$ and $s:=1$ otherwise. Then, with probability at least $1-\delta$ over the selection of the source class-conditionals $\tcP = \{\tP_c\}^{l}_{c=1}$ and the corresponding source training data $\tS_1\dots,\tS_l$, for any $f \in \cF$, we have
\begin{equation*}
\begin{aligned}
\mathcal{L}_{\cD}(f) &\le C_1 (k-1) \kappa \Avg_{i\neq j \in [l]}\left[ \left(\fr{1}{s} + \fr{\Lambda^2_{ij}}{n\Lambda^2}
+ \fr{\Lambda^2}{n\Lambda^2_{ij}}\right) \cdot V_f(\tS_i,\tS_j)\right]
+ (k-1)\kappa \delta^2 \\
&\quad+ 
C_2 \frac{(k-1)\lceil \cC(f) \rceil B \kappa n^2}{\Lambda\sqrt{m-n}}
\sqrt{(q+\log(p))\log(m)} \\
&\quad+ 
C_3 \frac{(k-1)\lceil \cC(f) \rceil B \kappa}{\Lambda\sqrt{m-n}}\sqrt{\log\left(\fr{l^2(\lceil \cC(f) \rceil+1)^2 \left(\left|\log \left(\fr{\Lambda}{40B}\right)\right|+2\right)^2}{\delta}\right)} \\
&\quad +
C_4 \frac{(k-1)\lceil \cC(f) \rceil B}{\Lambda \sqrt{l}}
\left(\!\sqrt{\log(l)(q + \log(p))} + \sqrt{\log\left(\fr{(\lceil \cC(f) \rceil+1)\left(\left|\log \left(\fr{\Lambda}{40B}\right)\right|+2\right)}{\delta}\right)}\ \right)\!,
\end{aligned}
\end{equation*}
where $C_1,C_2,C_3,C_4>0$ are appropriate constants.
\end{restatable}

Finally, although the bounds in the above theorems tend to zero as $m$ and $l$ tend to infinity and when class-features variability collapse happens during training, similarly to standard generalization bounds in the deep learning literature \citep[e.g.,][]{golowich,10.5555/3295222.3295372,pmlr-v40-Neyshabur15,https://doi.org/10.48550/arxiv.1806.05159}, these bounds may be quite loose in practical situations. The aim of these results is to serve as a ``proof of concept'' that few-shot learnability is possible by training a large model that exhibits class-features variability collapse with respect to the training data when trained on many source classes and many samples for each class.

\section{Main Components of the Analysis}

In this section we present the main ingredients of the proofs of our main results, proving our main theorems in Section~\ref{app:fullbound}.

\subsection{Generalization to New Classes}

In this section we derive a generalization bound on the transfer error $\mathcal{L}_{\mathcal{D}}(f)$ as a function of an average two-class soft-margin error of the NCC classifier (based on $f$) over the training classes. We start by reducing the problem to two-class classification: using the union bound, we can easily show (see the proof of Lemma~\ref{prop:marginBoundA}) that
\begin{equation}
\label{eq:pairwise}
\mathcal{L}_{\cD}(f) \le (k-1) \cdot \E_{P_1,P_2}\E_{S_1,S_2}\E_{x_1 \sim P_1}\big[\bI\left[\|f(x_1) - \mu_f(S_2)\| < \|f(x_1) - \mu_f(S_1)\|\right] \big].
\end{equation}

\begin{wrapfigure}{r}{0.48\textwidth}
\begin{tikzpicture}[
  declare function={
    func(\x)= (\x < -1) * (0) + and(\x >= -1, \x < 0) * (1+\x) + (\x >= 0) * (1);
    g(\x) = (\x < -1) * (0) + (\x >= -1) *(1);
    h(\x) = (\x < 0) * (0) + (\x >= 0) *(1);
  }
]
\begin{axis}[
  axis x line=middle, axis y line=middle,
  ymin=-0.2, ymax=1.5, ytick={-1,...,1}, ylabel={}, ymax=1.3, yticklabels={},
  xmin=-2, xmax=1.5, xtick={-1,...,0}, xticklabels={$-\Delta$,0}, xlabel=$r$,
  domain=-2.5:2,samples=400, 
  axis equal image,
]
\addplot [black,thick,dashed]{h(x)};
\addplot [black,thick,dashed]{g(x)};
\addplot [black,thick] {func(x)};
\end{axis}
\node at (1.1,1.8) {$\bI[r > -\Delta]$};
\node at (4.7,1.3) {$\bI[r > 0]$};
\node at (5.3,2.7) {$\ell_{\Delta}(r)$};
\node at (3.7,2.6) {$1$};
\end{tikzpicture} 
\vspace{-0.5cm}
\caption{The soft-margin loss $\ell_{\Delta}$.
\label{fig:ellDelta}}
\end{wrapfigure}
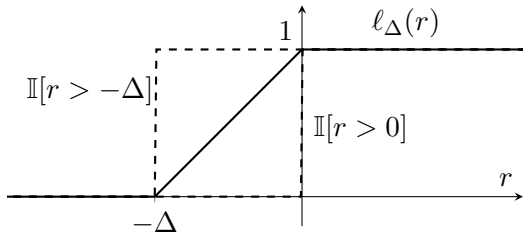
The interesting part of the analysis is to relate the expected two-class classification error above, where the expectation is taken over the random selection of the class-conditionals $P_1$ and $P_2$ of the target problem, to a similar quantity which depends on the training classes only; in particular, we aim to replace the expectation $\E_{P_1,,P_2}[\cdot]$ by $\Avg_{\tP_i \neq \tP_j \in [l]}[\cdot]$, the average over the pairs of the source classes.
Notice that analyzing such an average has two unusual components: First, we treat class-conditional distributions as data points on which the feature map $f$ is trained. Second, the average is not a standard linear statistics over the data points (i.e., class conditionals $\{tP_c\}_{c=1}^l$), since although we assume that both the source and target class-conditionals, $\{\tP_c\}^{l}_{c=1}$ and $\{P_c\}^{k}_{c=1}$ are i.i.d.\ samples from $\cD$, the individual terms in the average depend on two distributions, and hence they are interdependent. To be able to analyze their concentration properties, we build on the work of \citet{Maurer2019UniformCA}, who considered the deviation of nonlinear functions of independent random variables from their means. More specifically, we apply their Corollary~3 to bound the deviation of the average pairwise loss from its expectation. While the corollary generally applies to any bounded loss functions, certain terms that arise from this bound can be more easily quantified when the loss function is Lipschitz continuous. Therefore, we replace the zero-one loss function $\bI[r > 0]$ with the ``soft-margin'' loss function $\ell_\Delta:\R \to [0,\infty)$, defined as follows:
\begin{equation}
\label{eq:ellDelta}
\ell_\Delta(r) ~:=~ \threepartdef
{0} {r < -\Delta;}
{1+\frac{r}\Delta} {r \in [-\Delta,0];}
{1} {r>0.}
\end{equation}
It is clear that for all $\Delta>0$ and $r \in \R$, we have $\bI[r > 0]\le \ell_\Delta(r)\le \bI[r > -\Delta]$ and $\ell_{\Delta}$ is a $\fr{1}{\Delta}$-Lipschitz continuous function (see Figure~\ref{fig:ellDelta}). Thus, the expected pairwise error of few-shot trained NCC classifiers $\hnn$ for any two classes $P_1, P_2$ can be bounded by the soft-margin error $\ell_\Delta(f;P_1,P_2)$ given below: for a feature map $f:\R^{d}\to\R^{p}$ and two class-conditional distributions $Q_i,Q_j$, the \emph{$n$-shot} (or, informally, \emph{few-shot})  \emph{soft-margin NCC classification error} of $f$ is defined as
\begin{equation*}
\ell_{\Delta}(f;Q_i,Q_j) ~:=~ \E_{\hS_i\sim Q^n_i}\E_{\hS_j \sim Q^n_j}\E_{x_i \sim Q_i}\left[\ell_{\Delta}(\|f(x_i) - \mu_{f}(\hS_i)\| - \|f(x_i) - \mu_{f}(\hS_j)\|)\right].
\end{equation*} 
This quantity measures the expected soft-margin test error of the NCC classifier when using $n$ samples per-class to estimate the centers of $f$ over $Q_i$ and $Q_j$. The next result, proved in Appendix~\ref{app:marginBoundA}, is obtained by bounding 
$\E_{P_1,P_2}[\ell_\Delta(f;P_1,P_2)]$ from the right-hand side of \eqref{eq:pairwise} by $\Avg_{i\neq j\in [l]}[\ell_{2\Delta}(f; \tP_i,\tP_j)$, the expected few-shot soft-margin error of $f$ over the source classes, plus additional terms, which depend on the complexity of $f$ and the number of classes (also note the technical change of the margin from $\Delta$ to $2\Delta$):

\begin{restatable}{lemma}{marginBoundA}\label{prop:marginBoundA} Let $\delta \in (0,1)$, $l \geq 2$, and let $\mathcal{F}$ be a class of ReLU neural networks of depth $q$. With probability at least $1-\delta$ over the selection of i.i.d. class-conditional distributions $\tcP = \{\tP_c\}^{l}_{c=1} \sim \cD^{l}$, for any $f \in \mathcal{F}$ and $\Delta > 0$, it holds that
\begin{align*}
\mathcal{L}_{\cD}(f) ~&\le~ (k-1) \left( \Avg_{i\neq j\in [l]}[\ell_{2\Delta}(f; \tP_i,\tP_j)] + \delta^2 \right) \\
&\quad\quad +
(k-1)
\frac{\lceil \cC(f) \rceil B}{\Delta \sqrt{l}}
\Bigg(355 \sqrt{\log(l)(q\log(2)+\log(p))} \\
&\qquad\qquad\qquad\qquad\qquad\qquad+ 16.5\sqrt{\log\left(\fr{\sqrt{3}(\lceil \cC(f) \rceil+1)(|\log (\Delta/B)|+2)}{\delta}\right)}\Bigg).
\end{align*}
\end{restatable}
The bound above can be broken down into several parts. The first term is the average few-shot soft-margin error of $f$ between the source classes $\tP_1,\dots,\tP_l$. Specifically, for each pair of source classes $(i,j)$, we measure the soft-margin \emph{test} error (over samples $x_i \sim \tP_i$) of NCC classifiers $\hnnS{\hS}$ produced using two random datasets $\hS_i \sim \tP^n_i$ and $\hS_j \sim \tP^n_j$ of size $n$. The second term, $\delta^2$, depends on a free parameter (the probability of error) that can be selected to be very small, as the only other term (the last one) which depends on $\delta$ scales as $\sqrt{\log(1/\delta)}$ with respect to this parameter. The third term is proportional to the ratio between the complexity of the selected function $f$ and the task, which is captured by $\sqrt{\log(l)(q+\log(p))}kB \cC(f)$, and $\Delta \sqrt{l}$. This means that as the number of source classes $l$ increases, we can expect the generalization to new classes to improve as long as $\Delta$ remains unchanged.

\subsection{Generalization to New Samples}

In the previous lemma, we bounded the loss function $\mathcal{L}_{\mathcal{D}}(f)$ using the average few-shot soft-margin error between pairs of source classes, $\Avg_{i\neq j\in [l]}[\ell_{2\Delta}(f; \tP_i,\tP_j)]$. Since during the training of a neural network we do not have access to the true source class-conditional distributions $\tP_1,\dots,\tP_l$, we would like to bound each term $\ell_{2\Delta}(f; \tP_i,\tP_j)$ using its empirical counterpart $\ell_{4\Delta}(f; \tS_i,\tS_j)$. The following lemma, proved in Appendix~\ref{app:marginBoundB} using similar techniques as in the proof of Lemma~\ref{prop:marginBoundA}, provides such a bound.

\begin{restatable}{lemma}{marginBoundB}\label{prop:marginBoundB}
Let $\delta \in (0,1)$, $1 \leq n \leq \sqrt{m}$, and let $\mathcal{F}$ be a class of ReLU neural networks of depth $q$. Let $\tP_i$ and $\tP_j$ be two class-conditional distributions. With probability at least $1-\delta$, for any $f\in \mathcal{F}$ and $\Delta > 0$, the following holds when sampling $\tS_i \sim \tP^m_i$ and $\tS_j \sim \tP^{m}_j$ independently:
\begin{equation*}
\begin{aligned}
\ell_{\Delta}(f;\tP_i,\tP_j) ~&\le~ 4\left(1+\frac{n}{m-n}\right) \cdot \ell_{2\Delta}(f;\tS_i,\tS_j) \\
&\qquad+ 
16 \lceil \cC(f) \rceil B\frac{\sqrt{m}(2n+4)^2}{(m-n)\Delta}
\sqrt{\pi(q\log(2)+\log(p))\log(2m)} \\
&\qquad+ 
32 \lceil \cC(f) \rceil B \frac{\sqrt{m}}{(m-n)\Delta}\sqrt{2 \log\left(\fr{3(\lceil \cC(f) \rceil+1)^2 (|\log(\Delta/B)|+2)^2}{\delta}\right)}.
\end{aligned}
\end{equation*}
\end{restatable}
The first term in the above bound, $\ell_{2\Delta}(f;\tS_i,\tS_j)$, is the few-shot soft-margin error of $f$ between the empirical distributions of the samples $\tS_i$ and $\tS_j$, respectively. This quantity measures the empirical error (over $x_i \in \tS_i$) of the NCC classifier $\hnnS{\hS}$ for random datasets $\hS_i$ and $\hS_j$, each containing $n$ samples selected uniformly at random from $\tS_i$ and $\tS_j$ (i.e., $\hS_i \sim U[\tS_i]^n$ and $\hS_j \sim U[\tS_j]^n$). Note that we are interested in the situation where $n \ll m$, that is, the coefficient of the first term is very close to $1$ (alternatively, we can think of having a leading coefficient $1$ and an extra error term $\frac{n}{m-n}\ell_{2\Delta}(f;\tS_i,\tS_j) \le \frac{n}{m-n}$). The other terms in the bound are of order $\tilde{\mathcal{O}}\left(\fr{\mathcal{C}(f)B\sqrt{q}n^2}{\Delta\sqrt{m}}\right)$, which tends to zero as the number of samples per class $m$ increases, as long as $n=o(m^{1/4})$. 

\subsection{Bounding the Few-Shot Soft-Margin Error}

In the previous sections, we showed that the transfer error $\mathcal{L}_{\mathcal{D}}(f)$ can be upper bounded using the average few-shot soft-margin error between the source classes, denoted as $\ell_{2\Delta}(f;\tP_i,\tP_j)$. We also demonstrated that each of these terms can be bounded using their empirical counterparts, $\ell_{4\Delta}(f;\tS_i,\tS_j)$. Informally, $(\ell_{4\Delta}(f;\tS_i,\tS_j)+\ell_{4\Delta}(f;\tS_j,\tS_i))$ measures the accuracy with which we can classify samples $\tS$ using a random NCC classifier $\hnnS{\hS}$ based on $n$ samples from each of the training sets $\tS_i$ and $\tS_j$. This quantity can be thought of as measuring the degree to which the samples $f(\tS_i)=\{f(x_i): x_i \in \tS_i\}$ and $f(\tS_j)=\{f(x_j): x_j \in \tS_j\}$ are clustered in space. Based on this intuition, we aim to upper bound each term $\ell_{4\Delta}(f;\tS_i,\tS_j)$ using measures of neural collapse. Specifically, next we derive generic upper bounds for $\ell_{\Delta}(f;Q_i,Q_j)$ in terms of $V_f(Q_i,Q_j)$ and $E_{\Delta}(f;Q_i,Q_j)$, where $Q_i$ and $Q_j$ represent arbitrary class-conditional distributions (recall that $E_\Delta$ is defined in \eqref{eq:EDelta}). All the proofs for this section are presented in Appendix~\ref{sec:error_analysis}.

We start with a lemma that connects the few-shot soft-margin error to the CDNV (this result is an extension of Proposition~5 in our preliminary work \citep{galanti2022on}).

\begin{restatable}{lemma}{cdnvError}\label{prop:cdnv_error}
Consider a pair of class-conditional distributions, $Q_i$ and $Q_j$, and a feature map $f:\R^d \to \R^p$. Let $\mu_c = \mu_f(Q_c)$, where $c \in \{i,j\}$, and let $\Delta\leq 0.1\|\mu_i-\mu_j\|$. Then
\begin{align*}
\ell_\Delta(f;Q_i,Q_j) \leq 12\left(\fr{1}{n}+1\right)\cdot V_f(Q_i,Q_j) + \fr{100}{n} \cdot \big(V_f(Q_i,Q_j)+V_f(Q_j,Q_i)\big).
\end{align*}
Furthermore, if $\{f\circ Q_i, f \circ Q_j\}$ are spherically symmetric, for $\Delta \leq 0.1 \cdot \|\mu_i-\mu_j\|/p$, we have
\begin{align*}
\ell_\Delta(f;Q_i,Q_j) \leq 128\left(\fr{1}{n}+\fr{1}{p}\right)\cdot V_f(Q_i,Q_j) + \fr{100}{n} \cdot \big(V_f(Q_i,Q_j)+V_f(Q_j,Q_i)\big).
\end{align*}
\end{restatable}

Informally, the expected soft-margin error $\ell_\Delta(f;Q_i,Q_j)$ of $f$ is upper bounded by $\mathcal{O}((1+\fr{1}{n}) \cdot V_f(Q_i,Q_j))$, and when $f\circ Q_i$ and $f\circ Q_j$ are spherically symmetric and $\Delta$ is also of order $1/p$, it can be bounded by the smaller quantity $\mathcal{O}((\fr{1}{p}+\fr{1}{n}) \cdot V_f(Q_i,Q_j))$. 
Therefore, in case of neural collapse (i.e., when $\Avg_{i\neq j}[V_f(\tS_i,\tS_j)]$ is small), we expect the term $\Avg_{i\neq j}[\ell_\Delta(f;\tS_i,\tS_j)]$ to be small.

There are two potential problems with the above lemma. In the first part, the soft-margin loss scales with $V_f(Q_i,Q_j))$, and so the number of samples $n$ has limited effect. While the second part tries to address this, improving the scaling to $(1/n+1/p)V_f(Q_i,Q_j))$ (note that typically $p$ is larger than $n$), this comes with the cost that $\Delta=\mathcal{O}(\|\mu_i-\mu_j\|/p$. Substituting such a $\Delta$ into Lemma~\ref{prop:marginBoundA}, the second term would become $\mathcal{O}(\fr{1}{\Delta \sqrt{l}})=\mathcal{O}(\sqrt{\fr{p}{l}})=\Omega(1)$, since neural collapse typically requires $p \ge l$ to allow representation of $l$ (approximately) orthogonal class mean feature vectors.
We address these issues in the next lemma.

\begin{restatable}{lemma}{cdnvErrorE}\label{prop:cdnv_errorE}
Consider a pair of class-conditional distributions, $Q_i$ and $Q_j$, and a feature map $f:\R^d \to \R^p$. Let $\Delta >0$ and $\mu_c=\mu_f(Q_c)$. Then
\begin{align*}
\ell_{\Delta}(f;Q_i,Q_j) ~&\le~ E_{2\Delta}(f;Q_i,Q_j) +\fr{4\|\mu_i-\mu_j\|^2}{n^2 \Delta^2}\cdot \big(V_f(Q_i,Q_j)+V_f(Q_j,Q_i)\big).
\end{align*}
Furthermore, let $s = p$ if $\{f\circ Q_i, f \circ Q_j\}$ are spherically symmetric and $s=1$ otherwise. Then, if $\Delta\leq \fr{1}{4}\|\mu_i-\mu_j\|$, we have
\begin{align*}
\ell_{\Delta}(f;Q_i,Q_j) ~&\le~ 64\left(\fr{1}{s}+\fr{\Delta^2}{\|\mu_i-\mu_j\|^2}\right) \cdot V_f(Q_i,Q_j) +\fr{4\|\mu_i-\mu_j\|^2}{n^2 \Delta^2}\cdot \big(V_f(Q_i,Q_j)+V_f(Q_j,Q_i)\big).
\end{align*}
\end{restatable}

The first part of the lemma shows that the few-shot soft-margin error $\ell_{\Delta}(f;Q_i,Q_j)$ can be bounded using the margin error of the NCC classifier, $E_{2\Delta}(f;Q_i,Q_j)$, together with the term $\frac{4\|\mu_i-\mu_j\|^2}{n^2 \Delta^2}(V_f(Q_i,Q_j)+V_f(Q_j,Q_i))$. The main advantage of this bound is its weaker dependence on $V_f(Q_i,Q_j)$ in the general case (not spherically symmetric embedding distributions, i.e., $s=1$). The quantity $E_{2\Delta}(f;Q_i,Q_j)$ measures the $2\Delta$-margin error of the NCC classifier with means $\mu_f(Q_i)$ and $\mu_f(Q_j)$ in the feature space given by $f$, which is a much weaker notion of separation than the normalized variance $V_f(Q_i,Q_j)$. Specifically, $E_{2\Delta}(f;Q_i,Q_j)$ can be zero with $\Delta=\mathcal{O}(\|\mu_i-\mu_j\|)$ even when $V_f(Q_i,Q_j)>0$. For example, if the embeddings $f(x)$ of samples from the two classes $Q_i$ and $Q_j$, respectively, are uniformly distributed over $0.5R$-radius circles around the points $\mu_i=(-R,0)$ and $\mu_j=(R,0)$ in $\mathbb{R}^2$, then they are perfectly NCC separable with margin $\Delta=R$ while the CDNV between the two distributions is a positive constant (independent of $R$). In particular, even though $E_{2\Delta}(f;Q_i,Q_j)=0$ for any $\Delta\leq 0.5R$, we have $V_f(Q_i,Q_j)>0$. In this case, choosing $\Delta = \Theta(\|\mu_i-\mu_j\|)$, the upper bound becomes $\mathcal{O}(\frac{1}{n^2}(V_f(Q_i,Q_j)+V_f(Q_j,Q_i)))$, in contrast to the previous bound that does not vanish as $n\to \infty$. 

The second part of the lemma follows by bounding $E_{2\Delta}(f;Q_i,Q_j)$ using $V_f(Q_i,Q_j)$ in the following manner: $E_{2\Delta}(f;Q_i,Q_j) \leq 64\left(\frac{1}{s}+\frac{\Delta^2}{\|\mu_i-\mu_j\|^2}\right) \cdot V_f(Q_i,Q_j)$. Choosing $\Delta = \mathcal{O}(\|\mu_i-\mu_j\|/n)$, we obtain a similar
$\mathcal{O}\left((\frac{1}{s}+\frac{1}{n})\cdot V_f(Q_i,Q_j)\right)$ bound as in the second part of Lemma~\ref{prop:cdnv_error}, but here $\Delta$ needs to scale with only $1/n$, not $1/p$ as in the previous lemma, allowing a meaningful combination with Lemma~\ref{prop:marginBoundA}.

Next, we derive a much better bound, which is exponentially small in $\frac{p}{V_f(Q_i,Q_j)}$ if $f\circ Q_c$ are assumed to be centered Gaussian distributions, as also assumed by \citet{Papyan24652}.

\begin{restatable}{lemma}{cdnvErrorG}\label{prop:cdnv_errorG}
Consider a pair of class-conditional distributions, $Q_i$ and $Q_j$, and a feature map $f:\R^d \to \R^p$. Let $\Delta \in (0,\|\mu_i-\mu_j\|)$ with $\mu_c = \mu_f(Q_c)$ for $c=1,2$. Assume that $f \circ Q_c$ are spherical Gaussian distributions with means $\mu_c$ such that $V^{ij}_f := \frac{\Var_f(Q_i)}{(\|\mu_i-\mu_j\| - \Delta)^2} \le 1/16$. Then we have $\ell_\Delta(f;Q_i,Q_j) \le  \frac{3\exp\left(-p/(32V^{ij}_f)\right)}{(\textnormal{e} \cdot V^{ij}_{f})^{p/2}}$.
\end{restatable}
The above lemma provides an upper bound for $\ell_\Delta(f;Q_i,Q_j)$ using a term that exponentially decays with $1/V^{ij}_f$. While this bound rapidly decays with the degree of neural collapse, it cannot be immediately applied to upper bound $\ell_{\Delta}(f;\tS_i,\tS_j)$ since $f(x_i)$ and $f(x_j)$ cannot be normally distributed for $x_i \sim U[\tS_i]$ and $x_j\sim U[\tS_j]$. However, we can instead apply this bound for $\ell_\Delta(f;\tP_i,\tP_j)$ when $f \circ \tP_i$ and $f \circ \tP_j$ are spherical Gaussian distributions and obtain a bound $\ell_{\Delta}(f;\tP_i,\tP_j)\leq \frac{3\exp\left(-p/(32V^{ij}_f)\right)}{(\textnormal{e} \cdot V^{ij}_{f})^{p/2}}$. In addition, we can use standard concentration inequalities (a variation of Lemma~1 in our earlier work, \citealp{galanti2022on}) to upper bound $V^{ij}_f = \frac{\Var_f(\tP_i)}{(\|\mu_i-\mu_j\| - \Delta)^2}$ using its empirical counterpart $\hat{V}^{ij}_f := \frac{\Var_f(\tS_i)}{(\|\mu_f(\tS_i)-\mu_f(\tS_j)\| - \Delta)^2}$. This would give a bound on $\ell_{\Delta}(f;\tP_i,\tP_j)$ in terms of $\hat{V}^{ij}_f$.

\subsection{Proofs of Theorems~\ref{thm:full_bound}--\ref{thm:full_bound2}}
\label{app:fullbound}

Given the above lemmas, we are ready to prove our main theorems.
We start with proving Theorem~\ref{thm:full_bound}.
Applying Lemma~\ref{prop:marginBoundA} with $\delta/2$ in place of $\delta$ and Lemma~\ref{prop:marginBoundB} for all pairs $i \neq j \in [l]$, with $\fr{\delta}{2l(l-1)}$ in place of $\delta$, the union bound implies that
with probability at least $1-\delta$ over the selection of $\tS_1,\dots,\tS_l$, for all $f\in \cF$ and $\Delta>0$, we have
\begin{equation}
\label{eq:transfer_risk_avg}
\begin{aligned}
\mathcal{L}_{\cD}(f) ~&\le~ 4 (k-1) \left(1+\frac{n}{m-n}\right)\left( \Avg_{i\neq j\in [l]}[\ell_{4\Delta}(f; \tS_i,\tS_j)] + \fr{\delta^2}{4} \right) \\
&\qquad+ 
8 (k-1)\lceil \cC(f) \rceil B\frac{\sqrt{m}(2n+4)^2}{(m-n)\Delta}
\sqrt{\pi(q\log(2)+\log(p))\log(2m)} \\
&\qquad+ 
16 (k-1)\lceil \cC(f) \rceil B \frac{\sqrt{m}}{(m-n)\Delta}\sqrt{2 \log\left(\fr{6l^2(\lceil \cC(f) \rceil+1)^2 (|\log(\Delta/B)|+2)^2}{\delta}\right)} \\
&\qquad +
\frac{(k-1)\lceil \cC(f) \rceil B}{\Delta \sqrt{l}}
\Bigg(355 \sqrt{\log(l)(q\log(2) + \log(p))} \\
&\qquad\qquad\qquad\qquad\qquad\qquad\qquad+ 16.5\sqrt{\log\left(\fr{2\sqrt{3}(\lceil \cC(f) \rceil+1)(|\log (\Delta/B)|+2)}{\delta}\right)}\Bigg).
\end{aligned}
\end{equation}

Choosing $\Delta := 0.025 \min_{i\neq j \in [l]}\|\mu_f(\tS_i)-\mu_f(\tS_j)\| = 0.025 \Lambda$, the first part of Lemma~\ref{prop:cdnv_error} implies
\begin{equation}\label{eq:avg_ineq}
\Avg_{i\neq j \in [l]}[\ell_{4\Delta}(f;\tS_i,\tS_j)]\leq (12+\fr{212}{n})\cdot \Avg_{i\neq j \in [l]}[V_f(\tS_i,\tS_j)].
\end{equation}
Combining with \eqref{eq:transfer_risk_avg} gives the desired bound
\begin{equation}
\label{eq:thm6_consts}
\begin{aligned}
\mathcal{L}_{\cD}(f) ~&\le~ (k-1) \left(1+\frac{n}{m-n}\right)\left(48+\frac{848}{n}\right) \Avg_{i\neq j \in [l]}[V_f(\tS_i,\tS_j)]
+ \frac{(k-1)m\delta^2}{m-n} \\
&\qquad+ 
320 (k-1)\lceil \cC(f) \rceil B\frac{\sqrt{m}(2n+4)^2}{(m-n)\Lambda}
\sqrt{\pi(q\log(2)+\log(p))\log(2m)} \\
&\qquad+ 
640 (k-1)\lceil \cC(f) \rceil B \frac{\sqrt{m}}{(m-n)\Lambda}\sqrt{2 \log\left(\fr{6l^2(\lceil \cC(f) \rceil+1)^2 (|\log(\fr{\Lambda}{40B})|+2)^2}{\delta}\right)} \\
&\qquad +
\frac{(k-1)\lceil \cC(f) \rceil B}{\Lambda \sqrt{l}}
\Bigg(14200 \sqrt{\log(l)(q\log(2) + \log(p))} \\
&\qquad\qquad\qquad\qquad\qquad\qquad
+ 660\sqrt{\log\left(\fr{2\sqrt{3}(\lceil \cC(f) \rceil+1)(|\log (\fr{\Lambda}{40B})|+2)}{\delta}\right)}\Bigg),
\end{aligned}
\end{equation}
which proves Theorem~\ref{thm:full_bound}.
The proofs of Theorem~\ref{thm:full_bound1} and Theorem~\ref{thm:full_bound2} follow the same lines, but instead of \eqref{eq:avg_ineq}, we use the first part of Lemma~\ref{prop:cdnv_errorE} with $\Delta=\phi \Lambda/8$ for Theorem~\ref{thm:full_bound1} and the second part of the same lemma with $\Delta=\Lambda/\sqrt{n}$ for Theorem~\ref{thm:full_bound2}.

%%%%%%%%%%%%%%%%%%%%%%%%%%%%%%%%%%%%%%%%%%%%%%%%%%%%%%%%%%%%%%%
\section{Experiments}\label{sec:experiments}
%%%%%%%%%%%%%%%%%%%%%%%%%%%%%%%%%%%%%%%%%%%%%%%%%%%%%%%%%%%%%%%

In this section, we conduct experiments to study the neural collapse phenomenon and how it generalizes to new data points and new classes. We observe neural collapse not only during training, as previously reported by \citet{Papyan24652}, but also on test data from the same classes and on data from new classes (this generalization of the neural collapse phenomenon to new data was formally proved in Lemma~1 of our previous work, \citealp{galanti2022on}). Our results show that neural collapse is strongly correlated with accuracy in few-shot learning scenarios. Our experiments are conducted on multiple datasets and using multiple standard architectures, providing strong empirical evidence that neural collapse provides a compelling explanation for the performance of foundation models in few-shot learning tasks. The experimental results are reported as the average over 20 random initializations, with 95\% confidence intervals.

\subsection{Setup}

\paragraph{Method.} In this study, we train a neural network classifier $\tih=\tig \circ f$ on a source task using stochastic gradient descent (SGD) with momentum \citep[see, e.g.,][]{SutskeverMDH13,YanYLLY18}, with a learning rate of $\eta$ (specified in each experiment), a momentum of 0.9, and a batch size of 64. The mapping implemented by all layers of the neural network except for the top linear layer $\tig$ is referred to as $f$.

To evaluate the few-shot performance of $f$ on target classification tasks, we train a new classifier $g$ on top of $f$ by minimizing the cross-entropy loss between its logits and the one-hot encodings of the labels. Formally, given a target few-shot classification task with training data $S=\cup^{k}_{c=1}\{x_{ci}\}_{i=1}^n$, we train a new top linear layer $g$ by solving a ridge regression problem on the dataset $\cup^{k}_{c=1}\{(f(x_{ci})\}^{n}_{i=1}$ with regularization $\lambda_n = \alpha\sqrt{kn}$. This results in the weight matrix $w_{f,S} = (f(X)^\top f(X)+\lambda_n I)^{-1} f(X)^\top Y$, where $f(X)=(f(x_{ci}))_{c,i} \in \R^{nc\times d}$ is the data matrix containing the feature vectors, $Y\in \mathbb{R}^{nc\times k}$ is the label matrix, and $X \in \mathbb{R}^{nc\times d}$ is the data matrix. We did not apply any fine-tuning to $f$ during this process.

To measure the performance of the model $f$, we randomly sample 5-class classification tasks (i.e., $k=5$) from the target dataset, with $n=1$ or $n=5$ training samples for each class, and evaluate the accuracy on 100 random test samples from each class. We report the resulting accuracy rates averaged over 100 randomly chosen tasks.

\paragraph{Architectures and hyperparameters.} We experimented with two types of architectures for $\tih$: wide ResNets~\citep{Zagoruyko2016WideRN} and vanilla convolutional networks. The wide ResNets consist of a convolutional layer with $3 \times 3$ kernels and 16 output channels, followed by three groups of layers. Each group includes a convolutional layer followed by a series of $N$ residual layers. Each residual layer in the $i$th group contains two convolutional layers with $3\times 3$ kernels and $2^{i+3} M$ output channels. Each convolutional layer is followed by a ReLU activation and batch normalization. The penultimate layer of the network is a mean pooling activation with $4 \times 4$ kernels, producing an output of size 256, followed by a linear layer. The vanilla convolutional networks have the same structure as the wide ResNets, but do not include the residual connections. These networks are denoted as WRN-$N$-$M$ and Conv-$N$-$M$, respectively.

\paragraph{Datasets.} We evaluate our approach using four datasets: (i) Mini-ImageNet~\citep{NIPS2016_90e13578}, (ii) CIFAR-FS~\citep{bertinetto2018metalearning}, (iii) FC-100~\citep{NEURIPS2018_66808e32}, (iv) and EMNIST (balanced)~\citep{cohen2017emnist}. Each dataset is divided into three sets: meta-train, meta-validation, and meta-test. The meta-training and meta-test sets are used for training and evaluating our models, respectively. The meta-validation set is not used. The Mini-ImageNet dataset consists of 100 randomly selected classes from ImageNet ILSVRC-2012~\citep{ILSVRC15}, with 600 images of size 84 x 84 pixels per class. It is divided into 64 meta-train classes, 16 meta-validation classes, and 20 meta-test classes. CIFAR-FS and FC-100 are both derived from the CIFAR-100 dataset~\citep{Krizhevsky09learningmultiple}. CIFAR-FS consists of a random split of the CIFAR-100 classes into 64 meta-train classes, 14 meta-validation classes, and 20 meta-test classes. FC-100 contains 100 classes grouped into 20 superclasses. These classes are divided into 60 meta-train classes from 12 superclasses, 20 meta-validation classes from 4 superclasses, and 20 meta-test classes from 4 superclasses. The EMNIST dataset is an extension of the original MNIST dataset. We randomly split its classes into 35 meta-train classes and 12 meta-test classes (which are: 2, 10, 11, 12, 13, 16, 18, 22, 25, 33, 34, 44). We apply random cropping augmentations during training for all experiments.

\subsection{Results}

\begin{table}[t]
\centering 
\resizebox{\linewidth}{!}{
\begin{tabular}{c c c c c c c c}
    \hline
Method & Architecture & \multicolumn{2}{c}{Mini-ImageNet}  & \multicolumn{2}{c}{CIFAR-FS} & \multicolumn{2}{c}{FC-100} \\
 &  & 1-shot & 5-shot  & 1-shot & 5-shot & 1-shot & 5-shot \\
    \hline
Matching Networks~\citep{NIPS2016_90e13578} & 64-64-64-64 & $43.56\pm 0.84$ & $55.31 \pm 0.73$ & - & - & - & - \\ 
LSTM Meta-Learner~\citep{Ravi2017OptimizationAA} & 64-64-64-64 & $43.44 \pm 0.77$ & $60.60\pm 0.71$ & - & - & - & - \\
MAML~\citep{pmlr-v70-finn17a} & 32-32-32-32 & $48.70\pm 1.84$ & $63.11\pm 0.92$ & $58.9 \pm 1.9$ & $71.5\pm 1.0$ & - & - \\
Prototypical Networks~\citep{NIPS2017_cb8da676} & 64-64-64-64 & $49.42\pm 0.78^\dagger$ & $68.20\pm 0.66^\dagger$ & $55.5 \pm 0.7$ & $72.0\pm 0.6$ & $35.3 \pm 0.6$ & $48.6\pm 0.6$ \\
Relation Networks~\citep{Sung_2018_CVPR} & 64-96-128-256 & $50.44\pm 0.82$ & $65.32\pm 0.7$ & $55.0 \pm 1.0$ & $69.3\pm 0.8$ & - & - \\
SNAIL~\citep{mishra2018a} & ResNet-12 & $55.71\pm 0.99$ & $68.88\pm 0.92$ & - & - & - & - \\
TADAM~\citep{NEURIPS2018_66808e32} & ResNet-12 & $58.50\pm 0.30$ & $76.7\pm 0.3$ & - & - & $ 40.1 \pm 0.4$ & $56.1\pm 0.4$\\
AdaResNet~\citep{pmlr-v80-munkhdalai18a} & ResNet-12 & $56.88\pm 0.62$ & $71.94\pm 0.57$ & - & - & - & - \\
Dynamics Few-Shot~\citep{Gidaris_2018_CVPR} &  64-64-128-128 & $56.20\pm 0.86$ & $73.0\pm 0.64$ & - & - & - & - \\
Activation to Parameter~\citep{Act2Param} & WRN-28-10 & $ 59.60 \pm 0.41^\dagger$ & $73.74\pm 0.19^\dagger$ & - & - & - & - \\
R2D2~\citep{bertinetto2018metalearning} & 96-192-384-512 & $51.2\pm 0.6$ & $68.8\pm 0.1$ & $65.3 \pm 0.2$ & $79.4\pm 0.1$ & - & - \\
Shot-Free~\citep{Ravichandran2019FewShotLW} & ResNet-12 & $59.04\pm n/a$ & $77.64\pm n/a$ & $69.2 \pm n/a$ & $84.7\pm n/a$ & - & - \\
TEWAM~\cite{Qiao_2019_ICCV} & ResNet-12 & $60.07\pm n/a$ & $75.90\pm n/a$ & $70.4 \pm n/a$ & $81.3\pm n/a$ & - & - \\
 TPN~\citep{liu2018learning} & ResNet-12 & $55.51 \pm 0.86$ & $75.64\pm n/a$ & - & - & - & - \\
 LEO~\citep{rusu2018metalearning} & WRN-28-10 &  $61.76\pm 0.08^\dagger$ & $77.59\pm 0.12^\dagger$ & - & - & - & - \\
 MTL~\citep{sun2019mtl} & ResNet-12 & $61.20\pm 1.80$ & $75.50 \pm 0.80$ & - & - & - & - \\
 OptNet-RR~\citep{lee2019meta} & ResNet-12 & $61.41 \pm 0.61$ & $77.88\pm 0.46$ & $72.6 \pm 0.7$ & $84.3\pm 0.5$ & $40.5 \pm 0.6$ & $57.6\pm 0.9$\\
 MetaOptNet~\citep{lee2019meta} & ResNet-12 & $62.64\pm 0.61$ & $78.63\pm 0.46$ & $72.0 \pm 0.7$ & $84.2\pm 0.5$ & $41.1 \pm 0.6$ & $55.3\pm 0.6$\\
 Transductive Fine-Tuning~\citep{Dhillon2020A} & WRN-28-10 &  $65.73 \pm 0.68$ & $78.40\pm 0.52$ & $76.58 \pm 0.68$ & $85.79\pm 0.5$ & $43.16 \pm 0.59$ & $57.57\pm 0.55$ \\
 Distill-simple~\citep{DBLP:conf/eccv/TianWKTI20} & ResNet-12 & $62.02\pm 0.63$ & $79.64\pm 0.44$ & $71.5 \pm 0.8$ & $86.0\pm 0.5$ & $42.6 \pm 0.7$ & $59.1\pm 0.6$\\
 Distill~\citep{DBLP:conf/eccv/TianWKTI20} & ResNet-12 & $64.82\pm 0.60$ & $82.14\pm 0.43$ & $73.9 \pm 0.8$ & $86.9\pm 0.5$ & $44.6 \pm 0.7 $ & $60.9\pm 0.6$\\
 \hline
 Ours (simple) & WRN-28-4 & $58.12 \pm 1.19$ & $72.0 \pm 0.99$ & $68.81\pm 1.20$ & $81.49\pm 0.98$ & $44.96 \pm 1.14$ & $57.21 \pm 10.89$ \\
 Ours (lr scheduling) & WRN-28-4 & $60.37 \pm 1.25$ & $72.35\pm 0.99$ & $70.0 \pm 1.29$ & $81.39\pm 0.96$ & $43.42 \pm 1.0$ & $54.14 \pm 1.1$ \\
 Ours (lr scheduling + model selection) & WRN-28-4 & $61.27 \pm 1.14$ & $74.74\pm 0.76$ & $72.37\pm 1.12$ & $82.94\pm 0.89$ & $45.81\pm 1.27$ & $56.85\pm 1.30$ \\
    \hline
\end{tabular}
}\vspace{-0.2cm}\caption{{\bf Comparison to prior work on Mini-ImageNet, CIFAR-FS, and FC-100 on 1-shot and 5-shot 5-class classification.} 
Reported numbers are test target classification accuracy of various methods for various data sets. The notation  a-b-c-d denotes a 4-layer convolutional network with a, b, c, and d filters in each layer. $^\dagger$The algorithm was trained on both training and validation classes. We took the data from \citet{Dhillon2020A} and \citet{DBLP:conf/eccv/TianWKTI20}.}
\label{tab:main_results}
\end{table}

\begin{figure}[t]
\centering
\begin{tabular}{@{}c@{~}c@{~}c@{~}c}
\includegraphics[width=.23\linewidth]{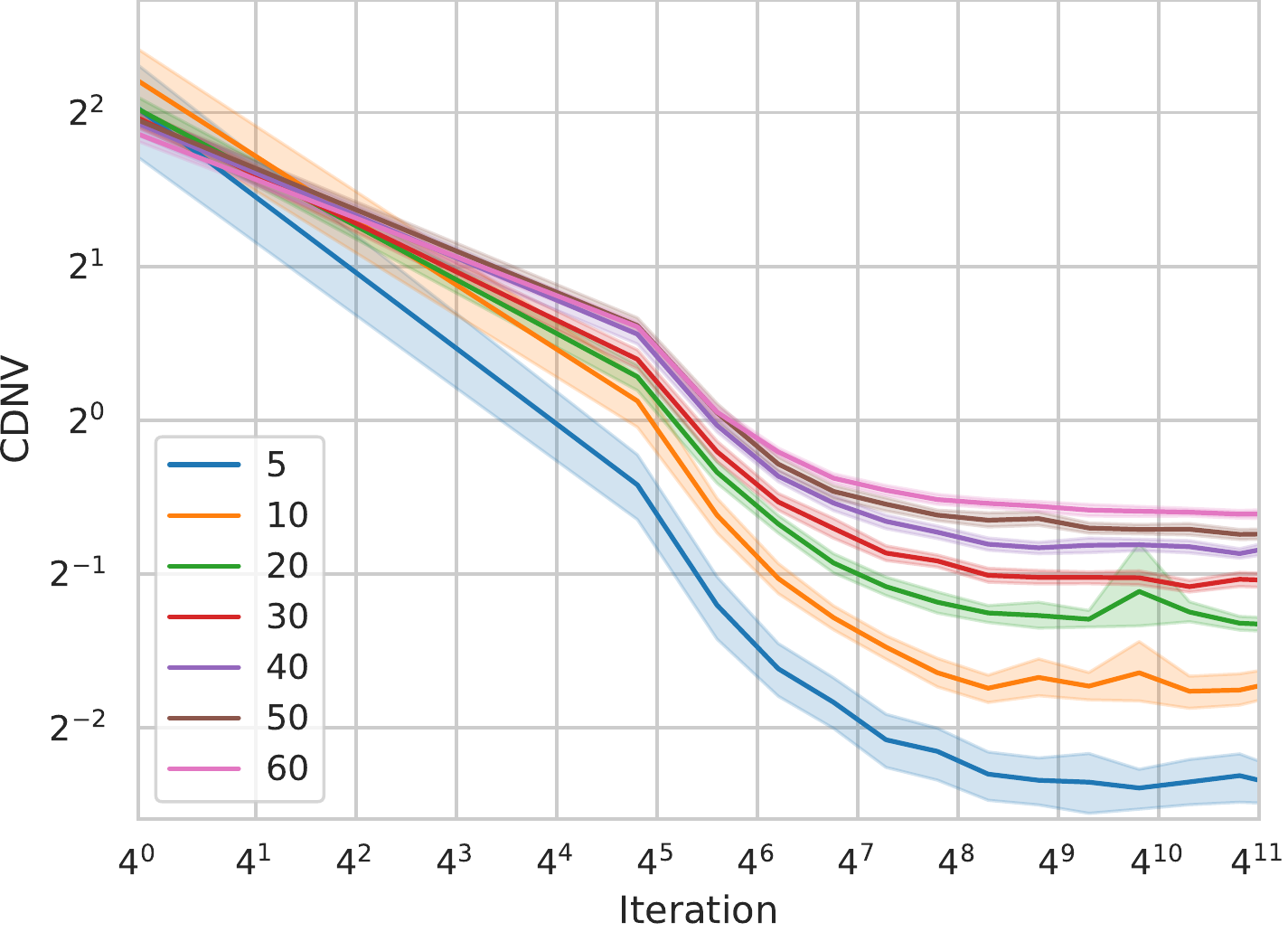}&
\includegraphics[width=.23\linewidth]{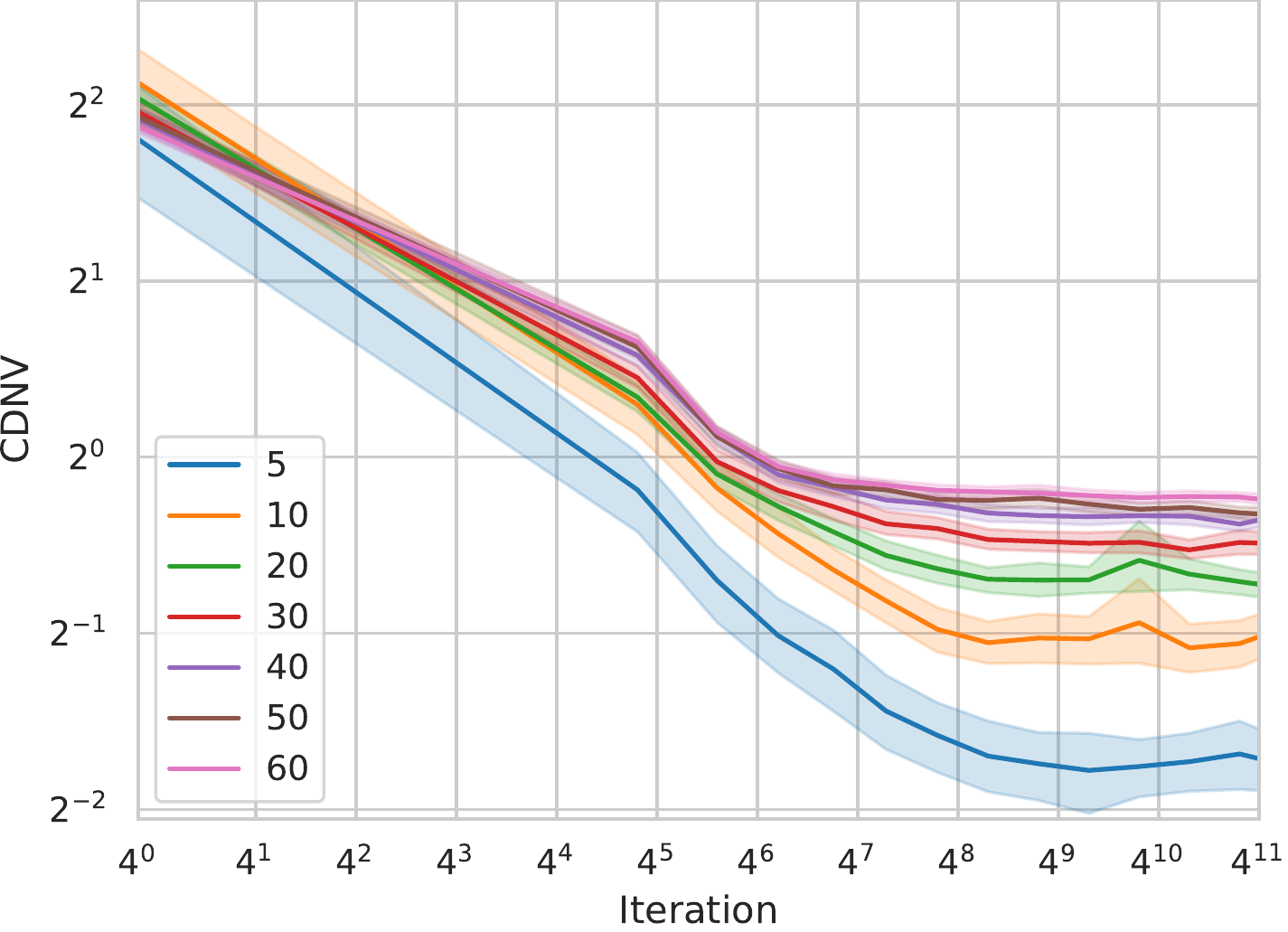}&
\includegraphics[width=.23\linewidth]{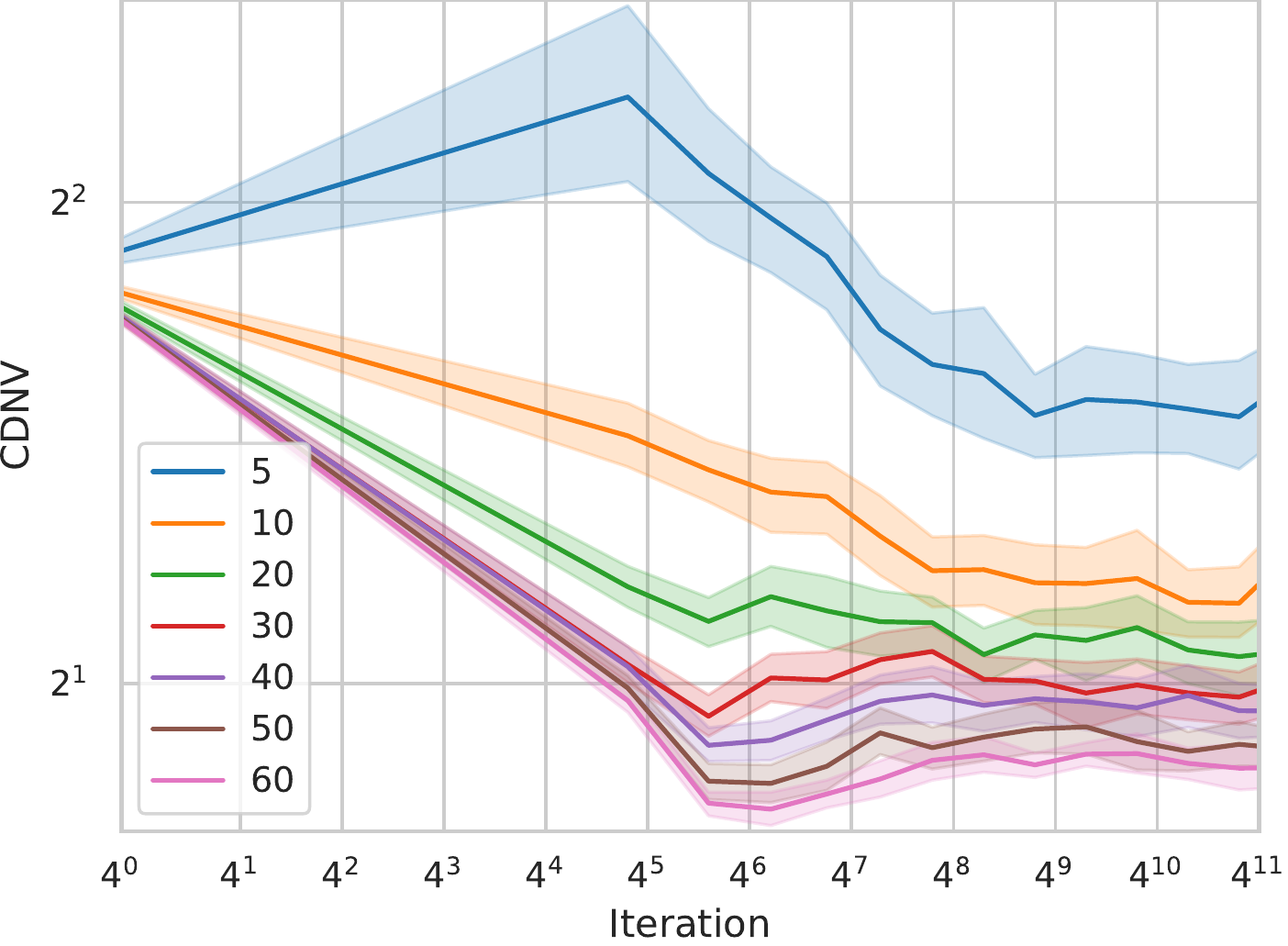}&
\includegraphics[width=.23\linewidth]{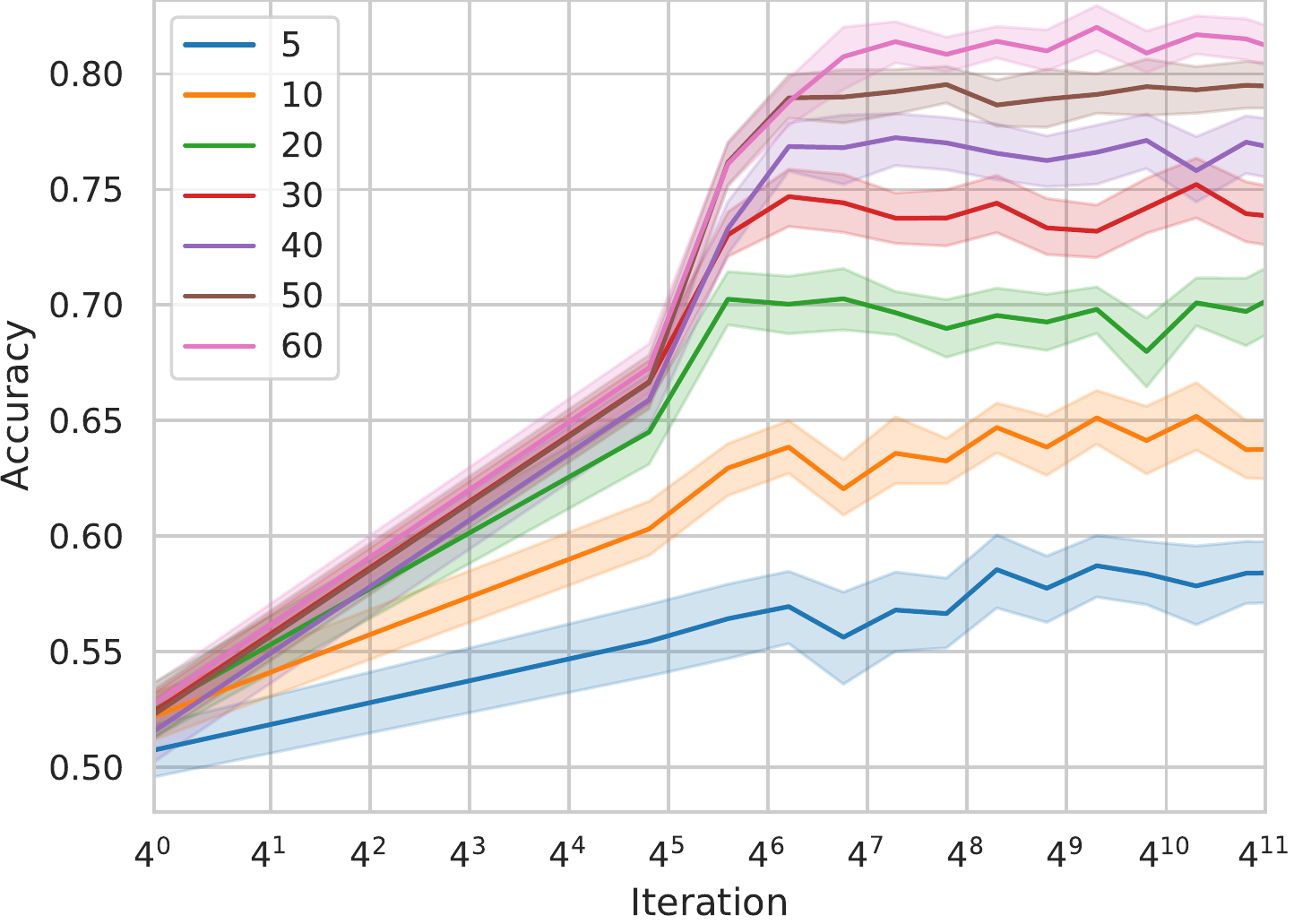}
\\
\multicolumn{4}{c}{\small{WRN-28-4 on CIFAR-FS}} \\[0.5em]
\includegraphics[width=.23\linewidth]{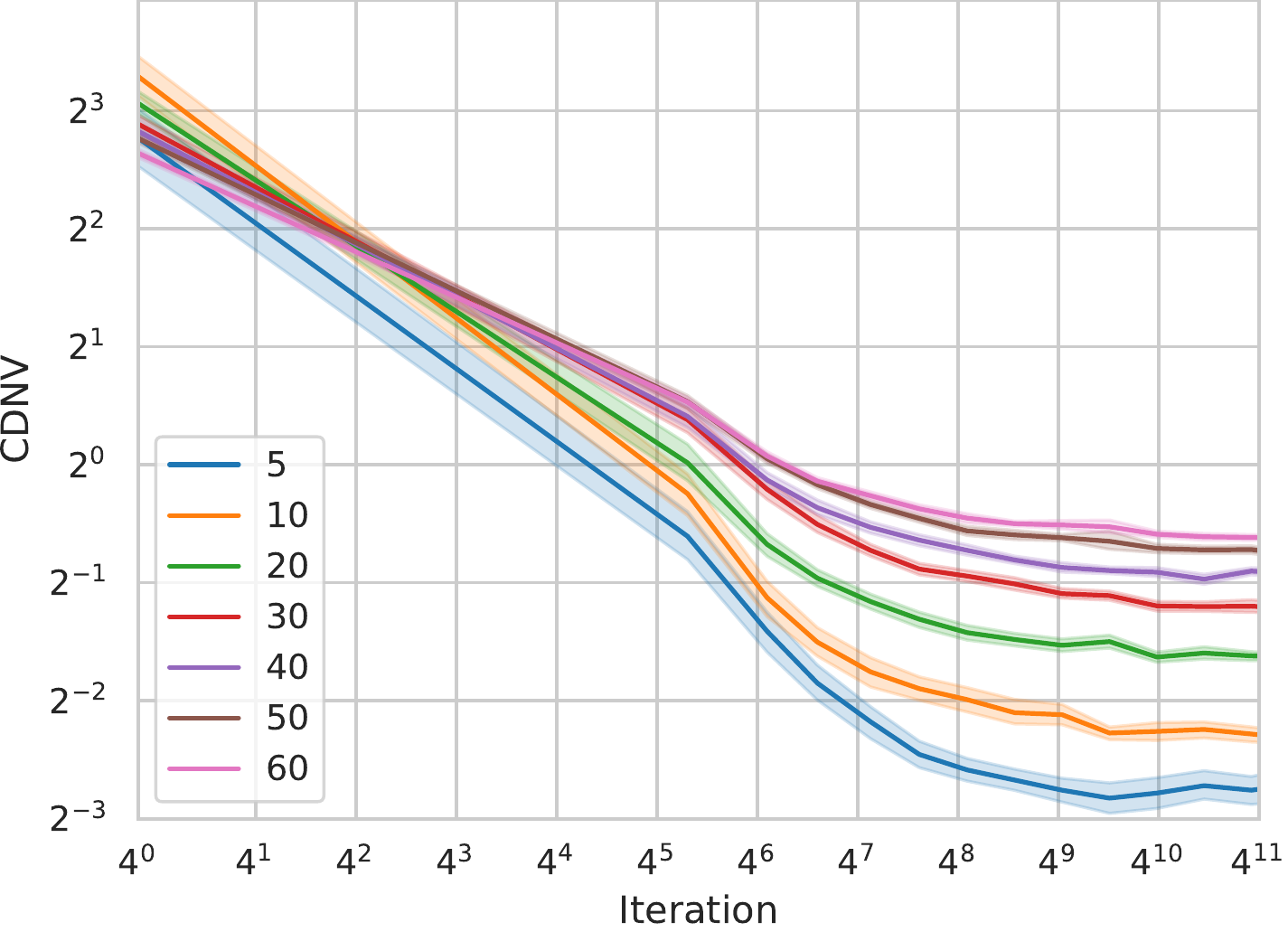}&
\includegraphics[width=.23\linewidth]{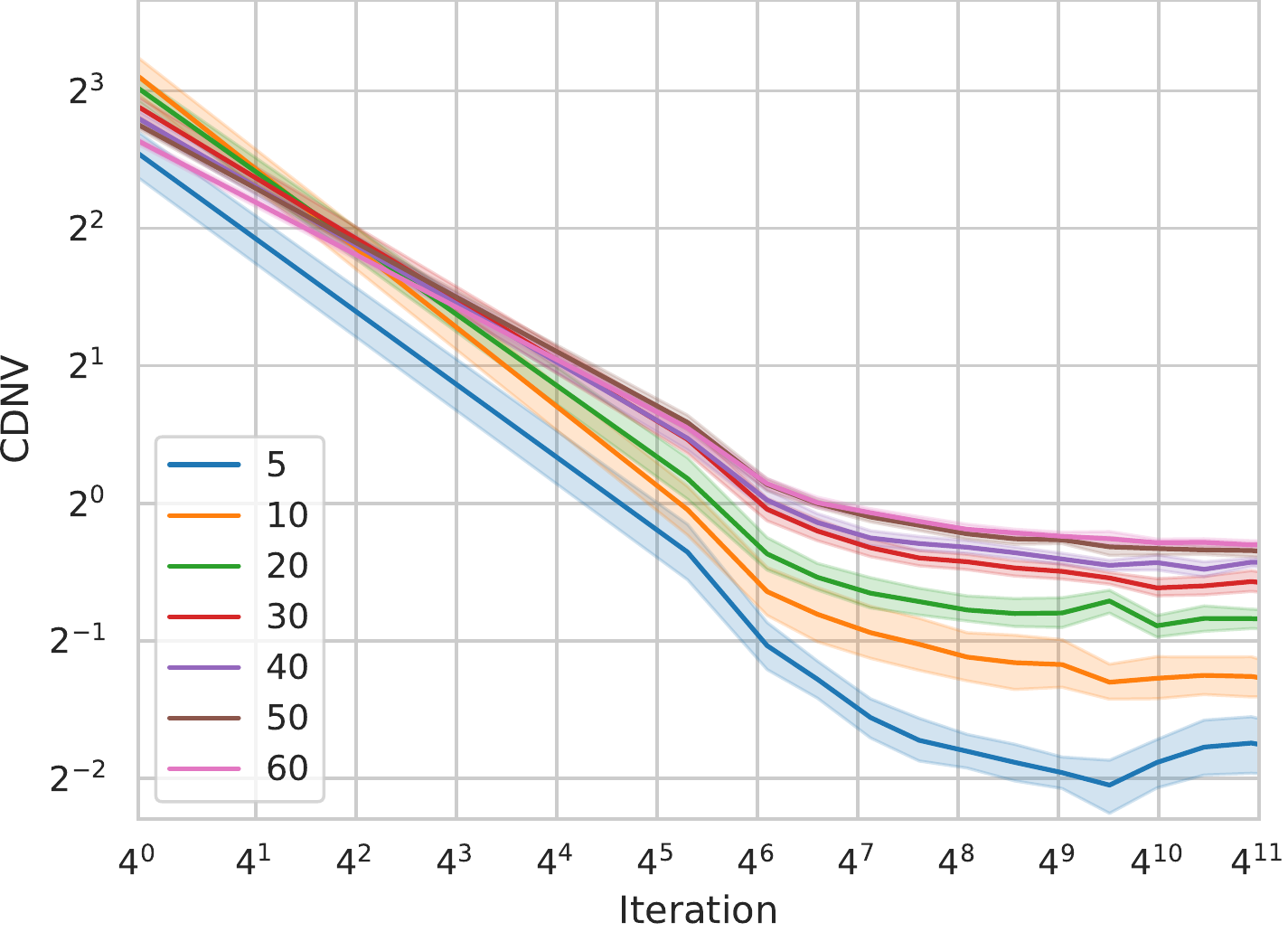}&
\includegraphics[width=.23\linewidth]{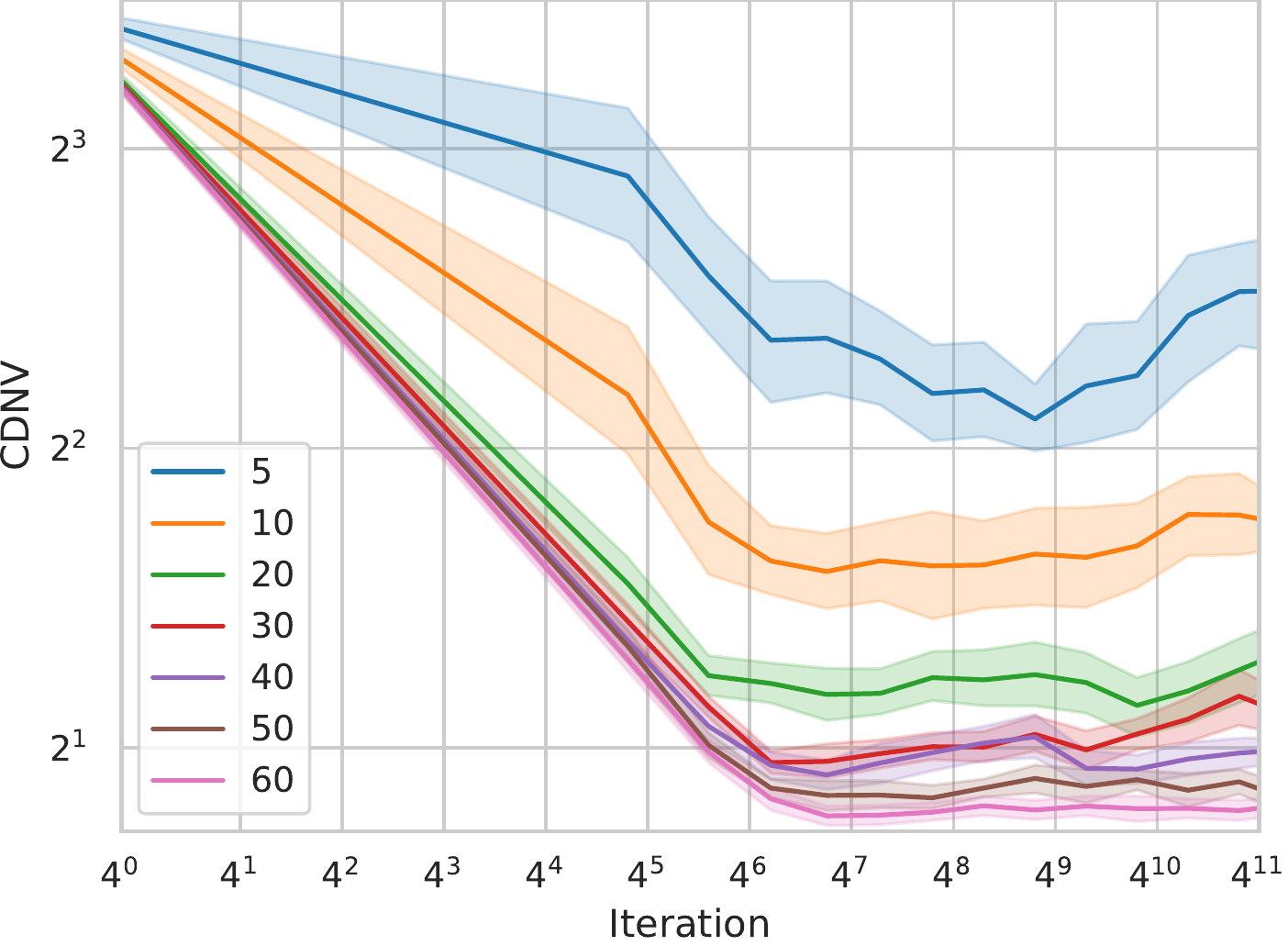}&
\includegraphics[width=.23\linewidth]{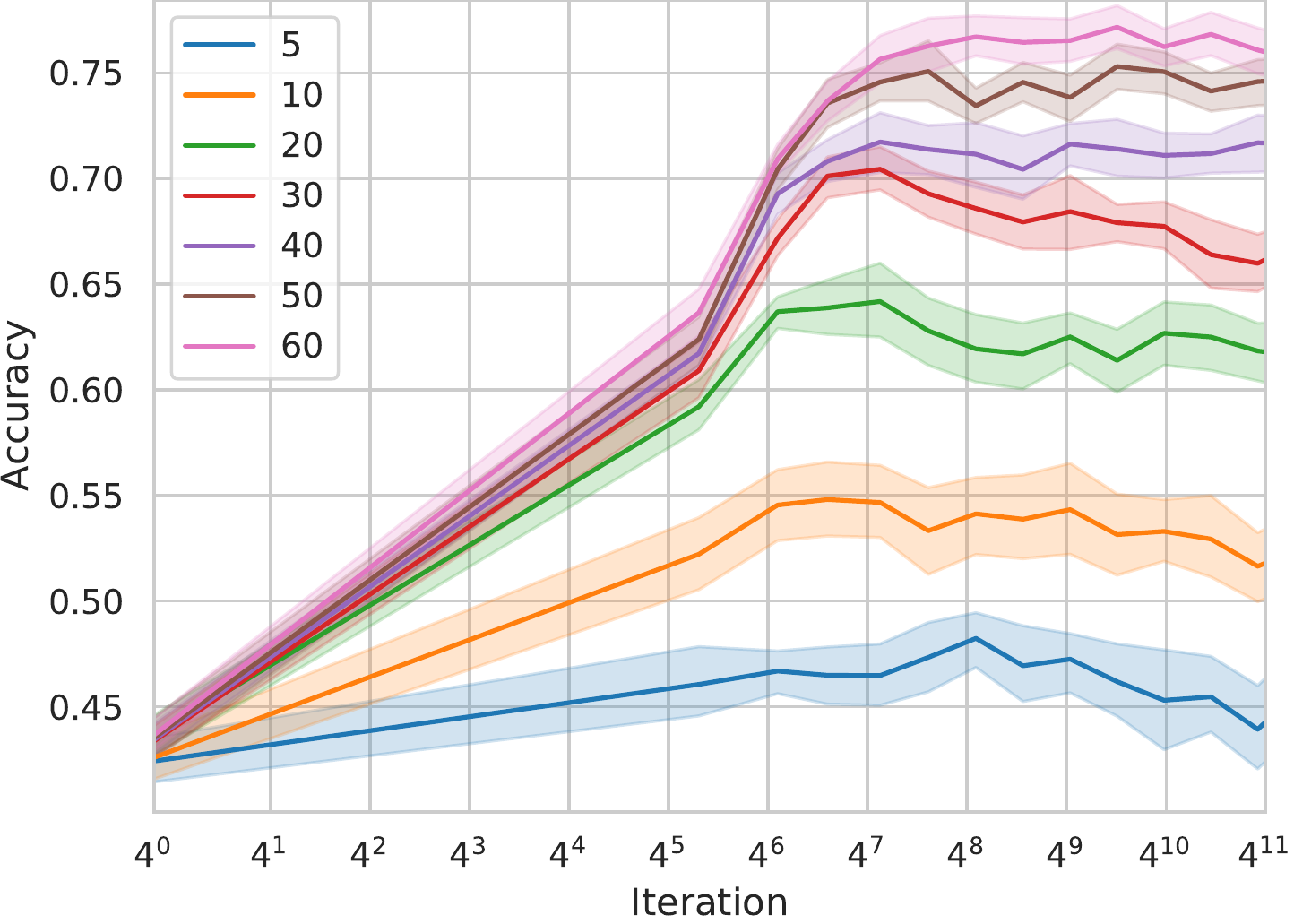}
\\
\multicolumn{4}{c}{\small{Conv-28-4 on CIFAR-FS}} \\[0.5em]
\includegraphics[width=.23\linewidth]{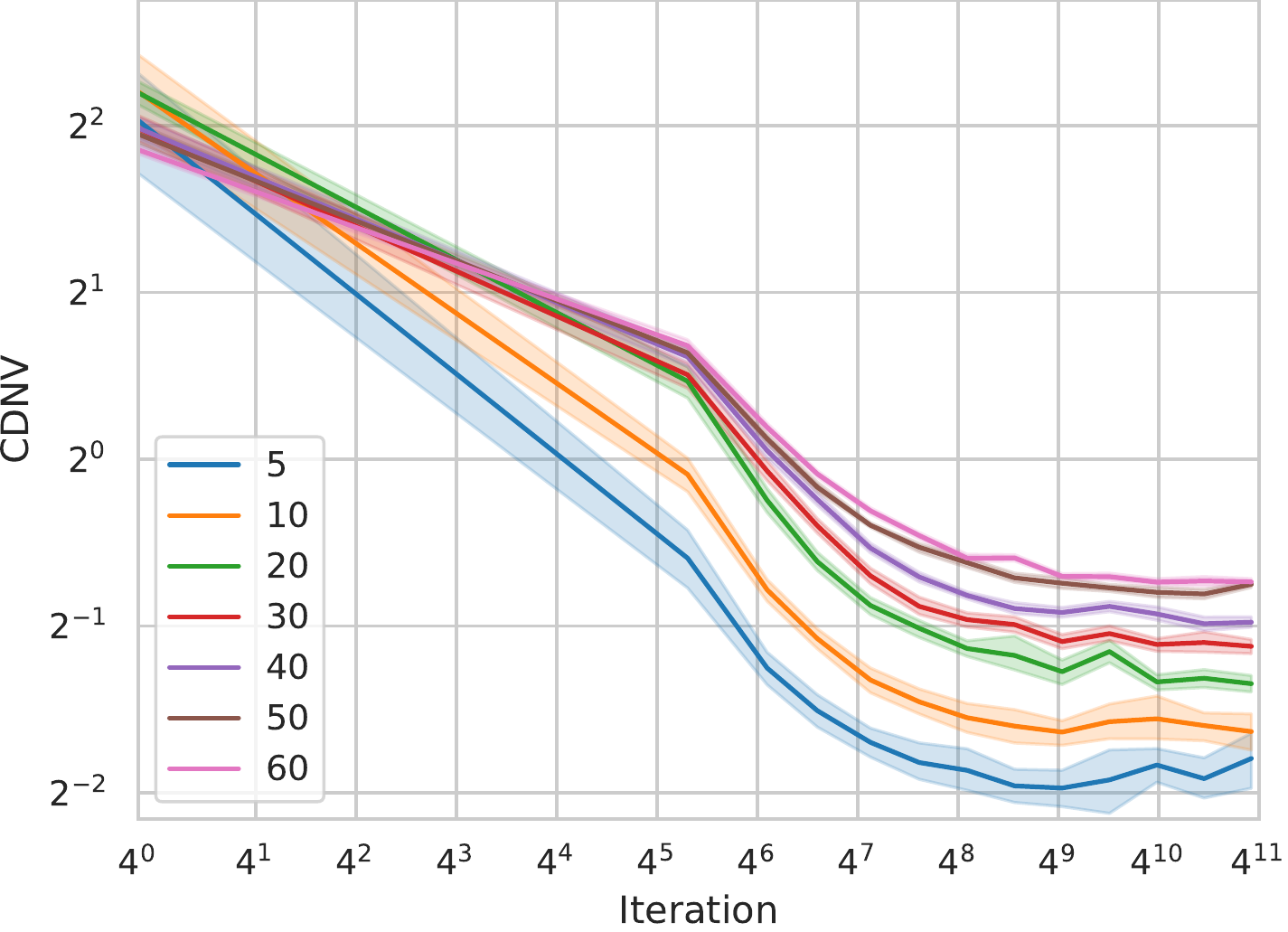}&
\includegraphics[width=.23\linewidth]{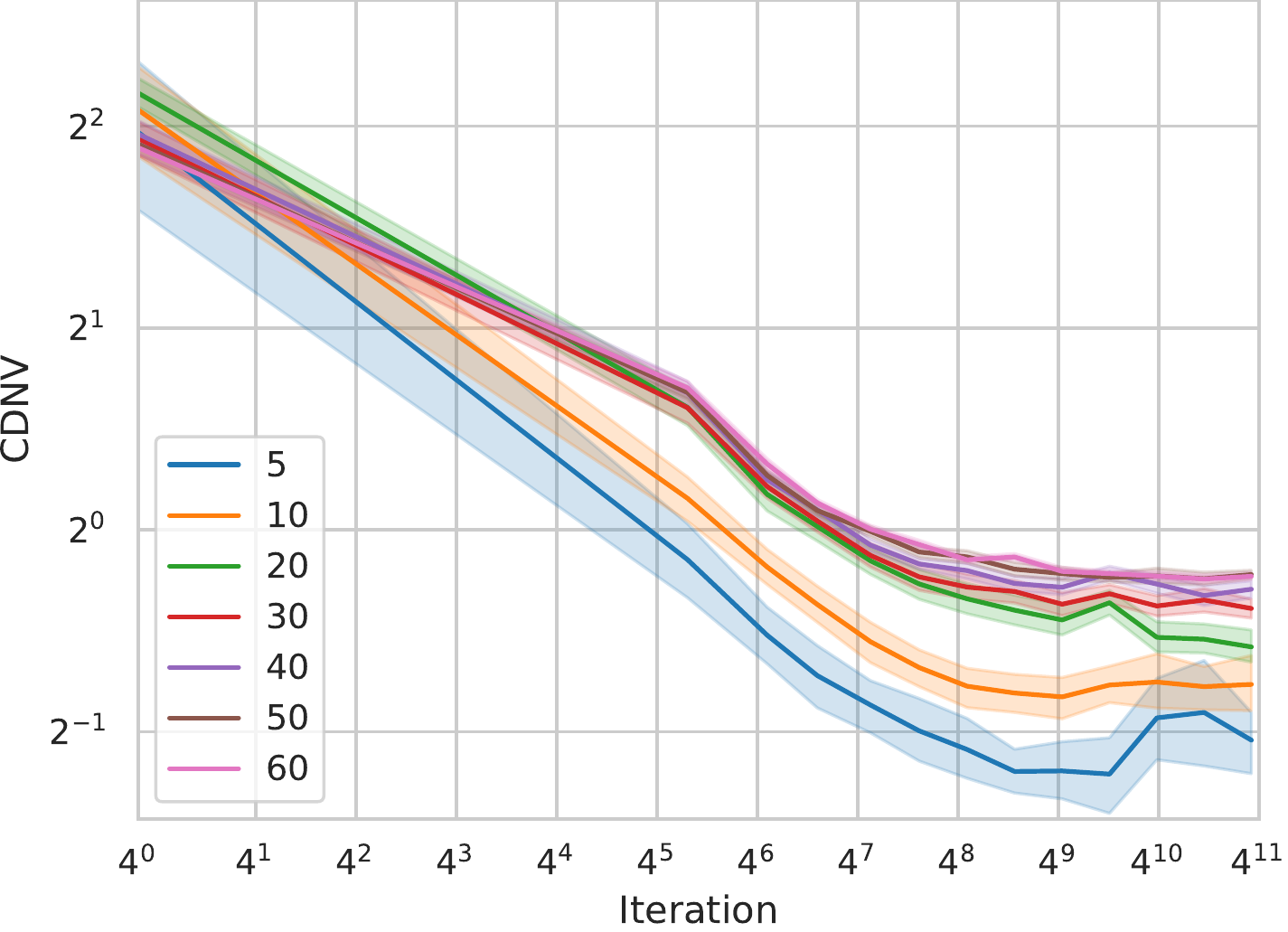}&
\includegraphics[width=.23\linewidth]{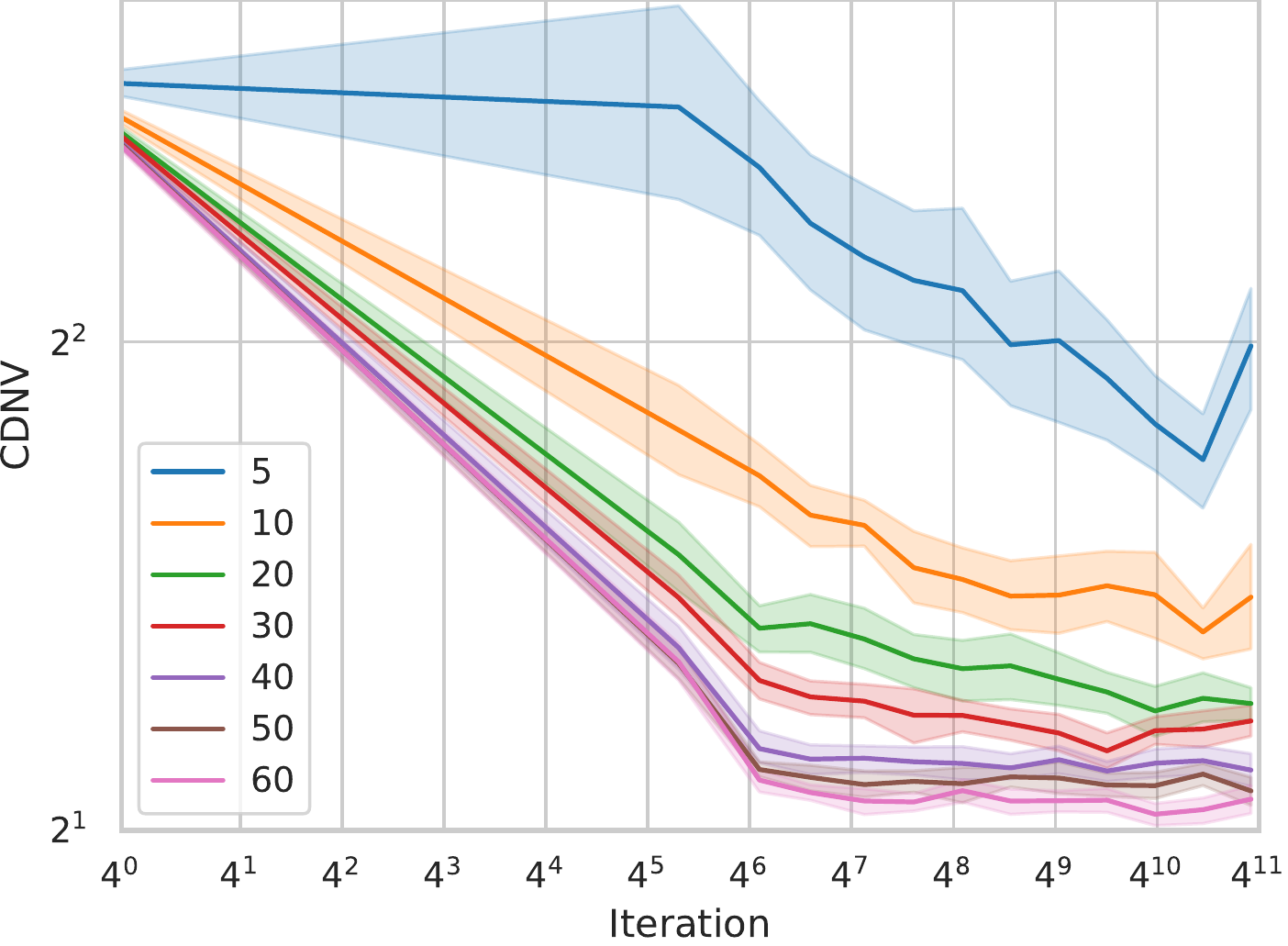}&
\includegraphics[width=.23\linewidth]{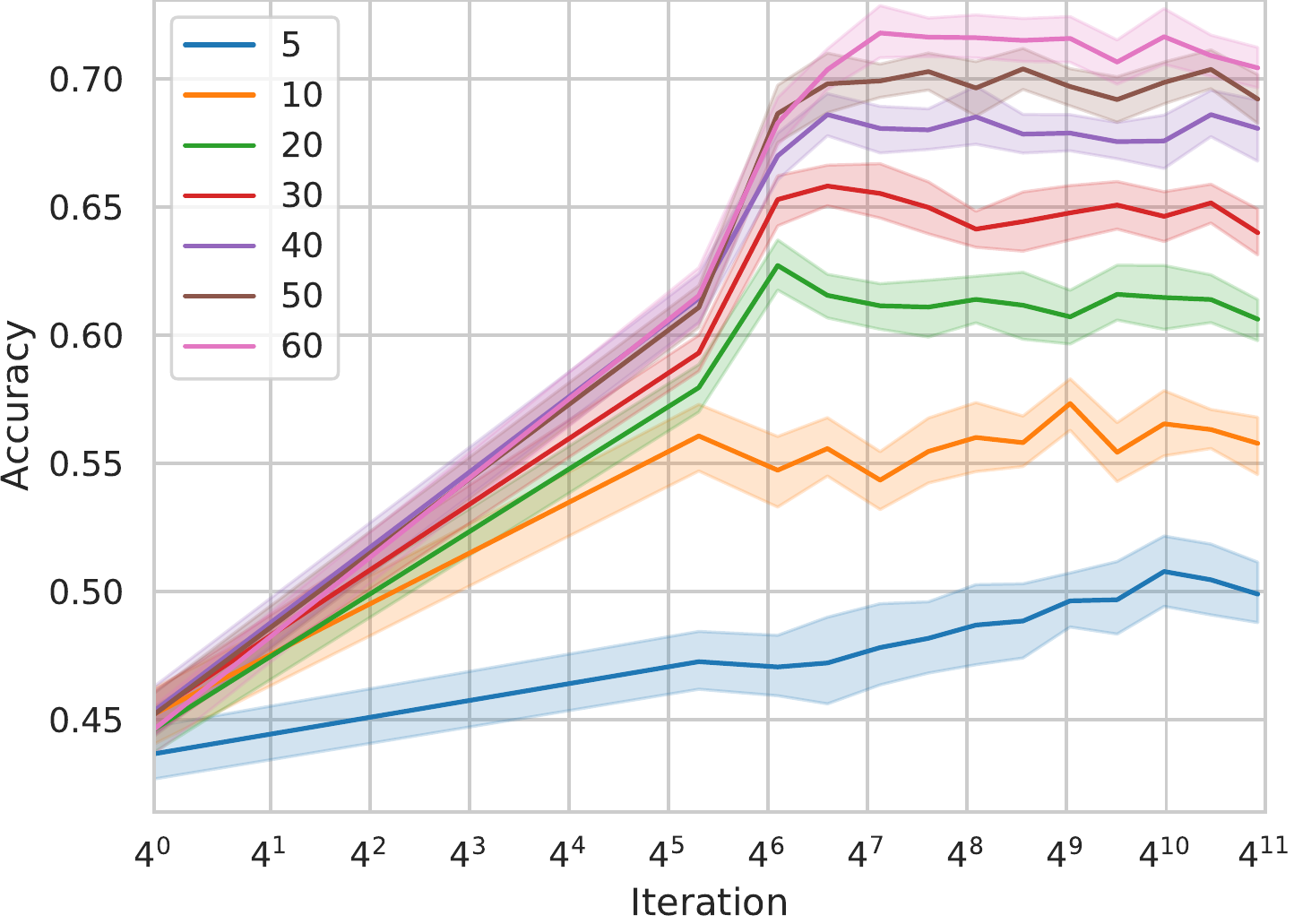}
\\
\multicolumn{4}{c}{\small{WRN-28-4 on Mini-ImageNet}} \\[0.5em]
\includegraphics[width=.23\linewidth]{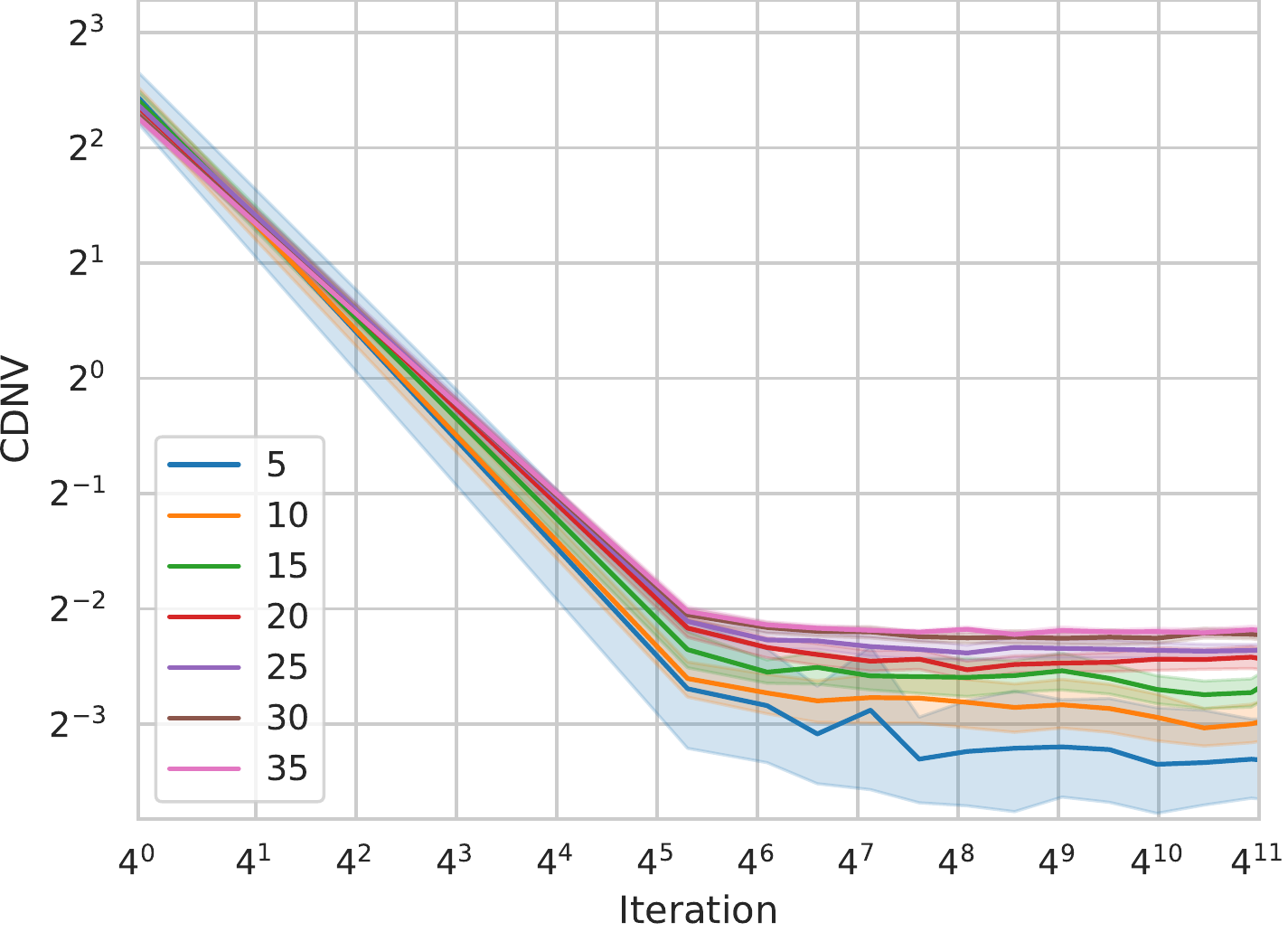}&
\includegraphics[width=.23\linewidth]{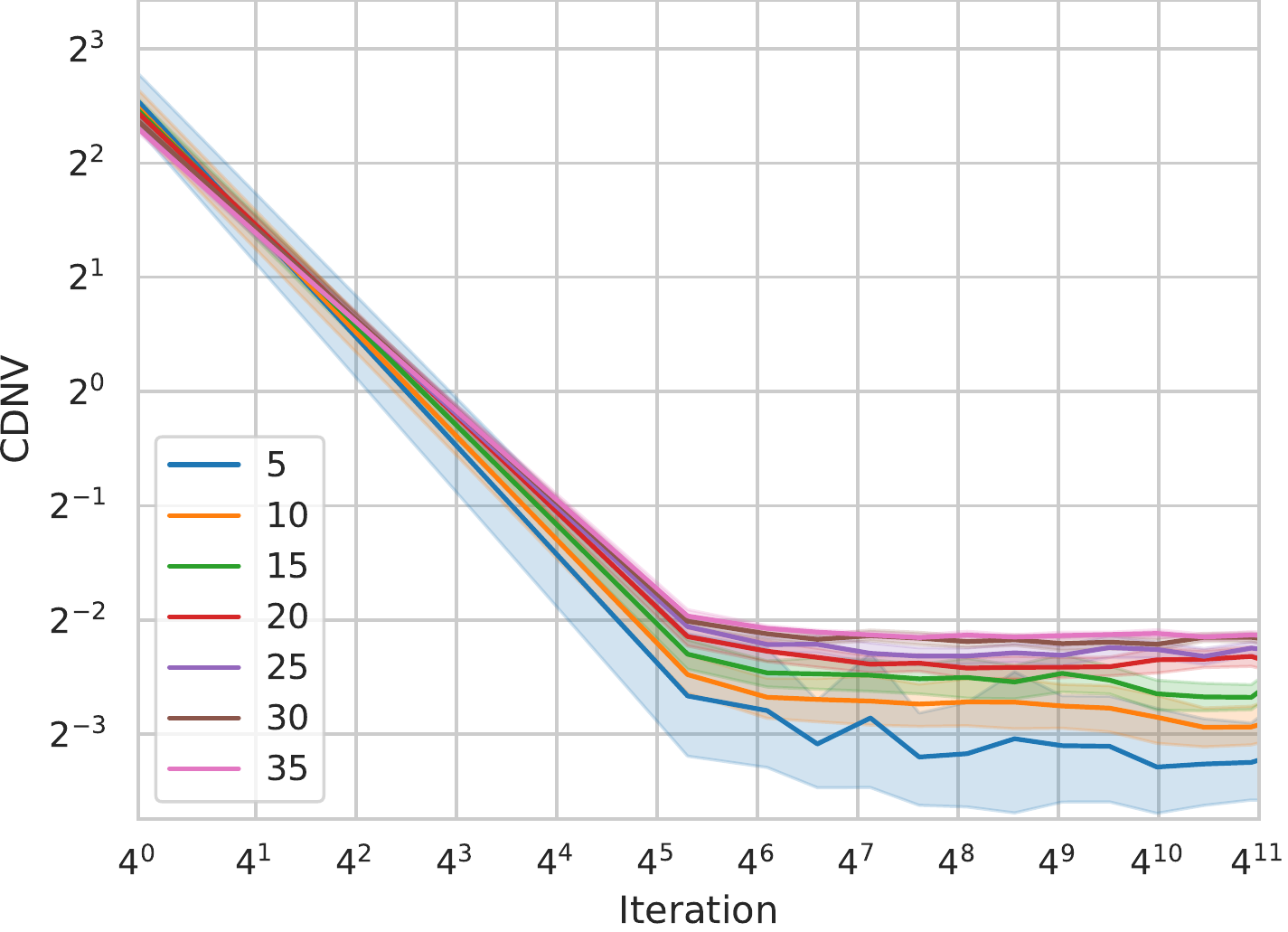}&
\includegraphics[width=.23\linewidth]{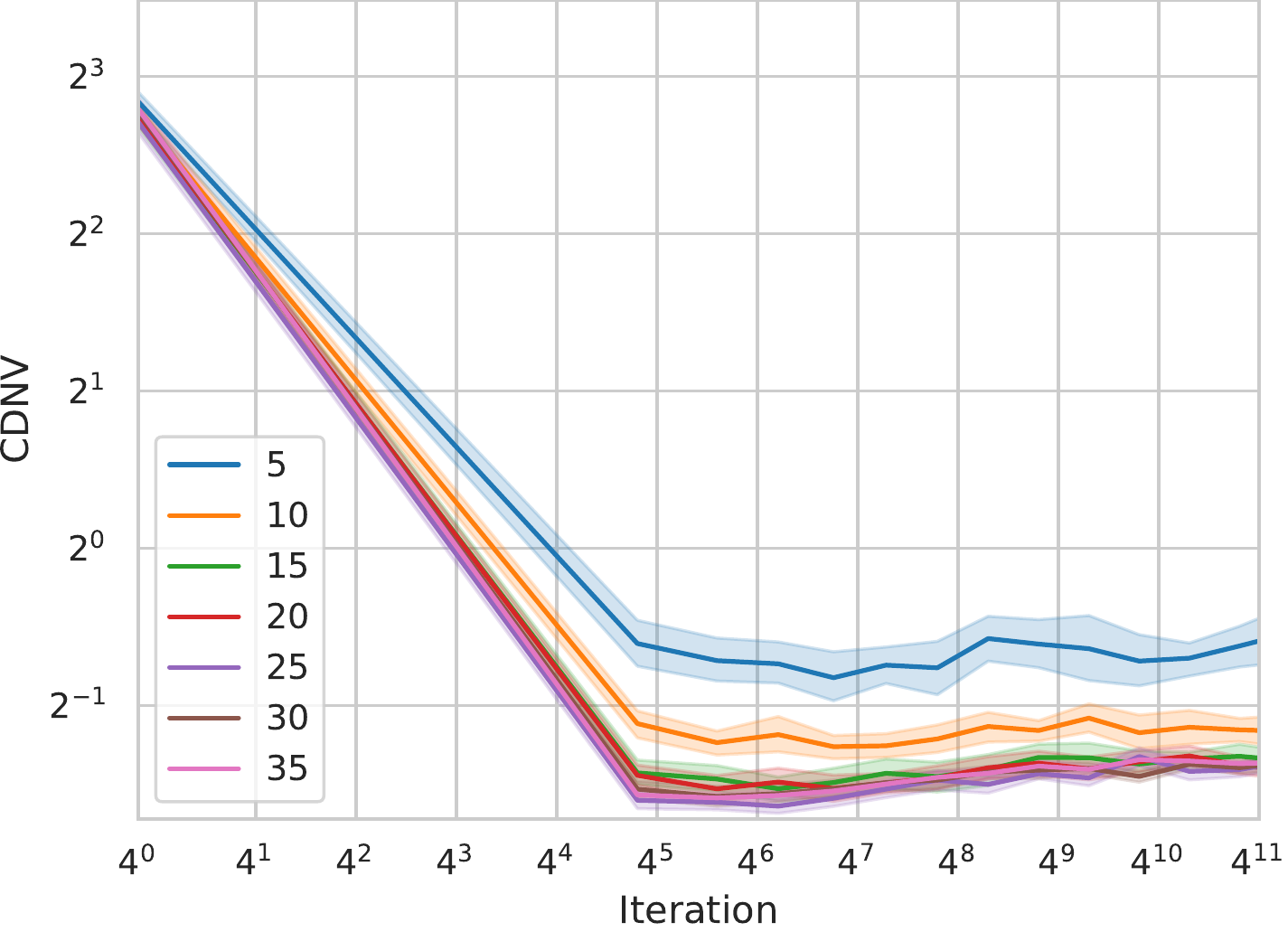}&
\includegraphics[width=.23\linewidth]{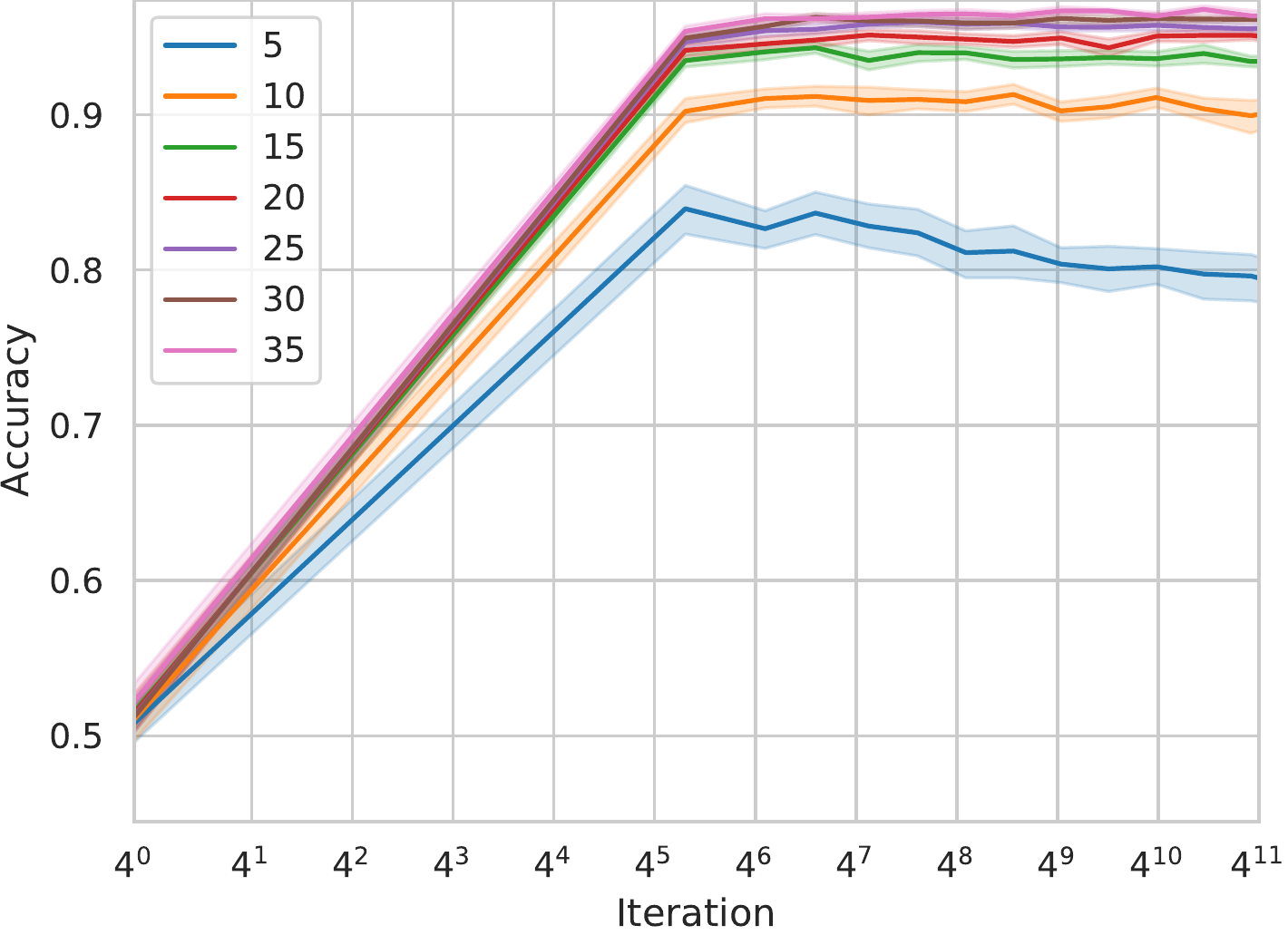}
\\
\multicolumn{4}{c}{\small{Conv-16-2 on EMNIST}}\\[0.5em]
\small{{\bf(a)} training CDNV} & \small{{\bf(b)} test CDNV} & \small{{\bf(c)} target CDNV} & \small{{\bf(d)} target accuracy}
\end{tabular}
    \caption{{\bf Class-features variability collapse.} {\bf (a)} CDNV on the source training data; {\bf (b)} CDNV over the source test data; {\bf (c)} CDNV over the target classes, all plotted in log-log scale. {\bf (d)} Target accuracy rate (lin-log scale).
    In each experiment we trained the model on a different number of source classes $l \in \{5,10,20,30,40,50,60\}$ (as indicated in the legend). 
    } \label{fig:var_nc_main}
\end{figure}

\paragraph{Few-shot performance of transfer learning.} While the use of ridge regression to train the few-shot classifier $g$ may seem simple, it has been shown to provide reasonable performance and is therefore suitable for evaluating the effectiveness of recent transfer/transductive learning methods for few-shot learning~\citep{Dhillon2020A,DBLP:conf/eccv/TianWKTI20}. To demonstrate this, we compared the 1- and 5-shot performance of our approach to several other few-shot learning algorithms on the Mini-ImageNet, CIFAR-FS, and FC-100 datasets, as summarized in Table~\ref{tab:main_results}. For each dataset, we report the average performance of our method on epochs 90 to 100. As shown in the table, our method performs competitively with other approaches on the three benchmarks, particularly in the 1-shot case (achieving state-of-the-art performance on FC-100). Note that the inferior performance of our method compared to Distill-simple~\citep{DBLP:conf/eccv/TianWKTI20} is likely due to the choice of the final few-shot classifier: our ridge regression classifier is not as effective as their logistic regression solution (as shown in Table~\ref{tab:main_results}), but is superior to their nearest neighbor classifier (see Table5 in their paper).

To improve the performance of our method slightly, we employed a standard learning rate scheduling with an initial learning rate of $\eta=0.05$, which is decayed twice by a factor of 0.1 after 30 epochs each (accuracy rates are reported as the average over epochs 90 to 100, as before). Since the performance of these networks plateaued slightly after the first learning rate decay on the source test data, we also applied model selection based on this information and used the network from the first 60 epochs (to avoid overfitting to the source data and classes at the smallest learning rate) with the best source test performance. This combination of variations typically resulted in a small improvement of a few percentage points on the problems considered, as shown in the last line of Table~\ref{tab:main_results}.

\paragraph{Neural collapse and few-shot learnability.} In our main experiment, we test the theoretical conclusions we made in Section~\ref{sec:analysis}. We argued that in the presence of neural collapse on the training classes, the trained feature map can be used to train a classifier with a small number of samples on new, unseen classes. In this section, we demonstrate that neural collapse generalizes to a certain degree to new samples from the same classes and also to new classes, and we show that it is correlated with good few-shot performance.

To validate these conclusions, we trained classifiers $\tih$ with a varying number of $l$ randomly selected source classes. For each run, we plot the CDNV as a function of epoch for the training and test datasets of the source classes and the test samples of the target classes in Figure~\ref{fig:var_nc_main}. In addition, we plot the 5-shot accuracy of ridge regression using the learned feature map $f$. Similar experiments with different numbers of target samples are shown in Figures~\ref{fig:multi_shots} and~\ref{fig:lr_scheduling}.

As shown in Figure~\ref{fig:var_nc_main}, the CDNV decreases over the course of training on both training and test datasets of the source classes, showing that neural collapse generalizes to new samples from the training classes. Since the classification tasks with fewer source classes are easier to learn, the CDNV tends to be larger when training with a wider set of source classes. In contrast, we observe that when increasing the number of source classes, the presence of neural collapse in the target classes strengthens. This is in alignment with our theoretical observations (the more ``training'' classes, the better generalization to new classes), and the few-shot performance also consistently improves when the overall number of source classes is increased. Note that in practice the CDNV does not decrease to zero (due to the limited sample sizes), and plateaus at a value above 1 on the target classes since we use a relatively small number of source classes (at most 60). To validate the generality of our results, this phenomenon is demonstrated in several settings, e.g., using different network architectures and datasets in Figure~\ref{fig:var_nc_main}. As can be seen, the values of the CDNV on the target classes are relatively large compared to those on the source classes, except for the results on EMNIST. However, these values still indicate a reasonable few-shot learning performance, as demonstrated in the experiments. These results consistently validate our theoretical findings, that is, that neural collapse generalizes to new source samples, it emerges for new classes, and its presence immediately facilitates good performance in few-shot learning.

\subsection{Dynamics of the Class-Embedding Distances}\label{sec:distances}

\begin{figure}[t]
    \centering
    \begin{tabular}{@{}c@{~}c@{~}c@{~}c}
\includegraphics[width=.23\linewidth]{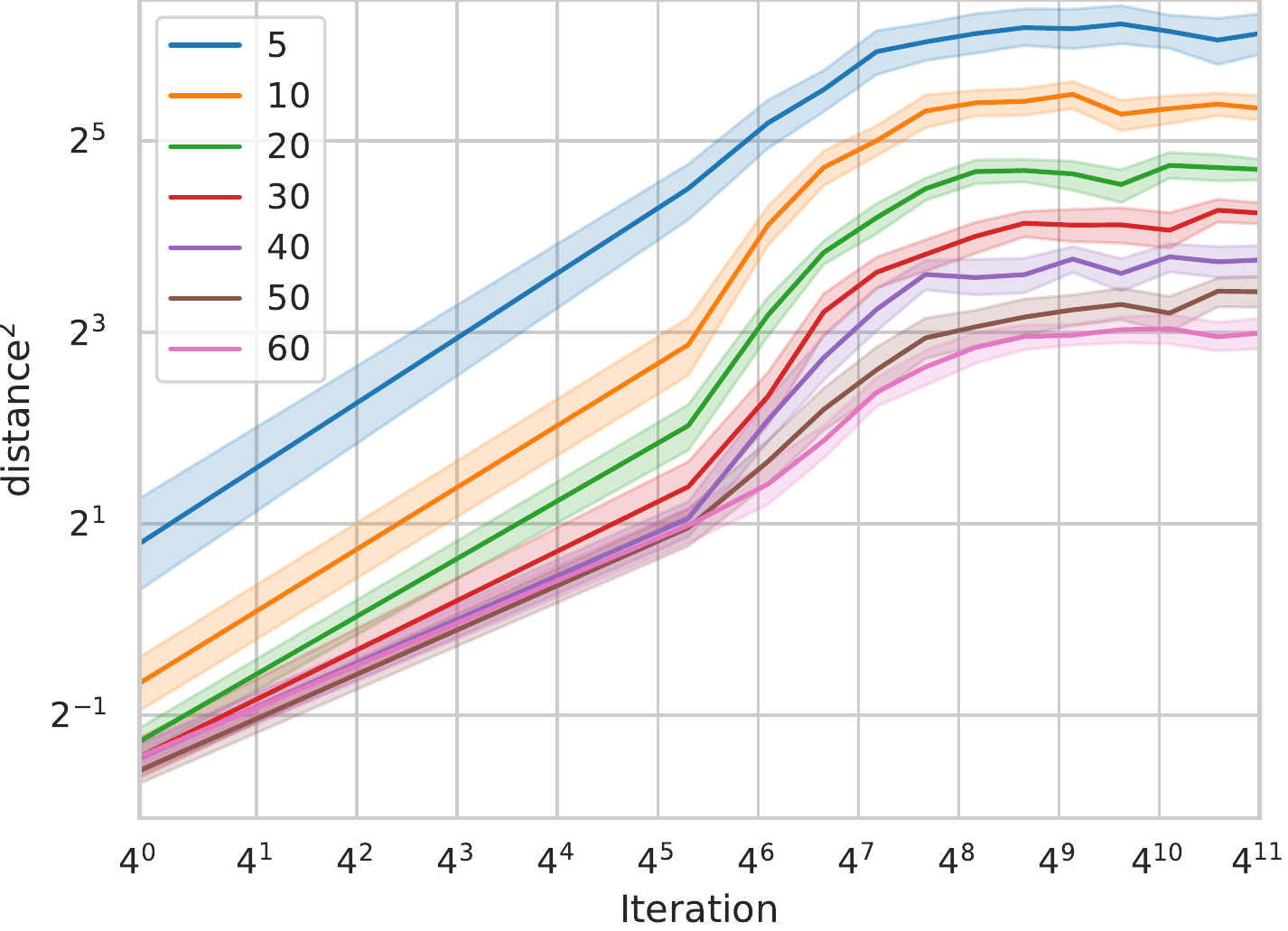}&
\includegraphics[width=.23\linewidth]{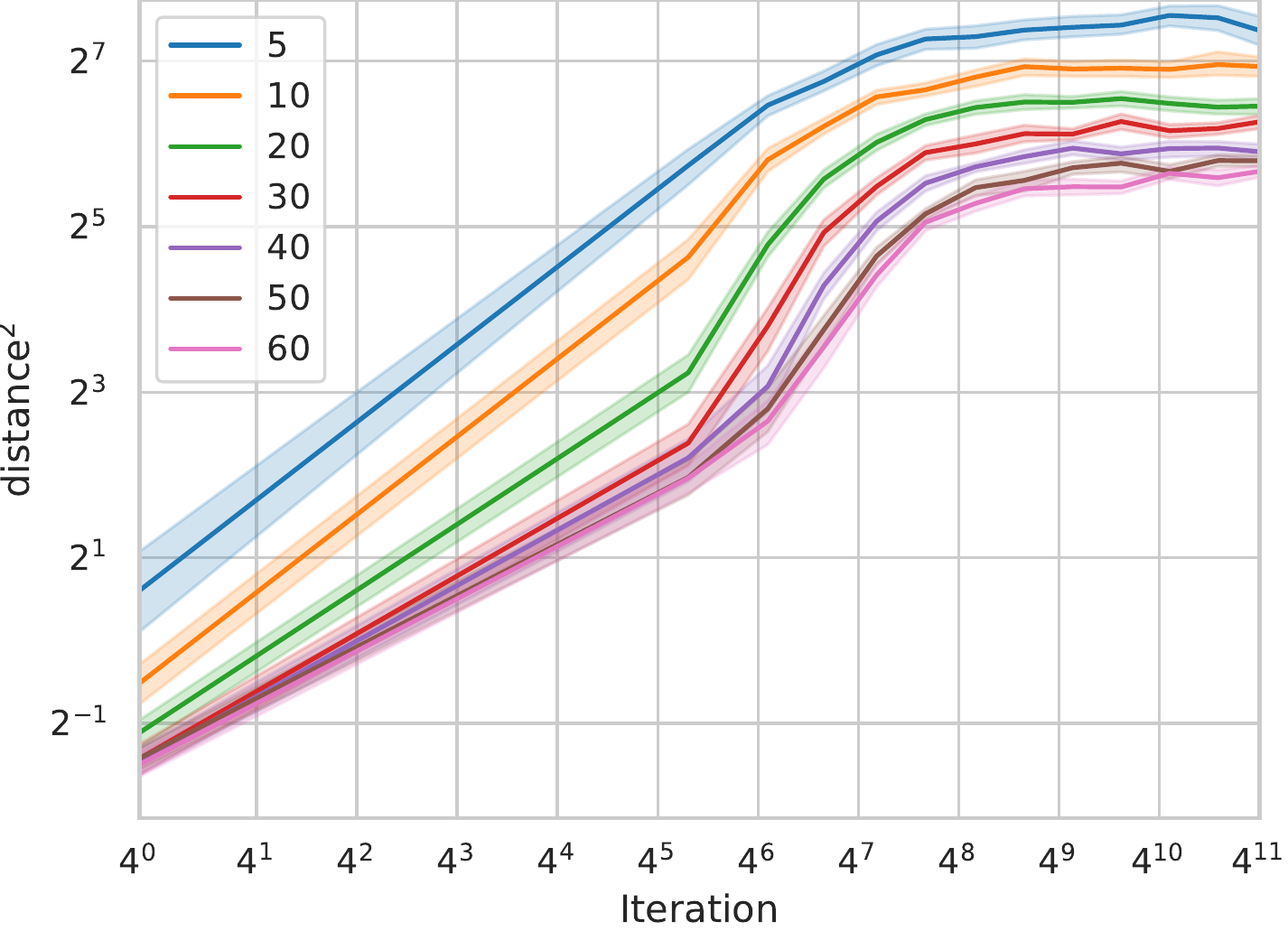}&
\includegraphics[width=.23\linewidth]{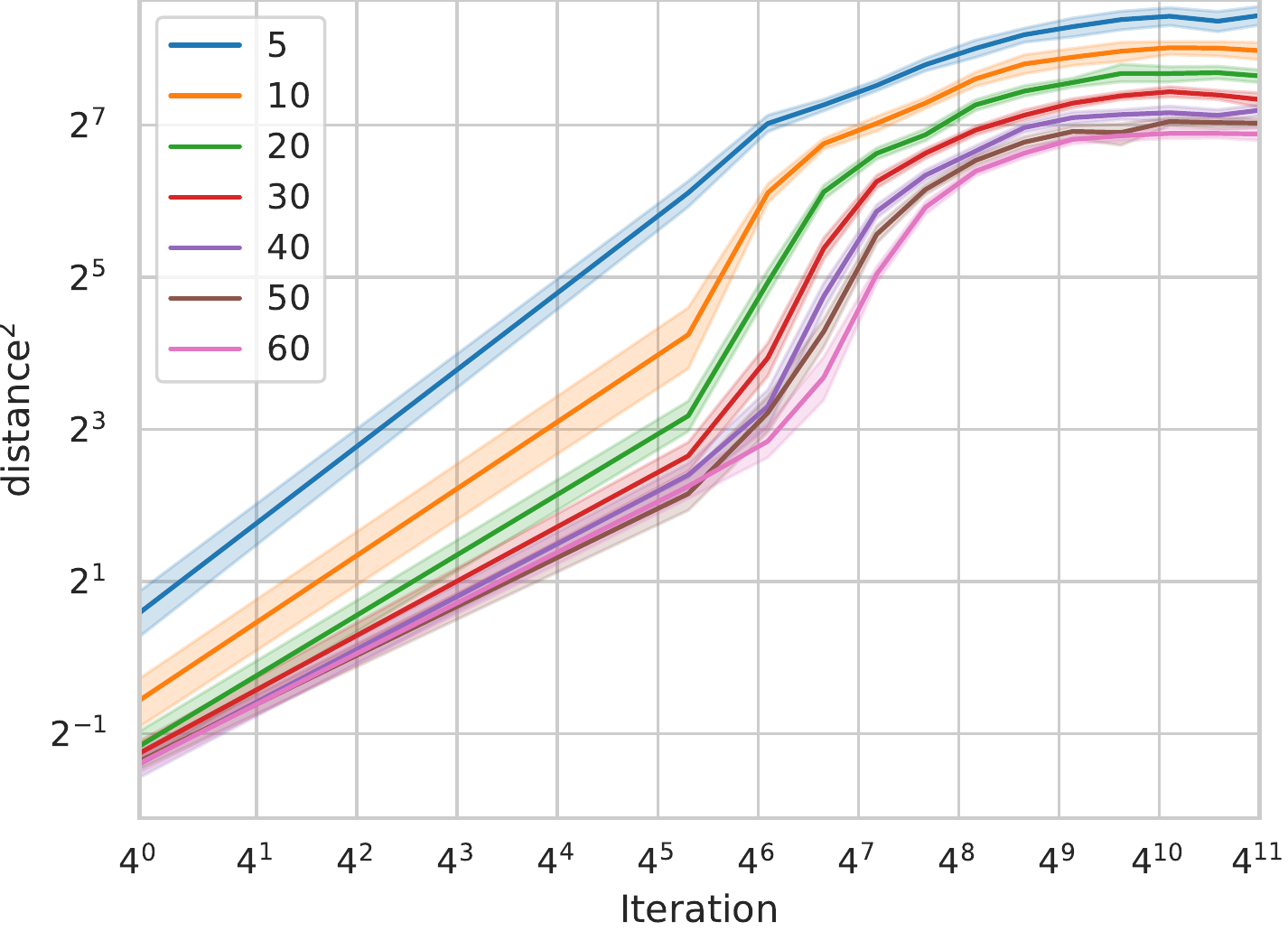}&
\includegraphics[width=.23\linewidth]{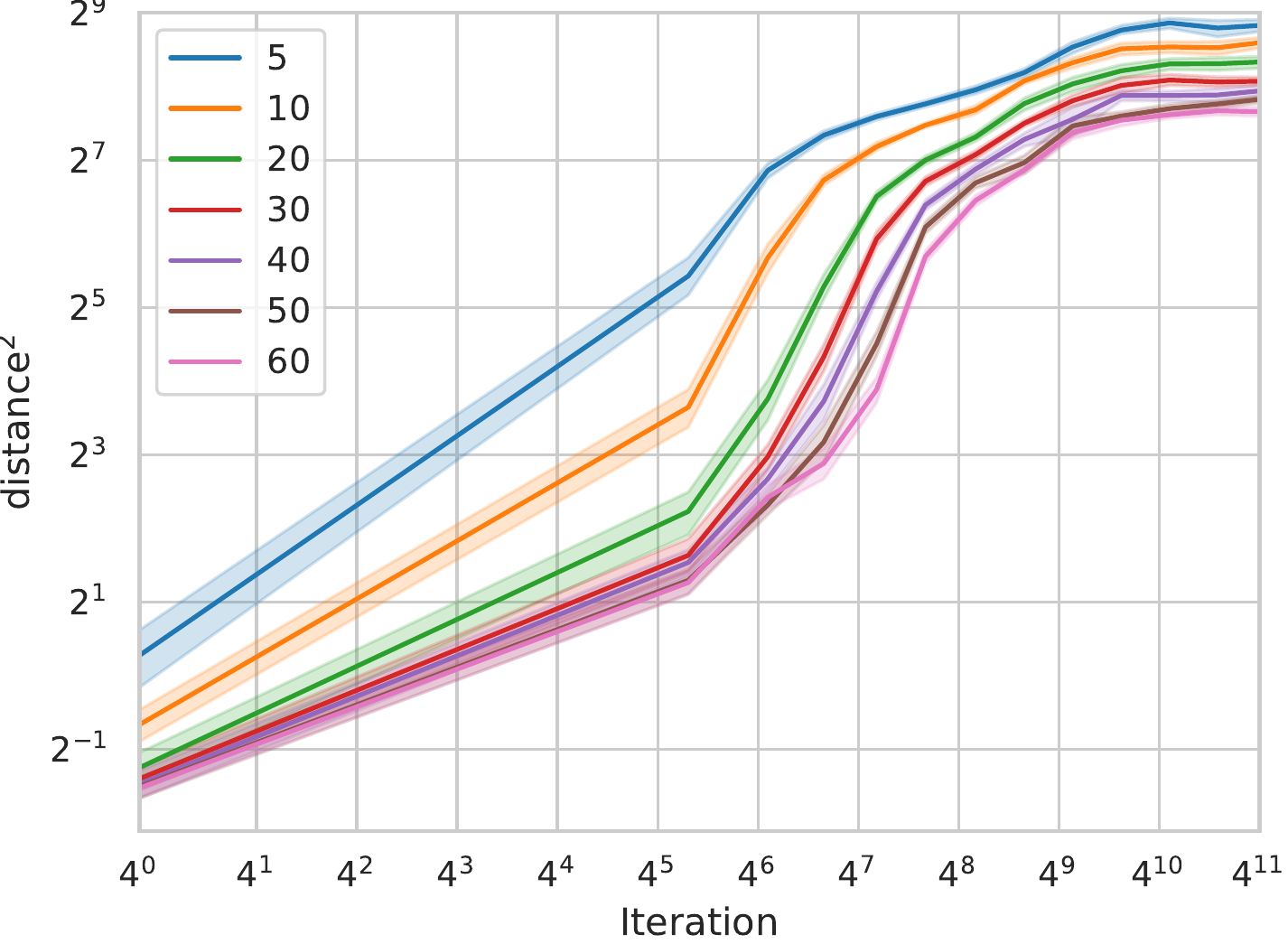}\\
\multicolumn{4}{c}{{\bf (a)} Source classes, train data} \\[0.5em]
\includegraphics[width=.23\linewidth]{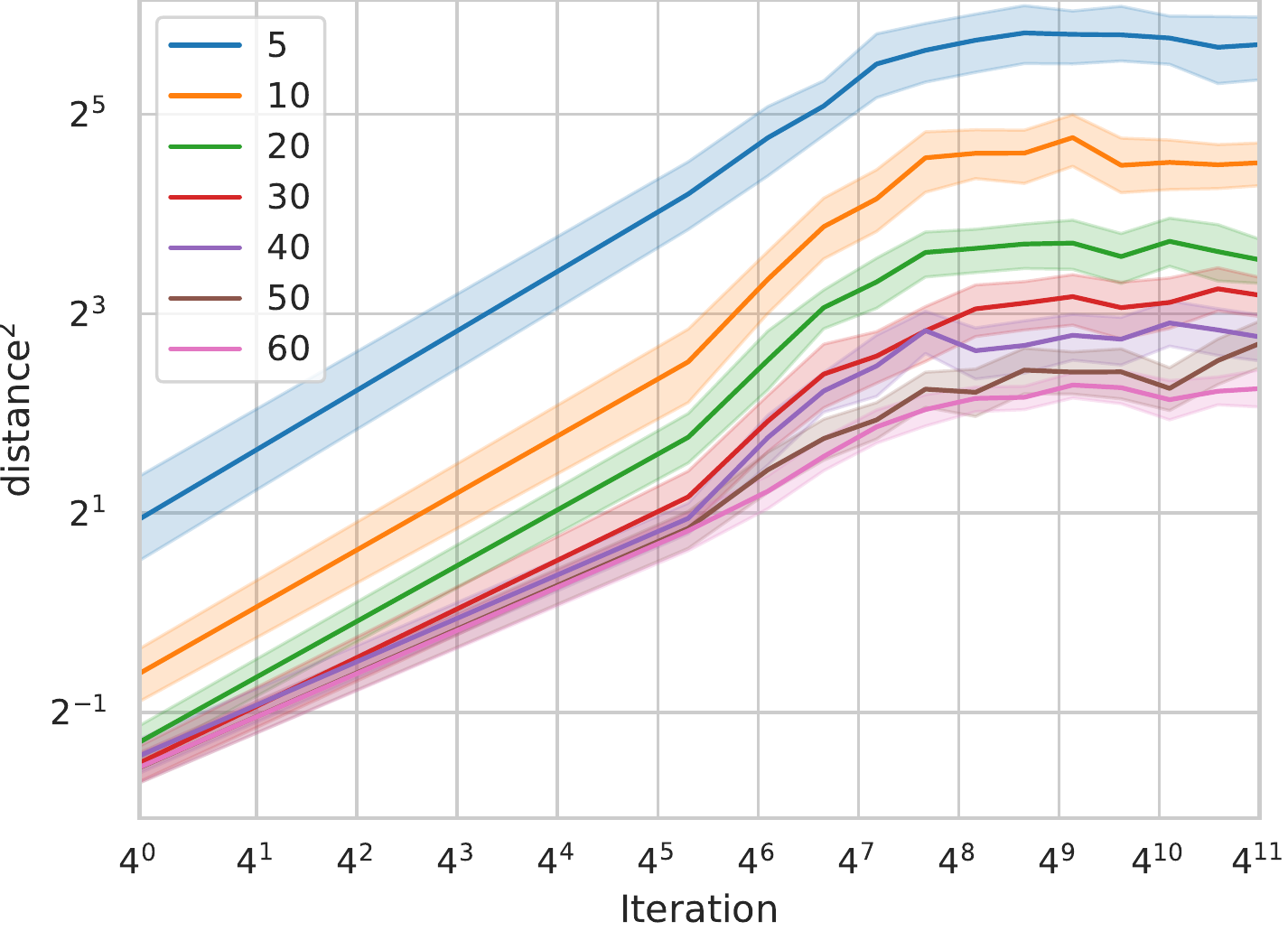}&
\includegraphics[width=.23\linewidth]{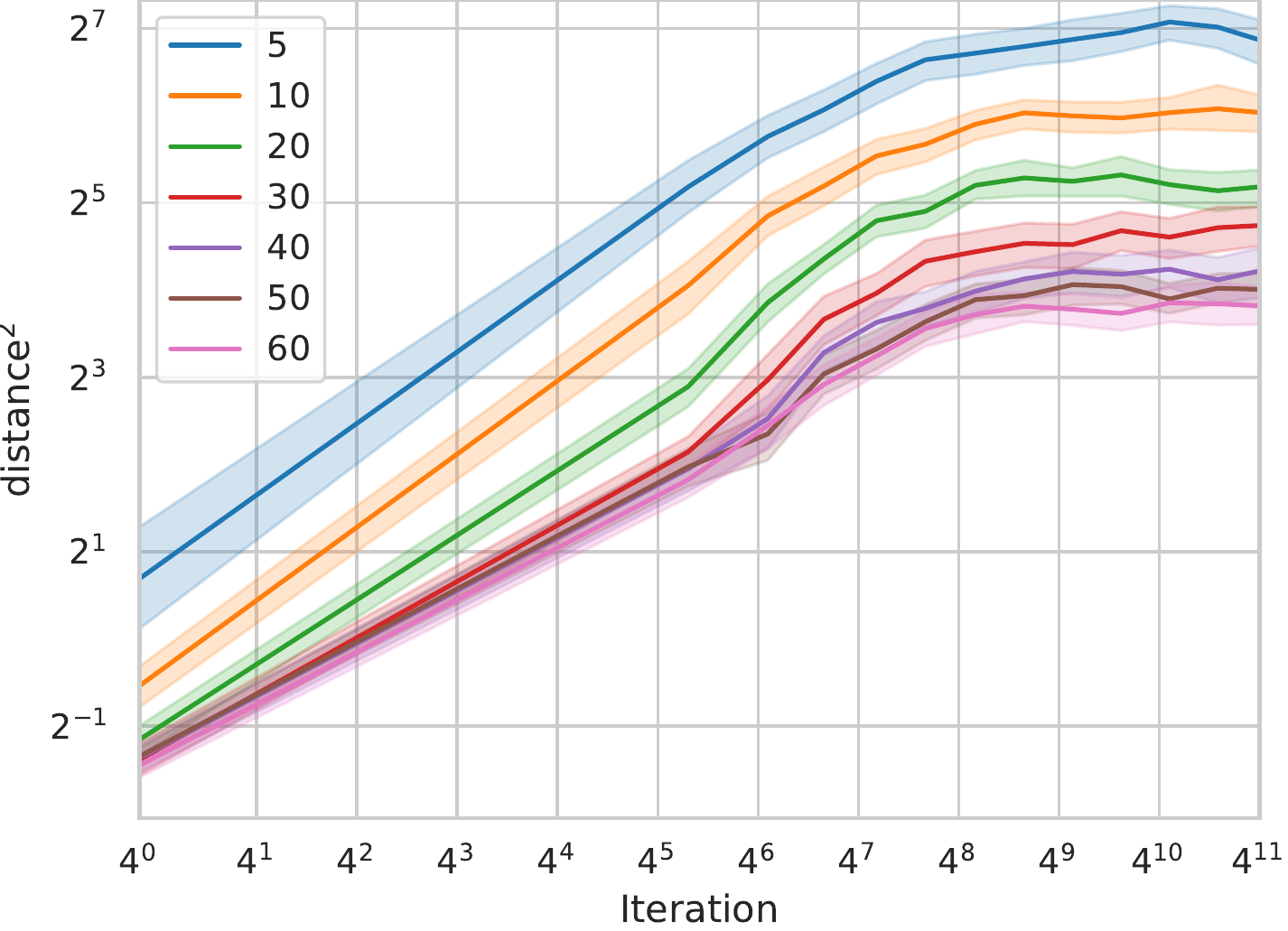}&
\includegraphics[width=.23\linewidth]{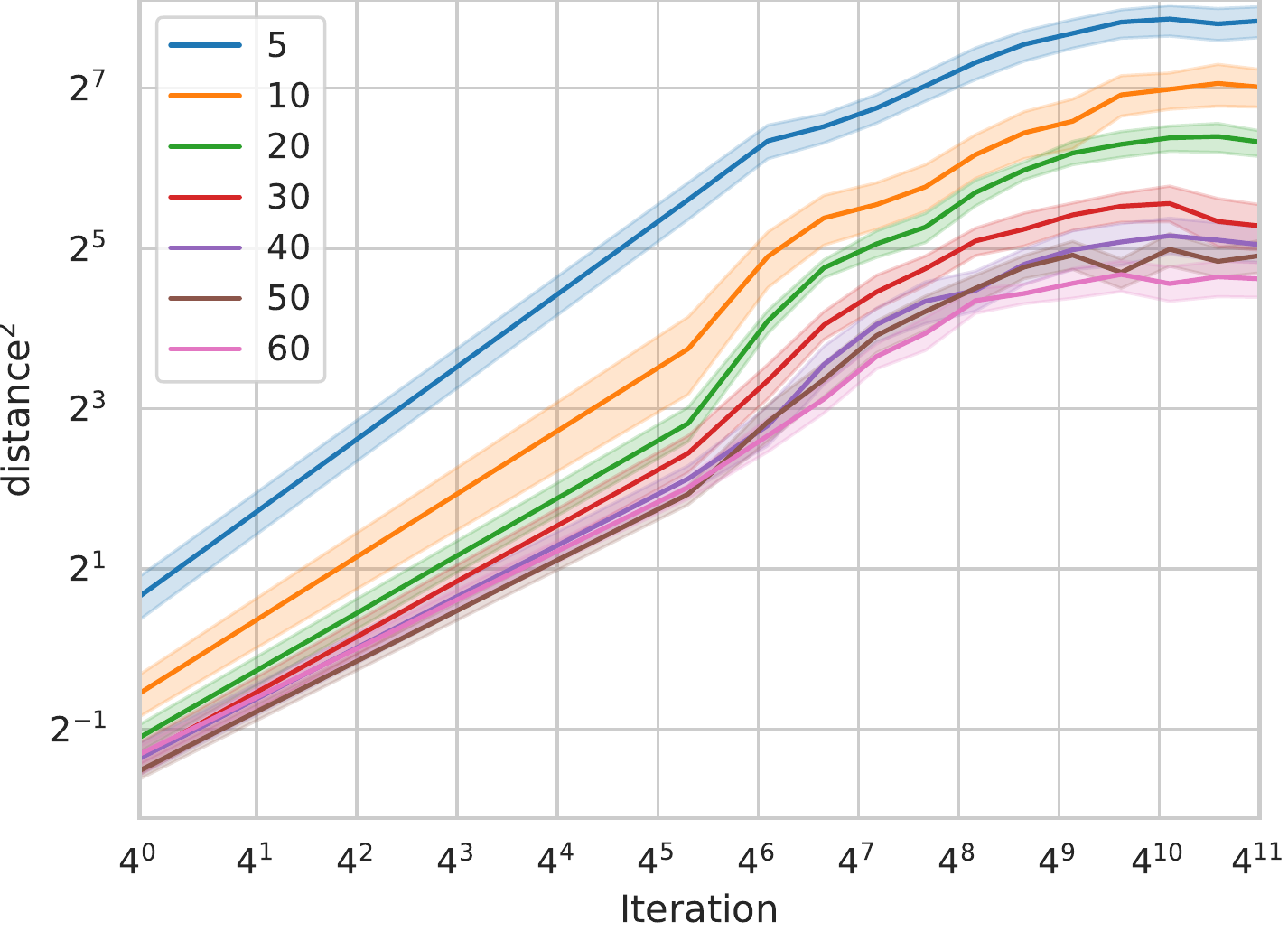}&
\includegraphics[width=.23\linewidth]{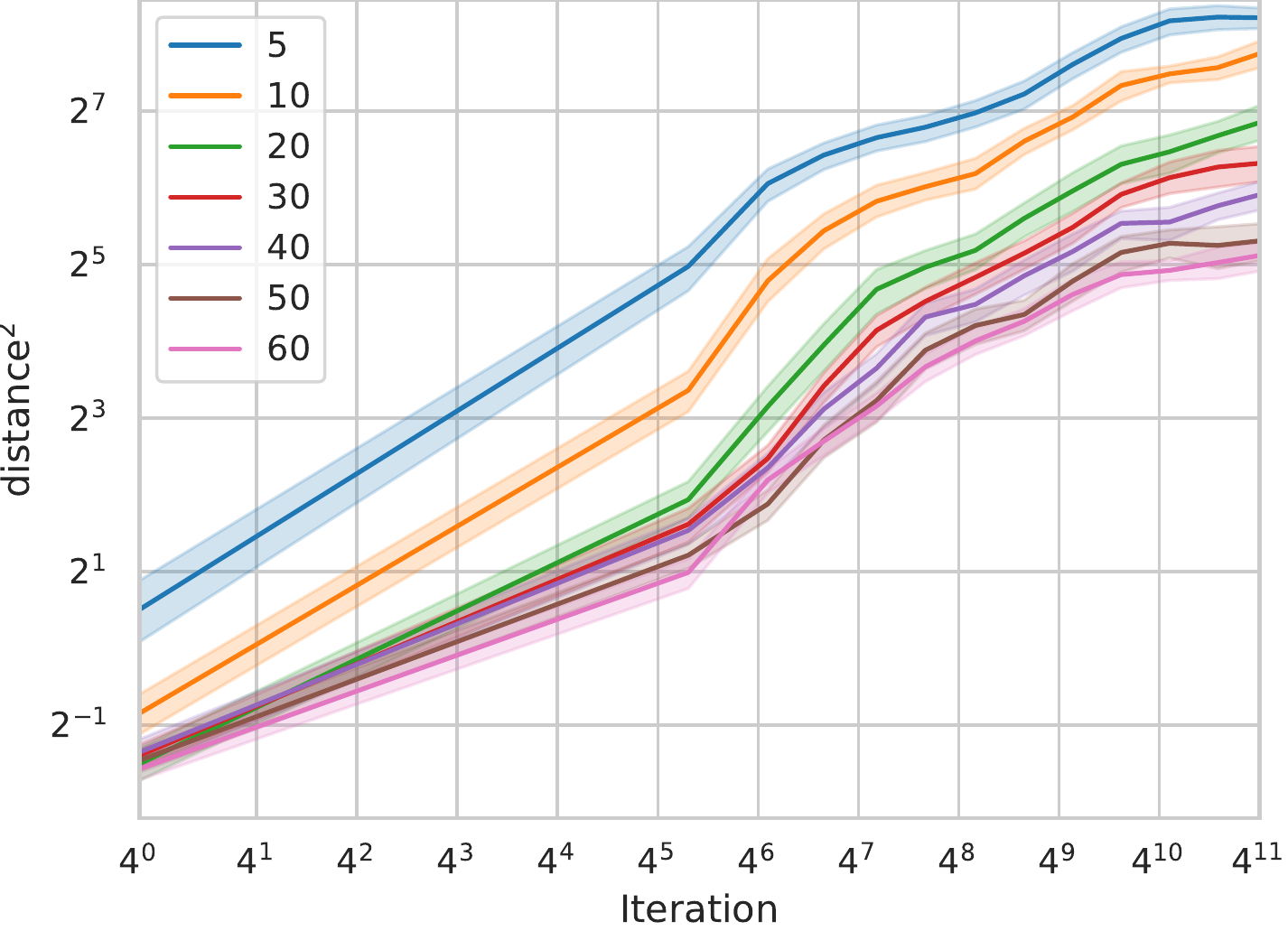}\\
\multicolumn{4}{c}{{\bf (b)} Source classes, test data} \\[0.5em]
\includegraphics[width=.23\linewidth]{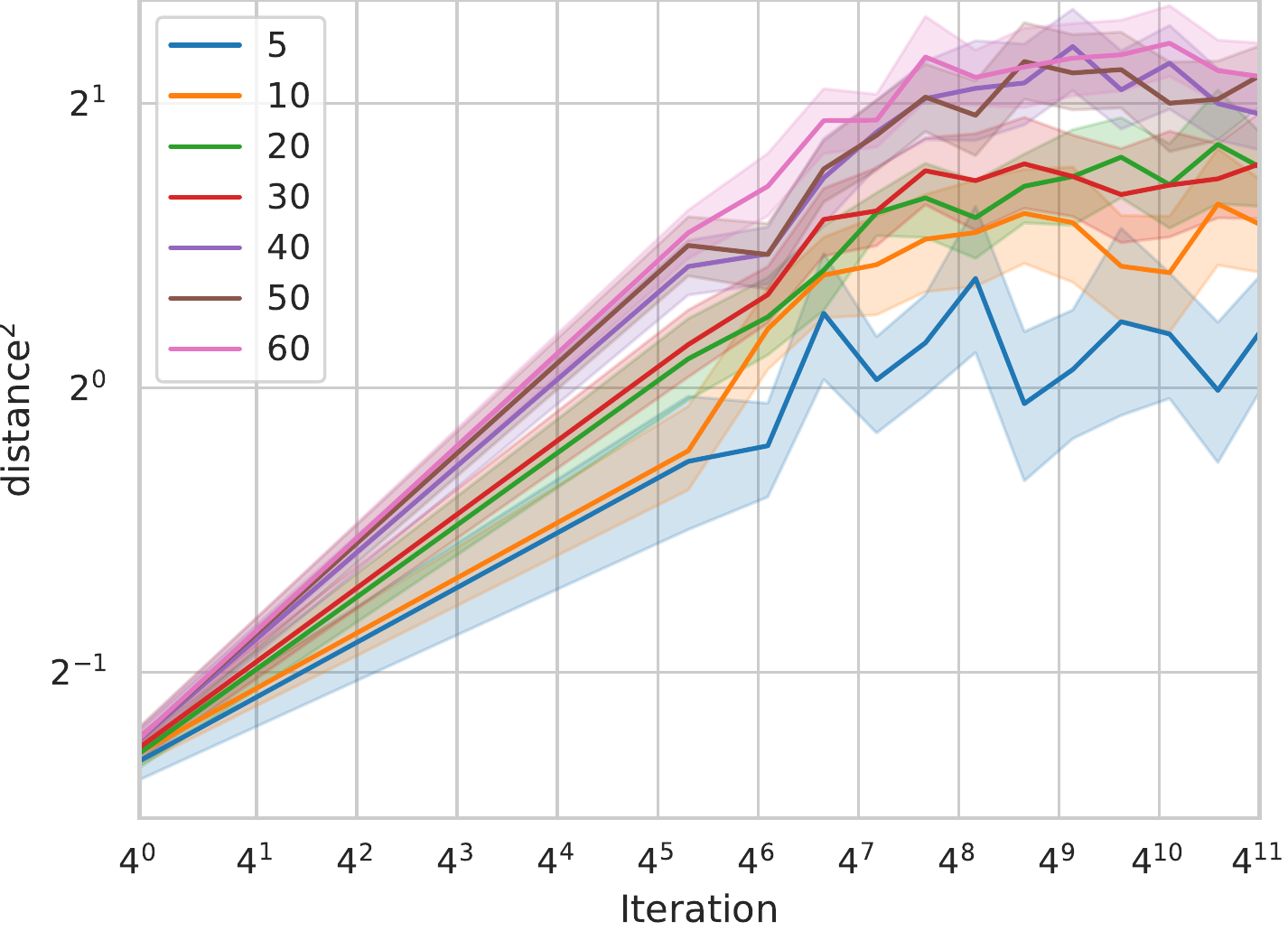}&
\includegraphics[width=.23\linewidth]{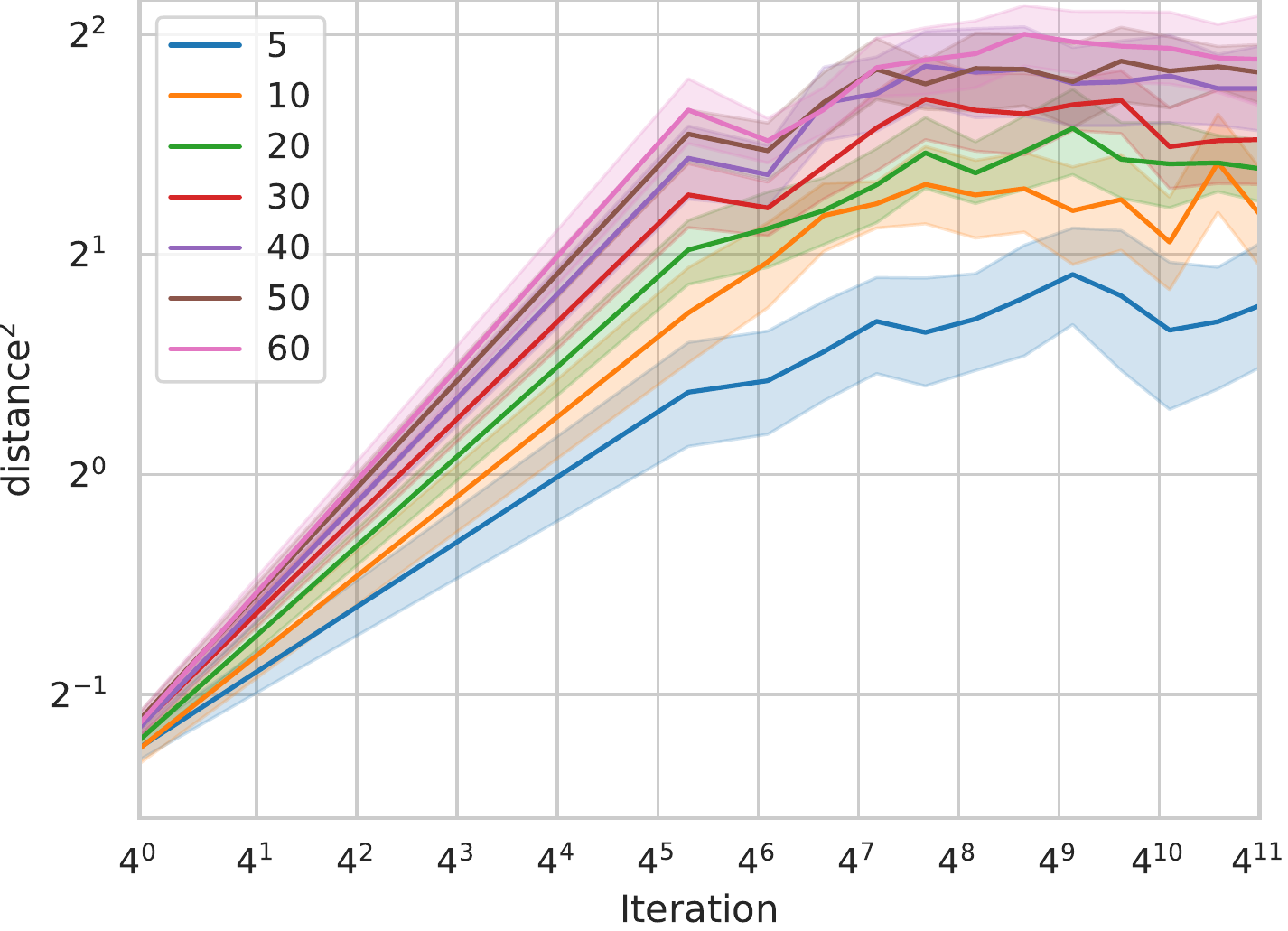}&
\includegraphics[width=.23\linewidth]{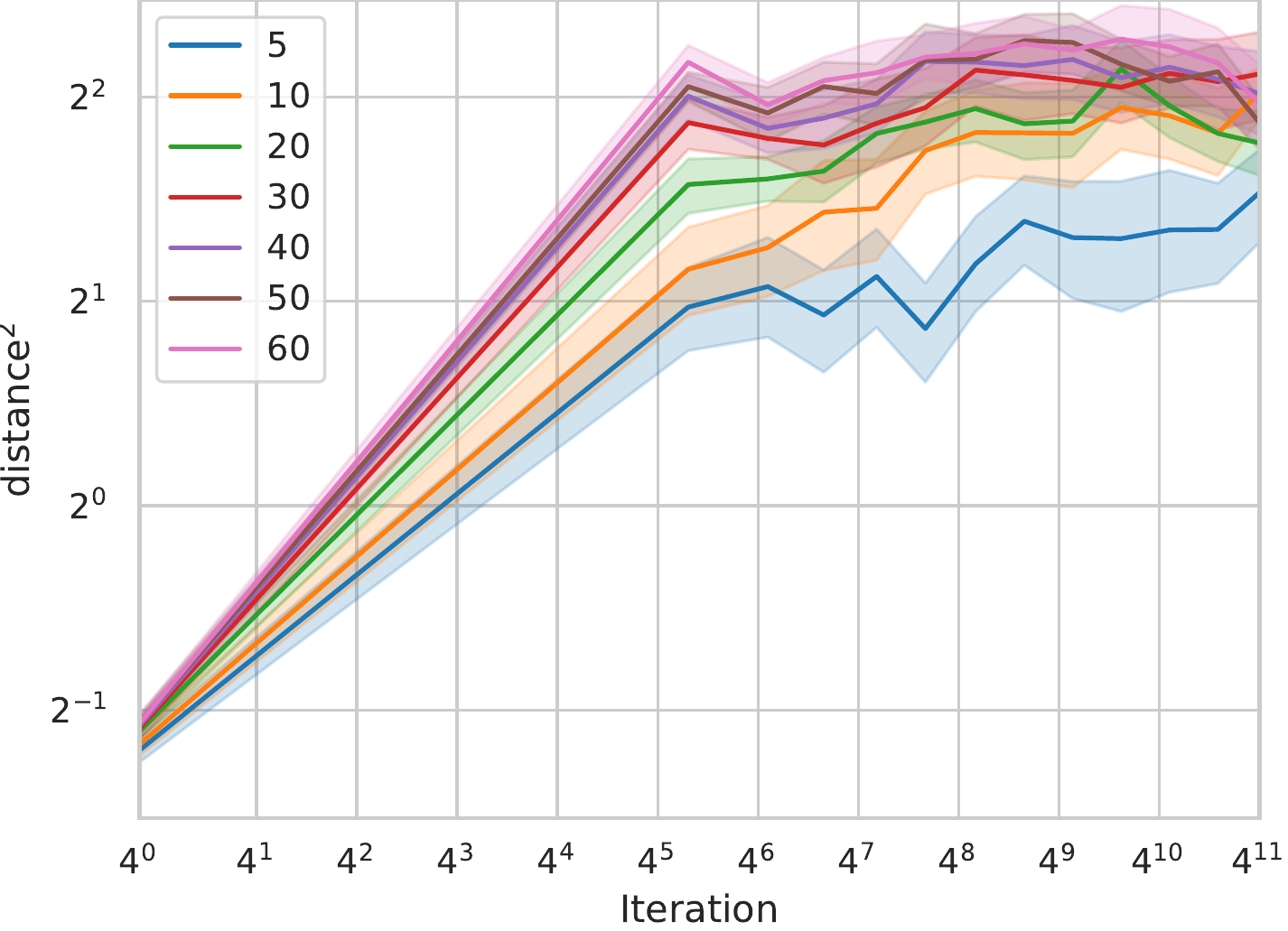}&
\includegraphics[width=.23\linewidth]{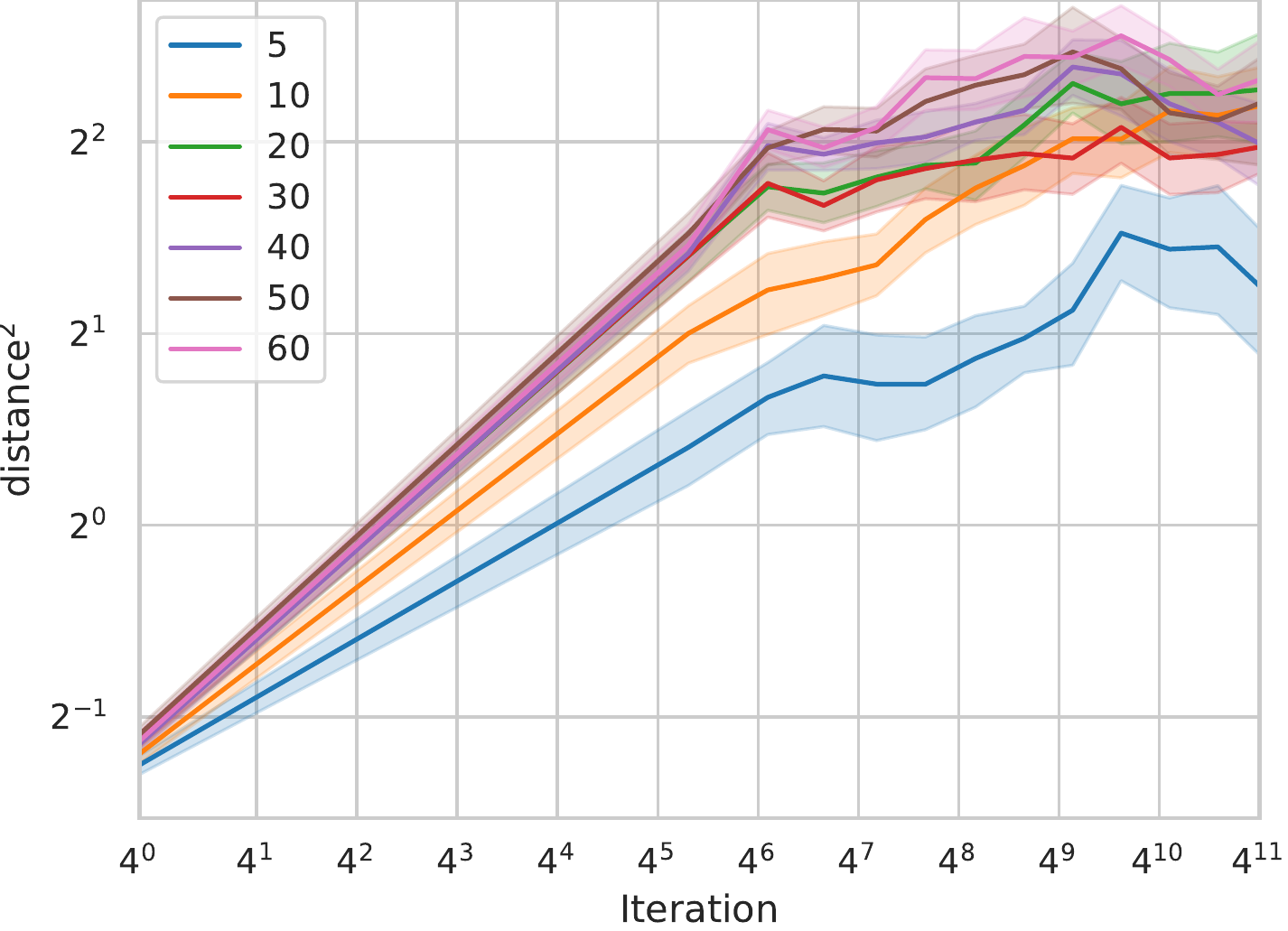}\\
\multicolumn{4}{c}{{\bf (c)} Target classes, test data} \\[0.5em]
\end{tabular}
    \caption{{\bf Dynamics of minimal class-means distance.} In {\bf (a)} we plot $\min_{i\neq j} \|\mu_f(\tS_i)-\mu_f(\tS_j)\|$, in {\bf (b)} we plot $\min_{i\neq j} \|\mu_f(\tP_i)-\mu_f(\tP_j)\|$ and in {\bf (c)} we plot $\min_{i\neq j} \|\mu_f(P_i)-\mu_f(P_j)\|$ as a function of the number of training iterations. In each experiment we trained a WRN-28-4 using SGD on a set of $l\in \{5,10,20,30,40,50,60\}$ source classes on CIFAR-FS (as indicated in the legend). The $i$'th column corresponds to the results of training with SGD with $\eta = 2^{-2i-2}$. 
    } \label{fig:dist_fs}
\end{figure}

Recall that Theorem~\ref{thm:full_bound} scales with $\Lambda^{-1} = \mathcal{O}(\fr{1}{\min_{i\neq j\in [l]}\|\mu_f(\tS_i) - \mu_f(\tS_j)\|})$. Therefore, we empirically investigated the dynamics of $\min_{i\neq j\in [l]}\|\mu_f(\tS_i) - \mu_f(\tS_j)\|$ during training in our standard setting (WRN-28-4 with the default hyperparameters, see Section~\ref{sec:setup}) on CIFAR-FS, considering a varying number source classes ($l\in \{5,10,20,30,40,50,60\}$) and learning rates ($\eta \in \{2^{-2i-2}\}^{4}_{i=1}$). As can be seen in Figure~\ref{fig:dist_fs}, the values of $\min_{i\neq j \in [l]}\|\mu_f(\tS_i) - \mu_f(\tS_j)\|$ tend to increase during training and is tends to be larger when using a smaller learning rate. For completeness, we also plotted the values of $\min_{i\neq j\in [l]}\|\mu_f(\tP_i) - \mu_f(\tP_j)\|$ and $\min_{i\neq j\in [k]}\|\mu_f(P_i) - \mu_f(P_j)\|$. Interestingly, $\min_{i\neq j \in [k]}\|\mu_f(P_i) - \mu_f(P_j)\|$ tends to increase with the number of source classes, indicating improved generalization to new classes.

\subsection{Varying the Number of Target Samples}

\begin{figure}[t]
    \centering
    \begin{tabular}{@{}c@{~}c@{~}c@{~}c}
\includegraphics[width=.23\linewidth]{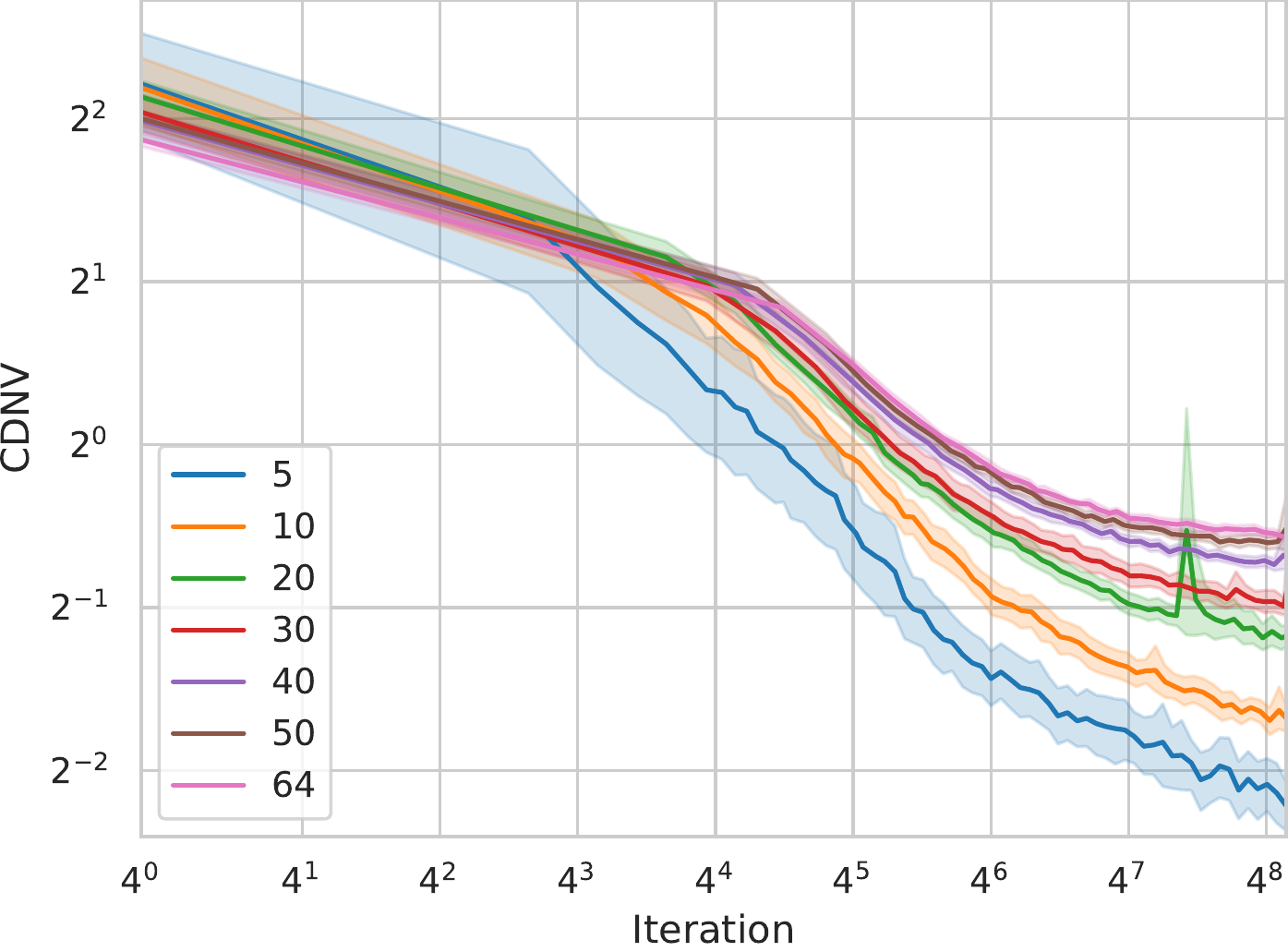}&
\includegraphics[width=.23\linewidth]{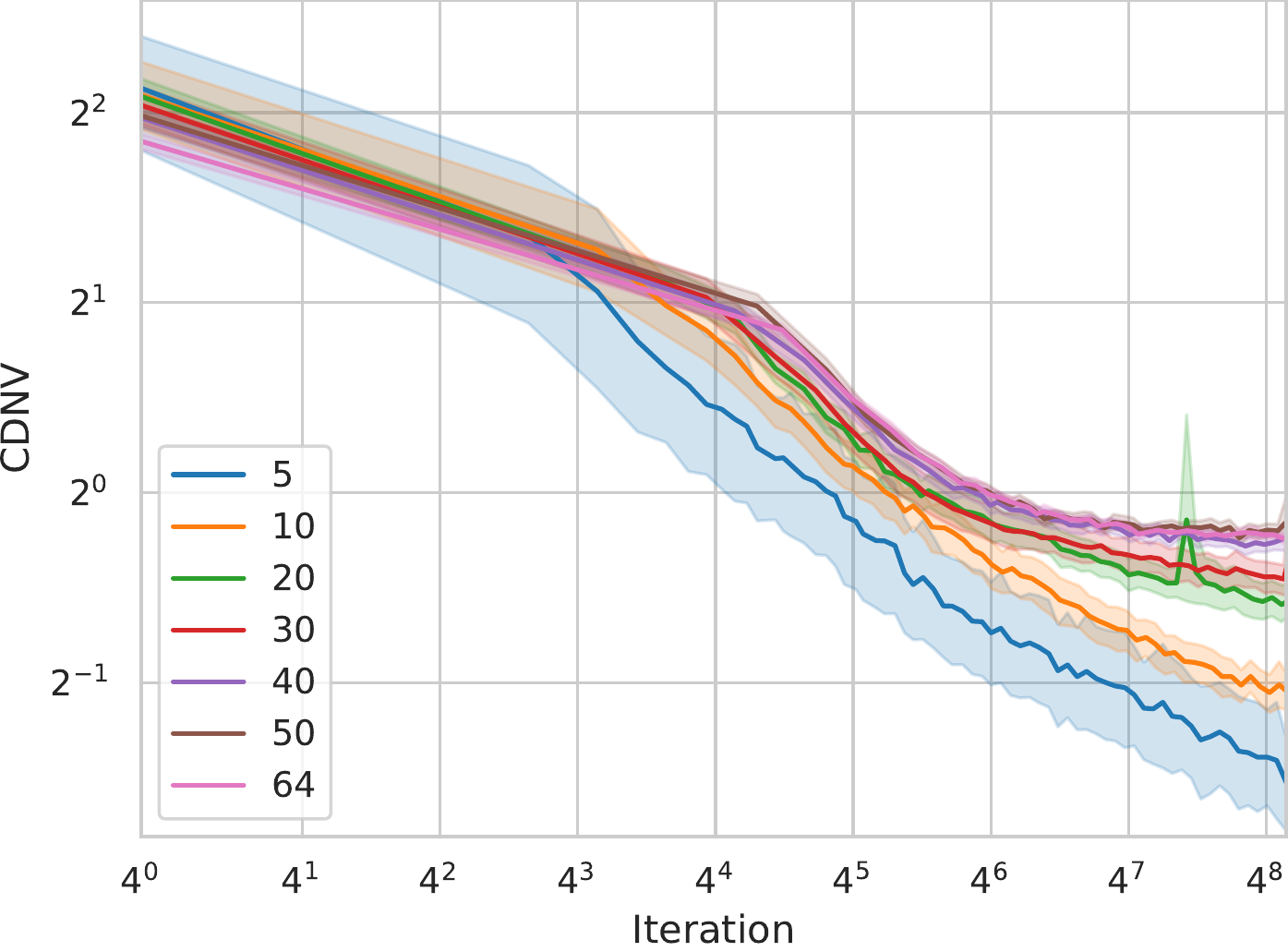}&
\includegraphics[width=.23\linewidth]{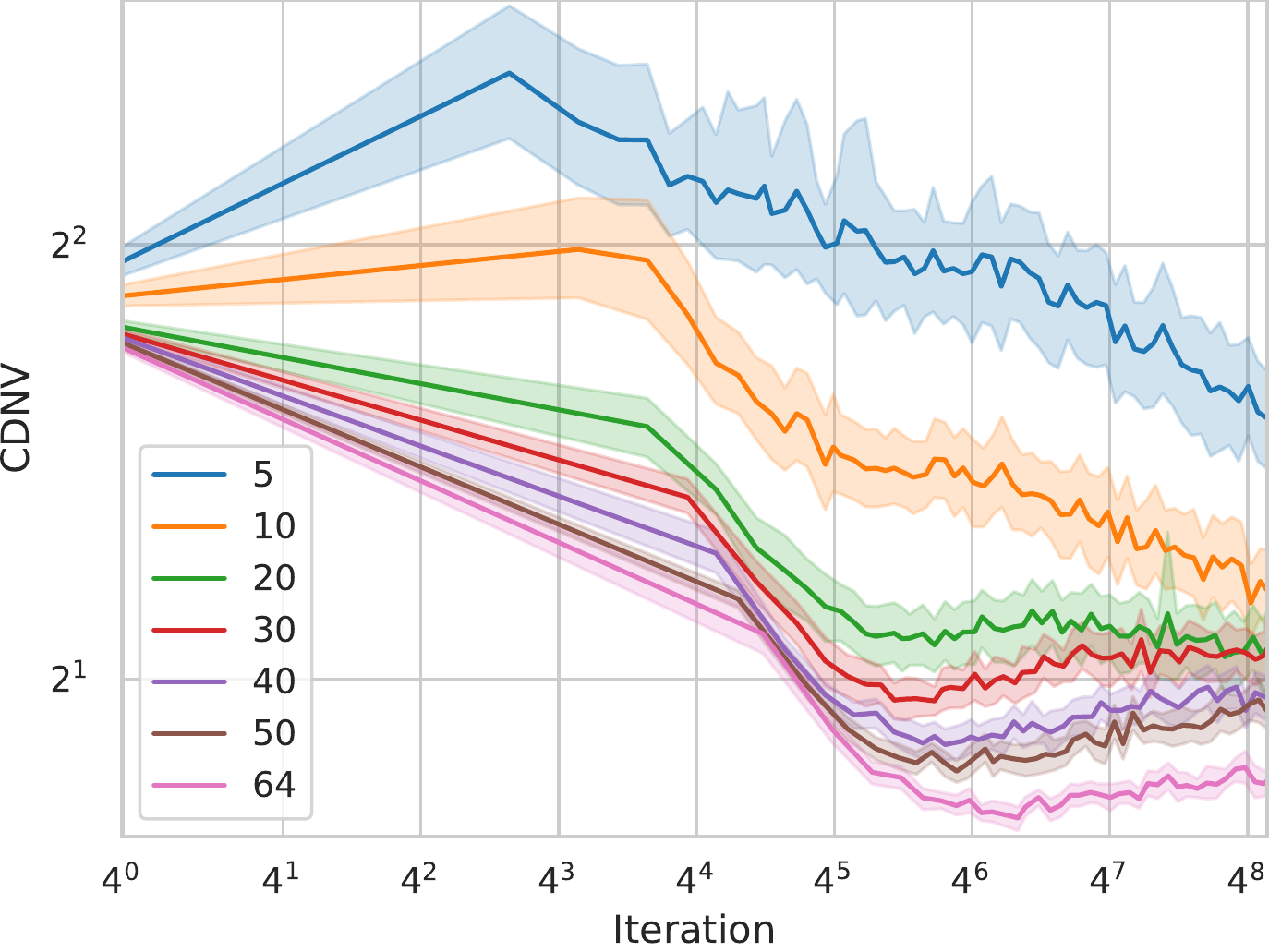}\\
{\bf (a)} CDNV, train & {\bf(b)} CDNV, test & {\bf(c)} CDNV, target \\ 
\includegraphics[width=.23\linewidth]{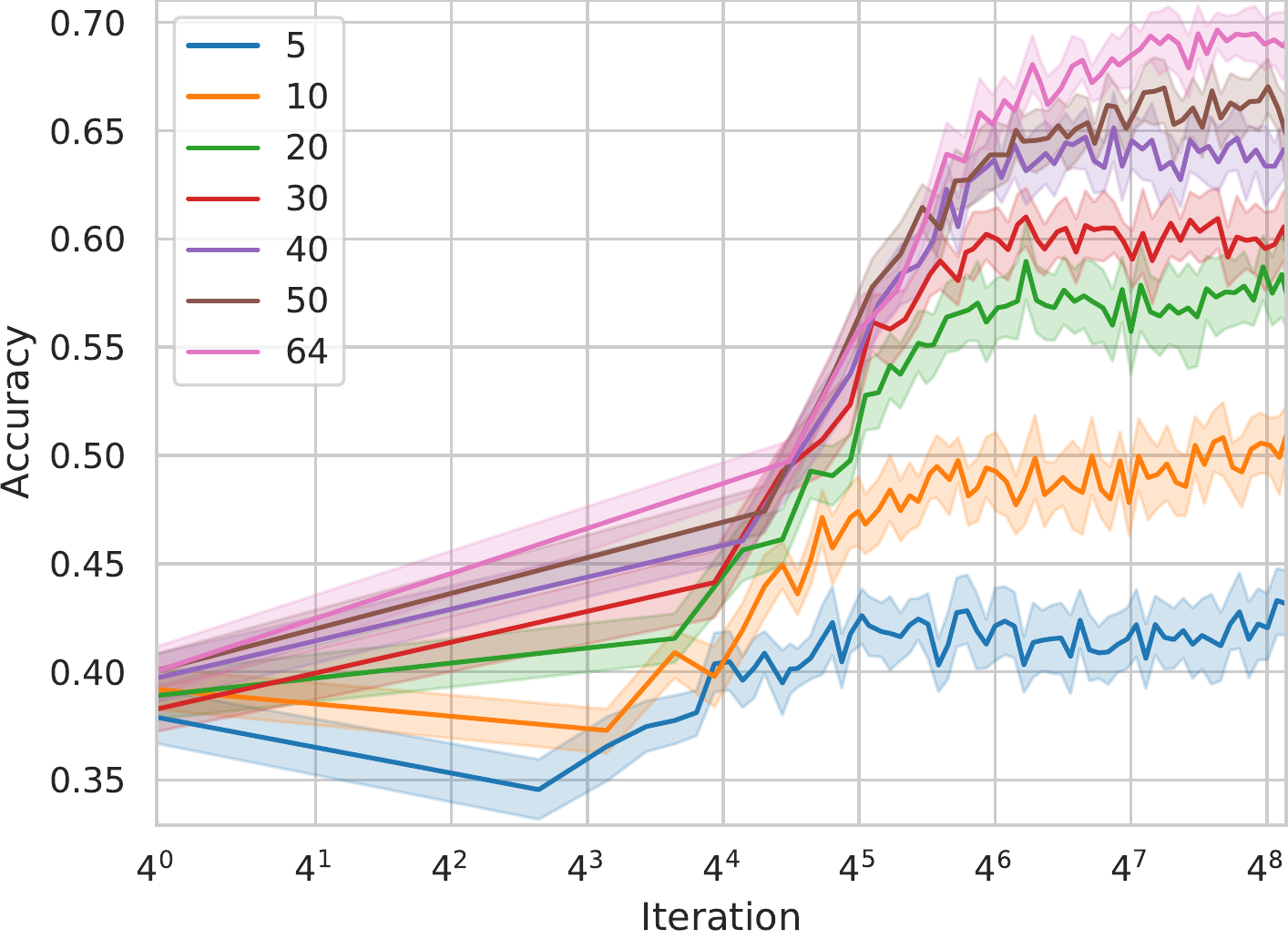}& 
\includegraphics[width=.23\linewidth]{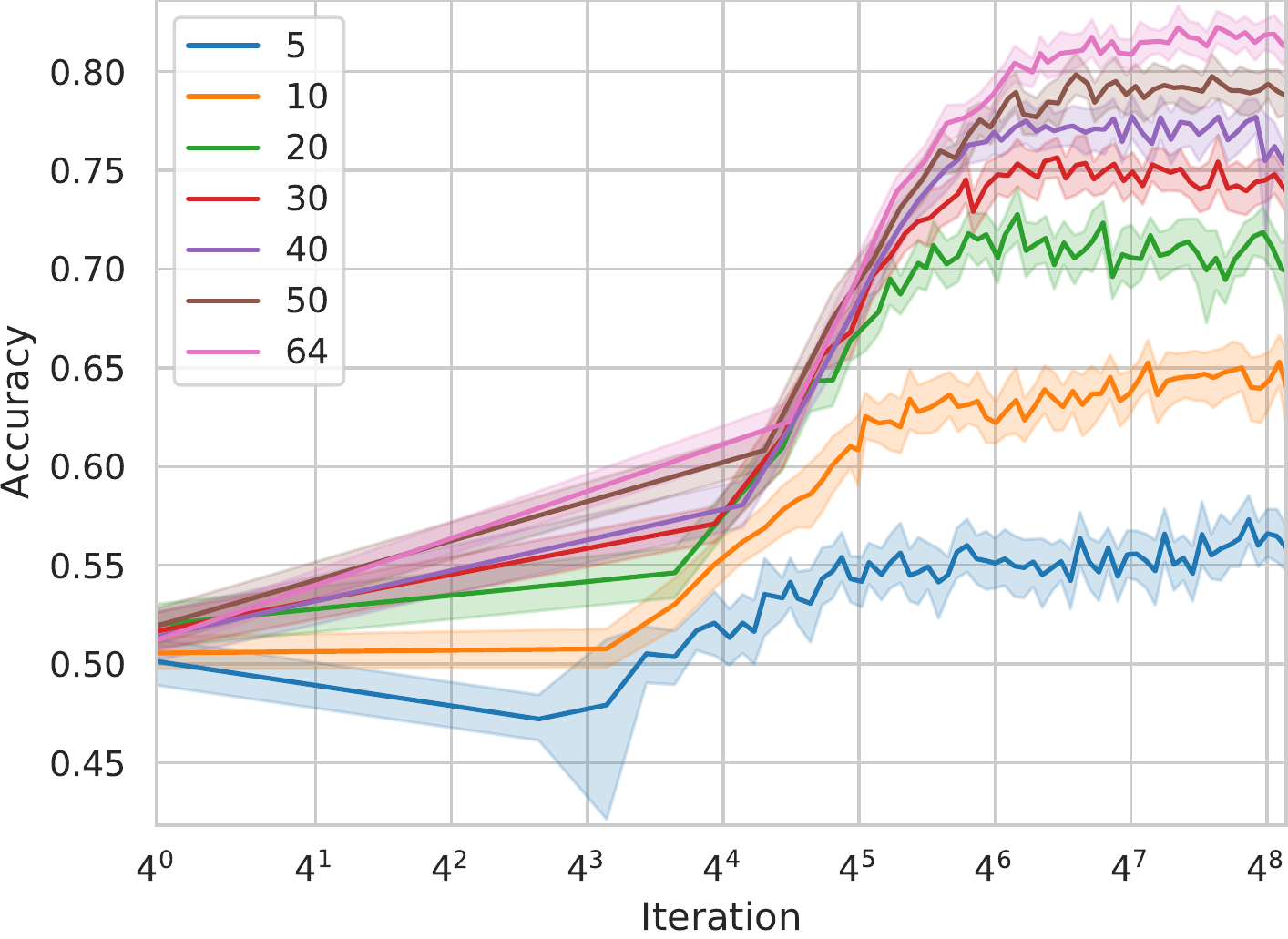}& 
\includegraphics[width=.23\linewidth]{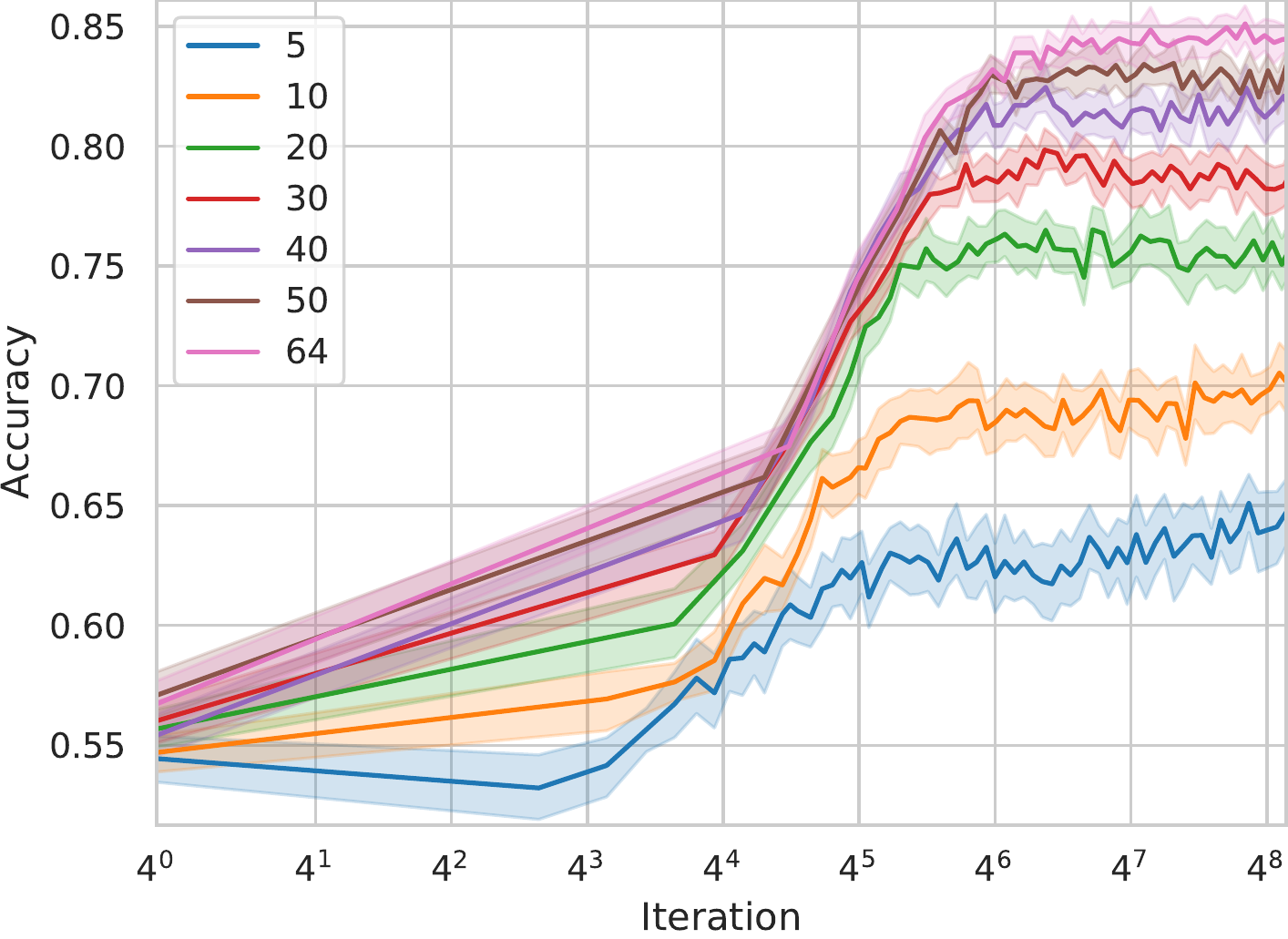}&
\includegraphics[width=.23\linewidth]{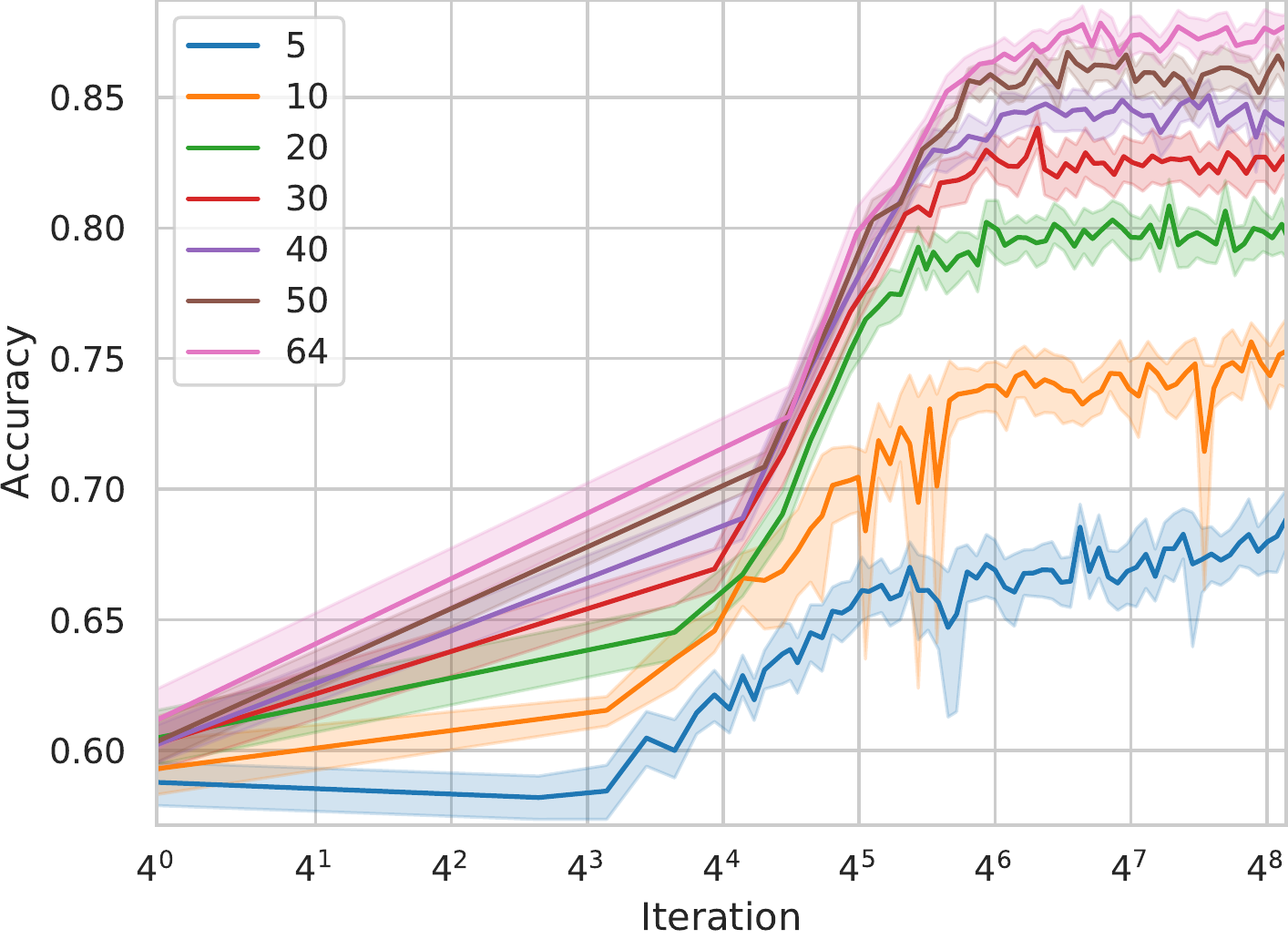}\\
{\bf(d)} 1-shot accuracy & {\bf(e)} 5-shot accuracy & {\bf(f)} 10-shot acc & {\bf(g)} 20-shot acc \\
\multicolumn{4}{c}{CIFAR-FS} \\[0.5em]
\includegraphics[width=.23\linewidth]{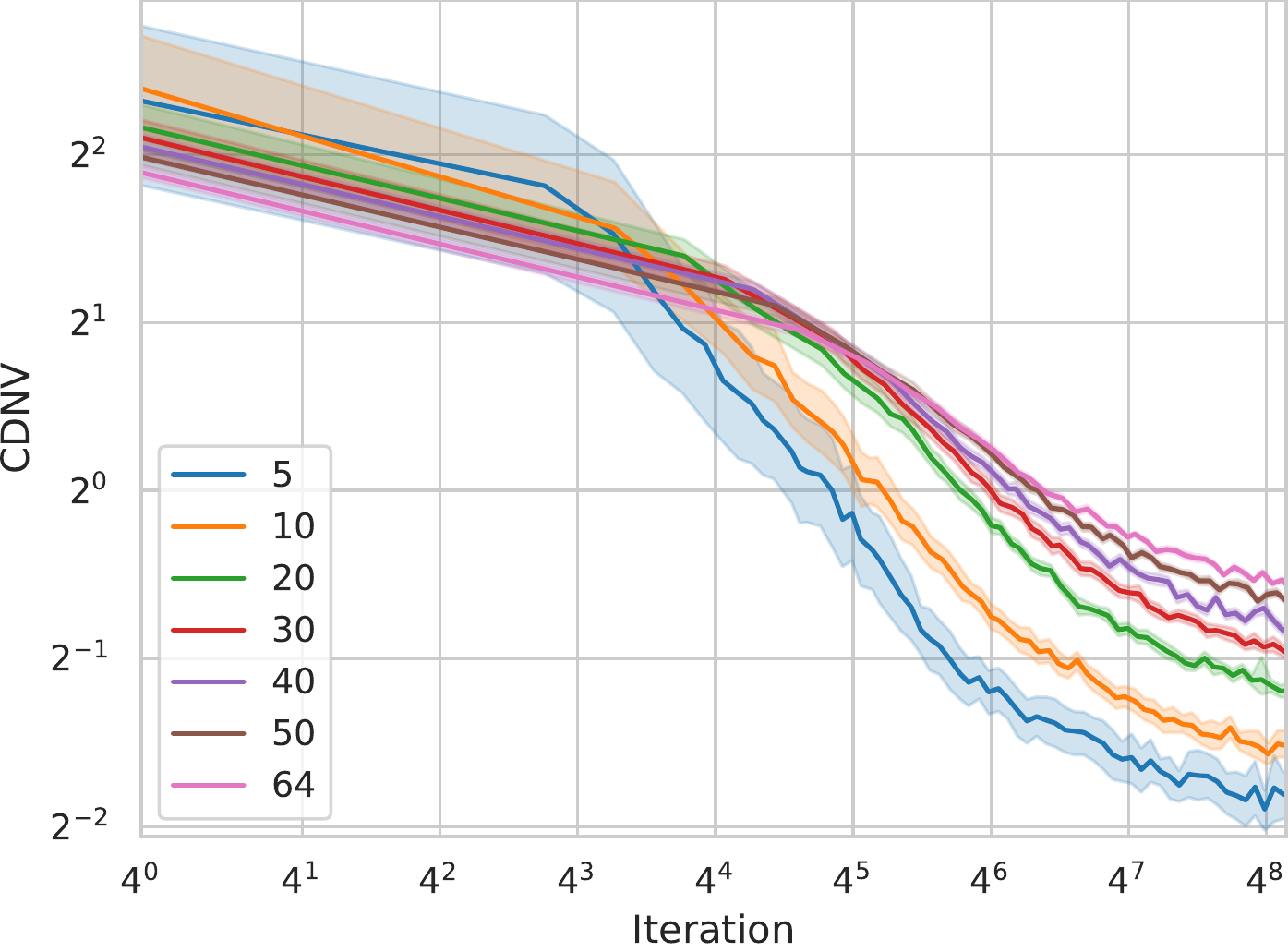}&
\includegraphics[width=.23\linewidth]{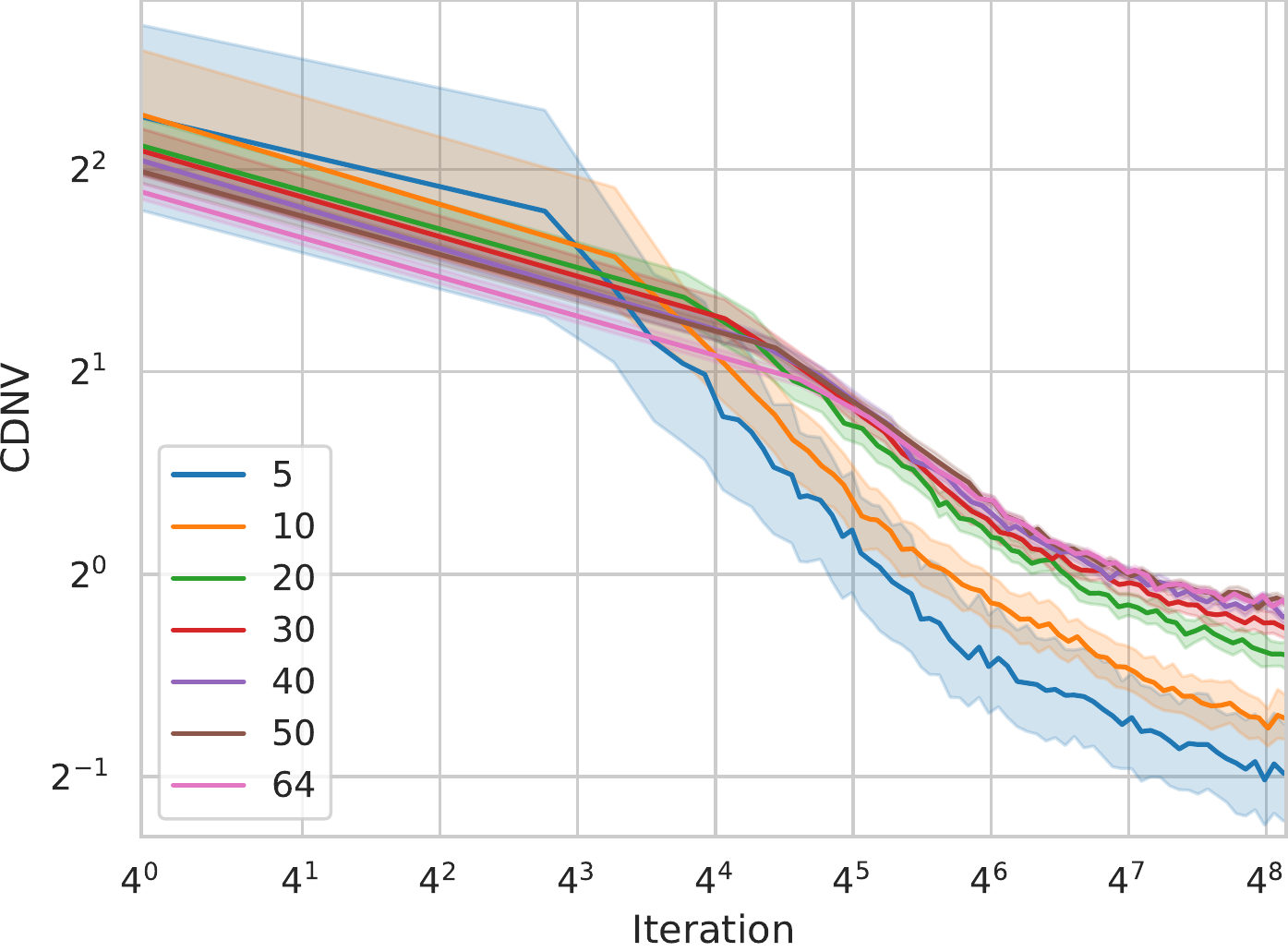}&
\includegraphics[width=.23\linewidth]{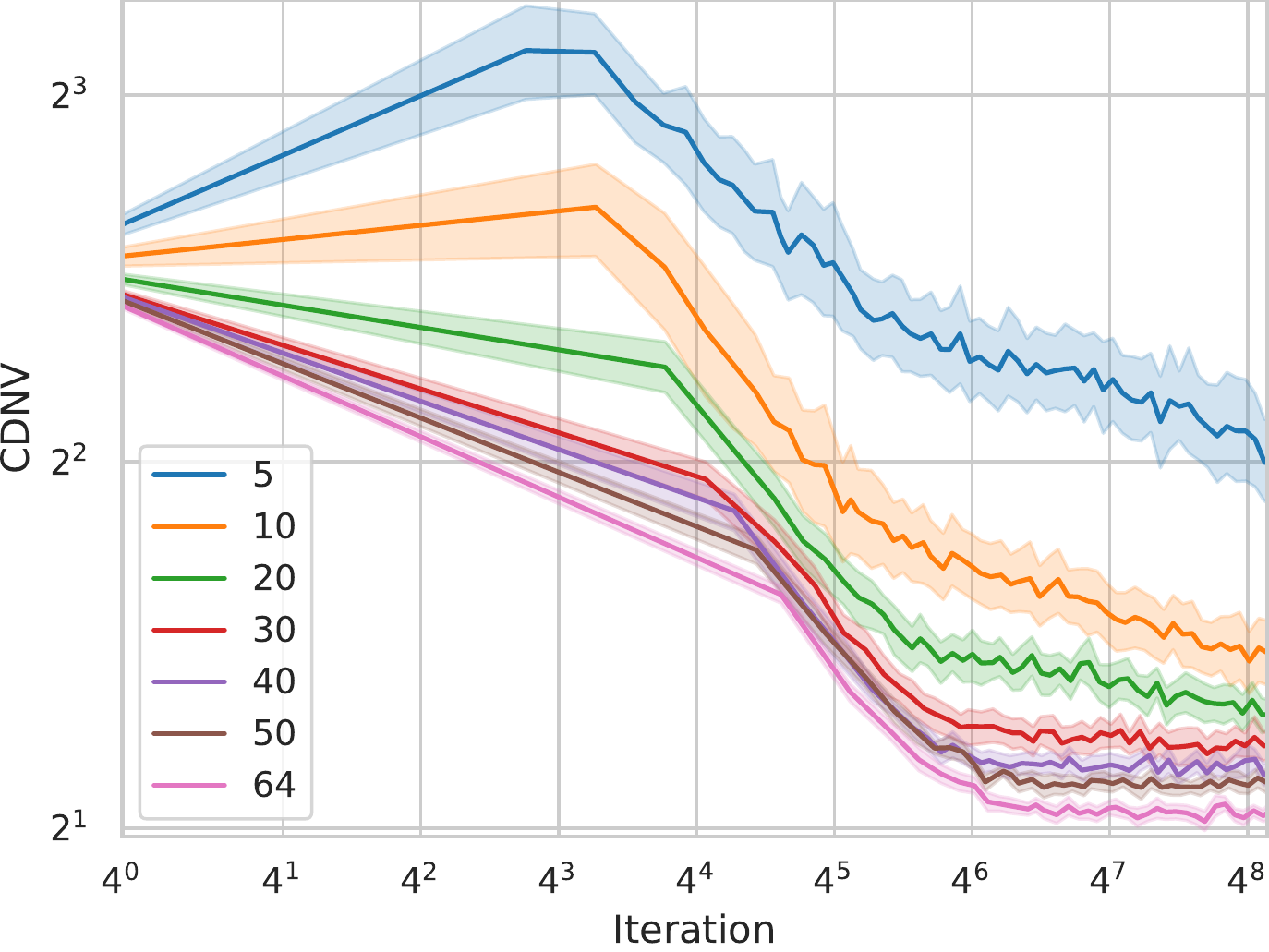}\\
{\bf (a)} CDNV, train & {\bf(b)} CDNV, test & {\bf(c)} CDNV, target \\
\includegraphics[width=.23\linewidth]{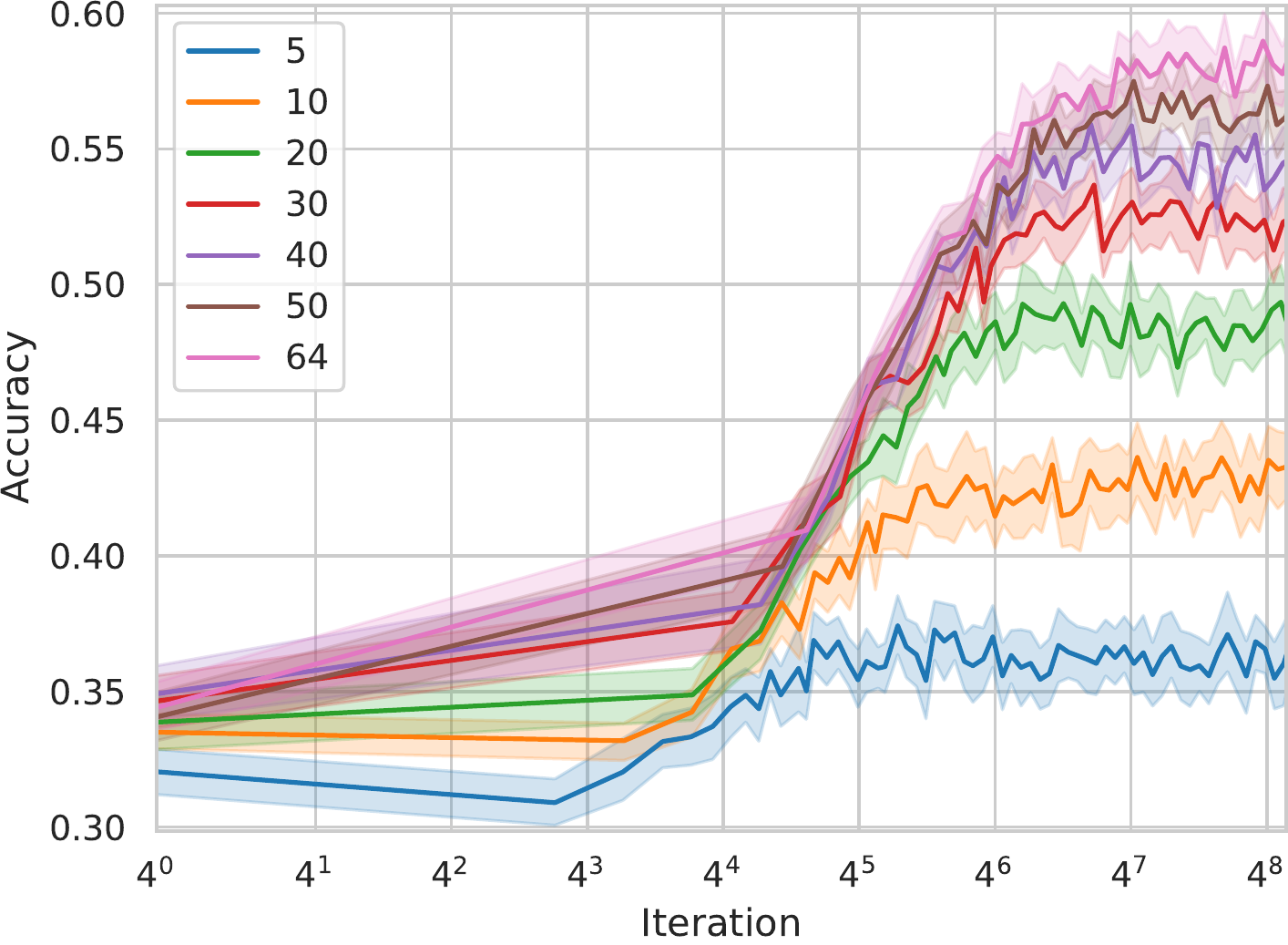}& 
\includegraphics[width=.23\linewidth]{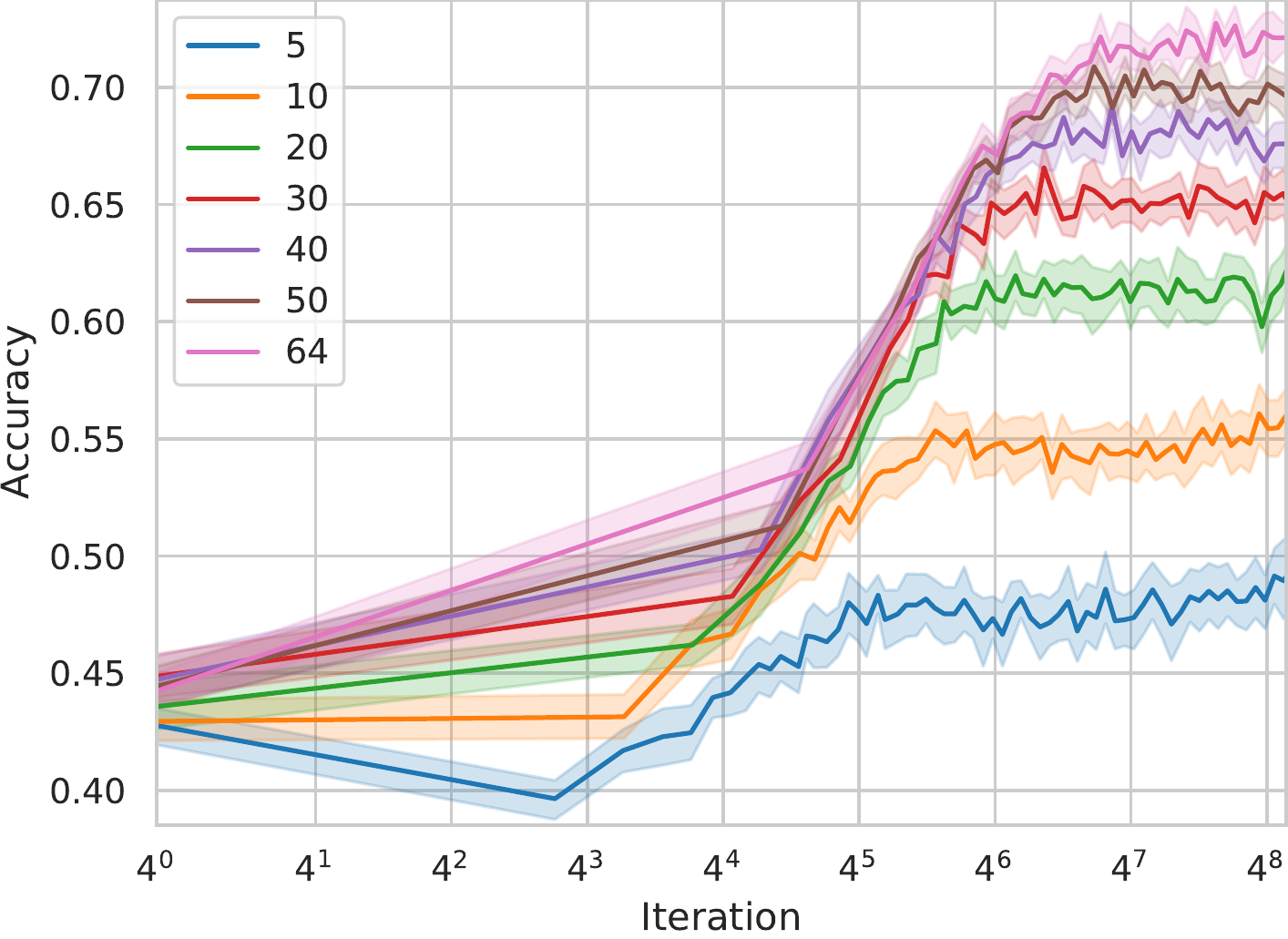}& 
\includegraphics[width=.23\linewidth]{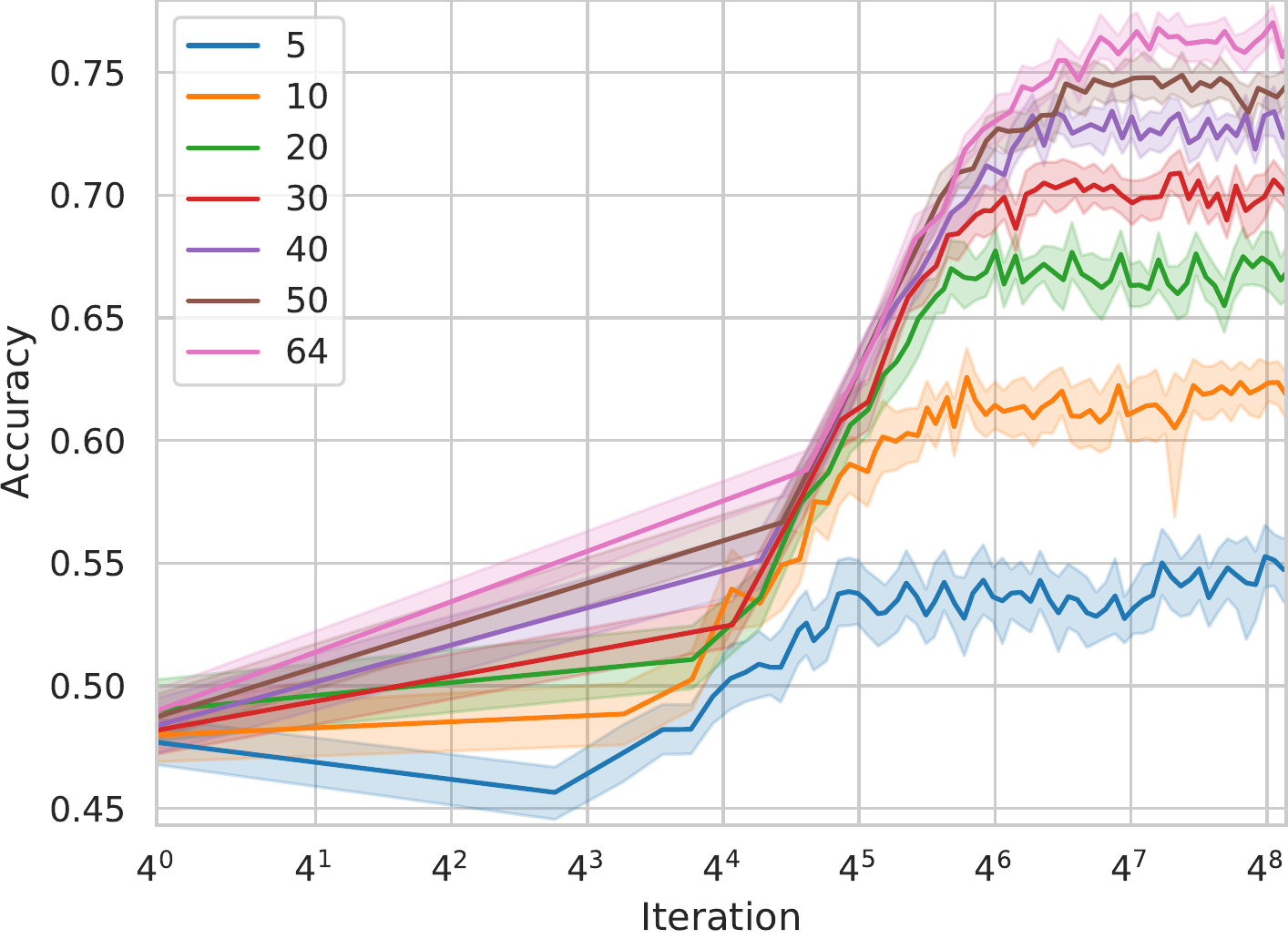}&
\includegraphics[width=.23\linewidth]{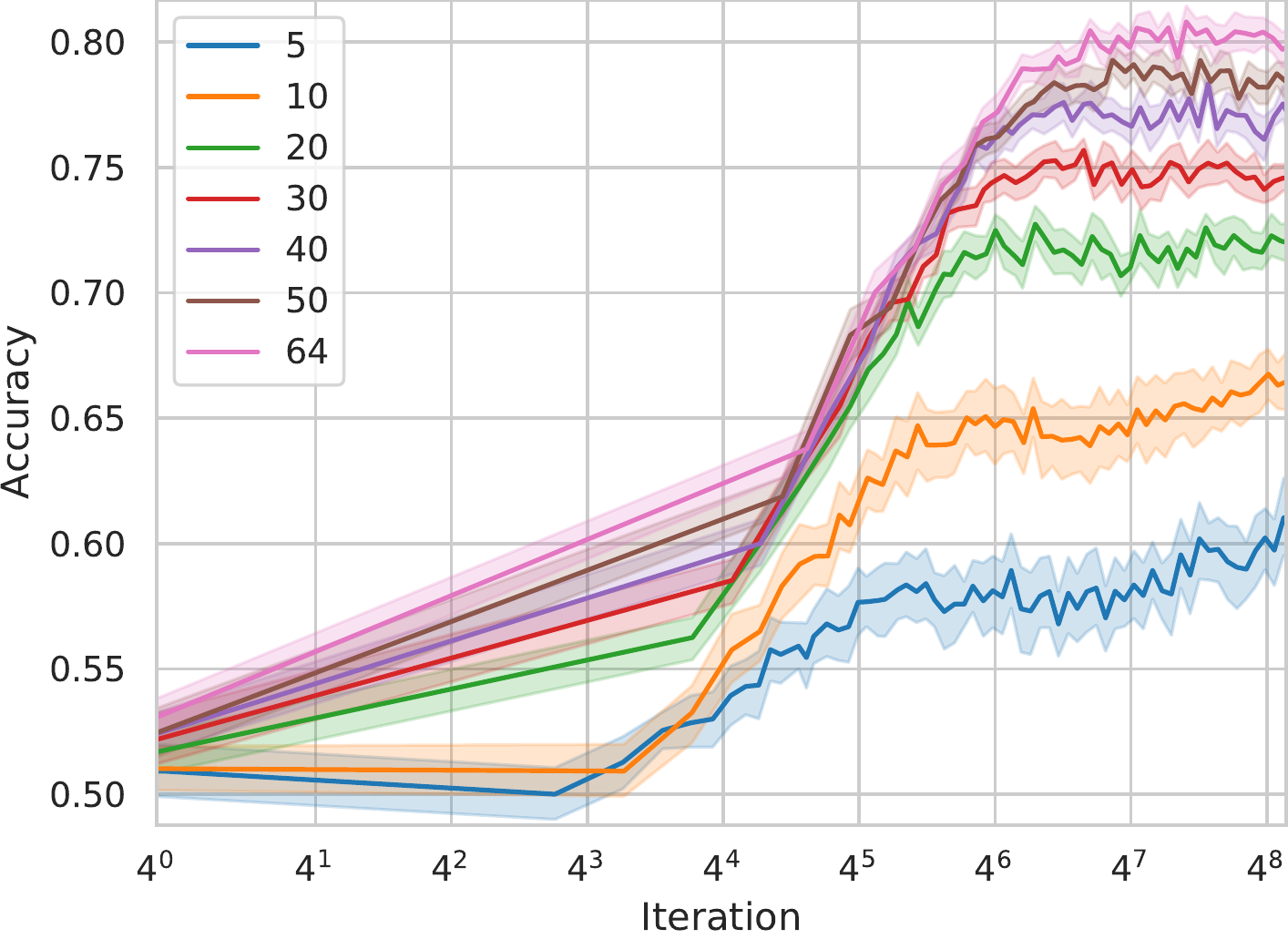}\\
{\bf(d)} 1-shot acc & {\bf(e)} 5-shot acc & {\bf(f)} 10-shot acc & {\bf(g)} 20-shot acc \\
\multicolumn{4}{c}{Mini-ImageNet} \\[0.5em]
\end{tabular}
\vspace{-0.3cm}
    \caption{ {\bf Class-features variability collapse and few-shot performance with learning rate $\eta=2^{-4}$.} We trained WRN-28-4 using SGD with learning rate scheduling on $l\in \{5,10,20,30,40,50,64\}$ source classes (as indicated in the legend). For each dataset, in {\bf(a-c)} we plot the CDNV on the training and test data and the target classes (resp.). In {\bf (d-g)} we plot the 1,5,10 and 20-shot accuracy rates (resp.). 
    } \label{fig:multi_shots}
    \vspace{-0.5cm}
\end{figure}

\begin{figure}[t]
    \centering
    \begin{tabular}{@{}c@{~}c@{~}c@{~}c}
\includegraphics[width=.23\linewidth]{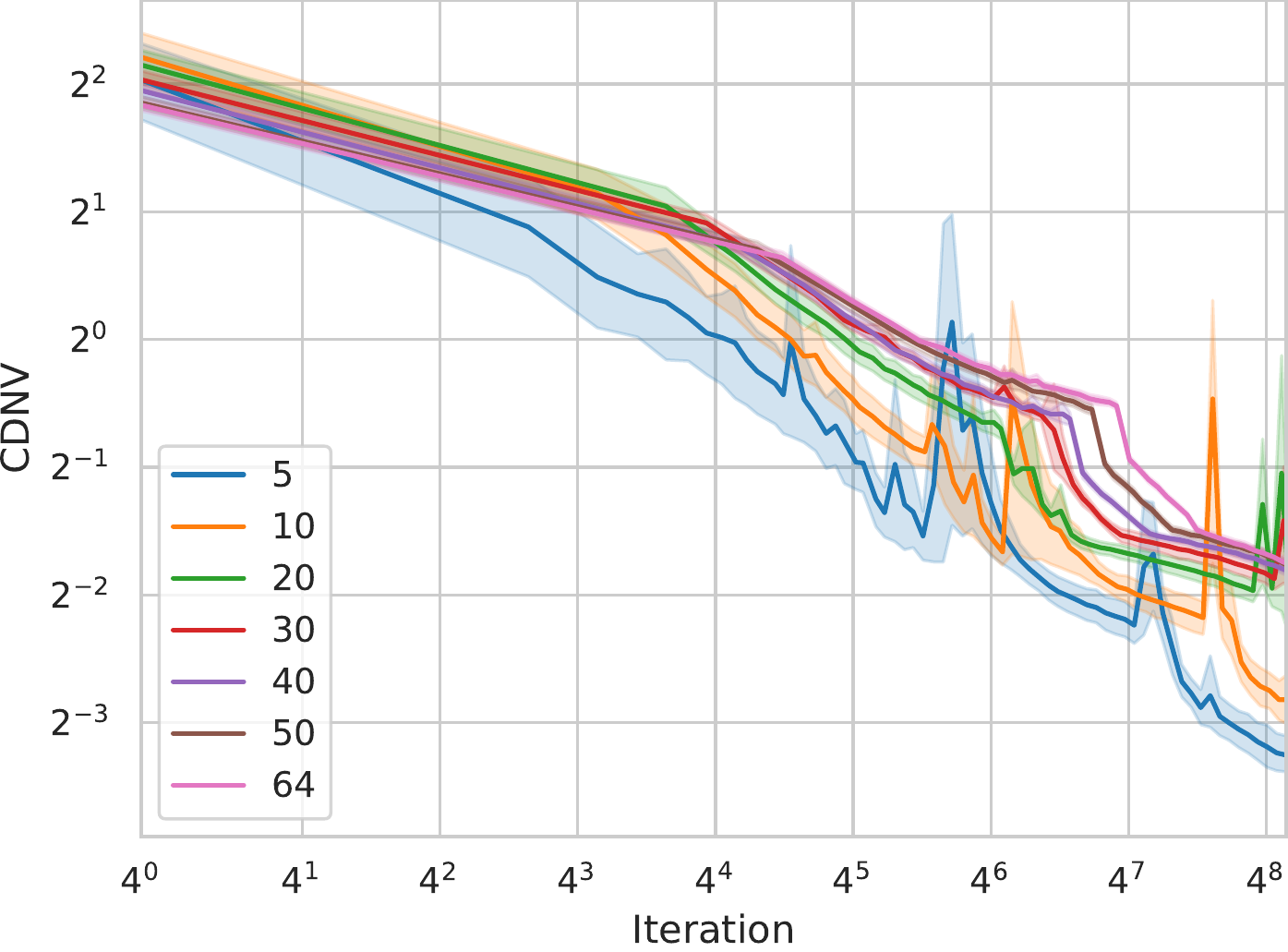}&
\includegraphics[width=.23\linewidth]{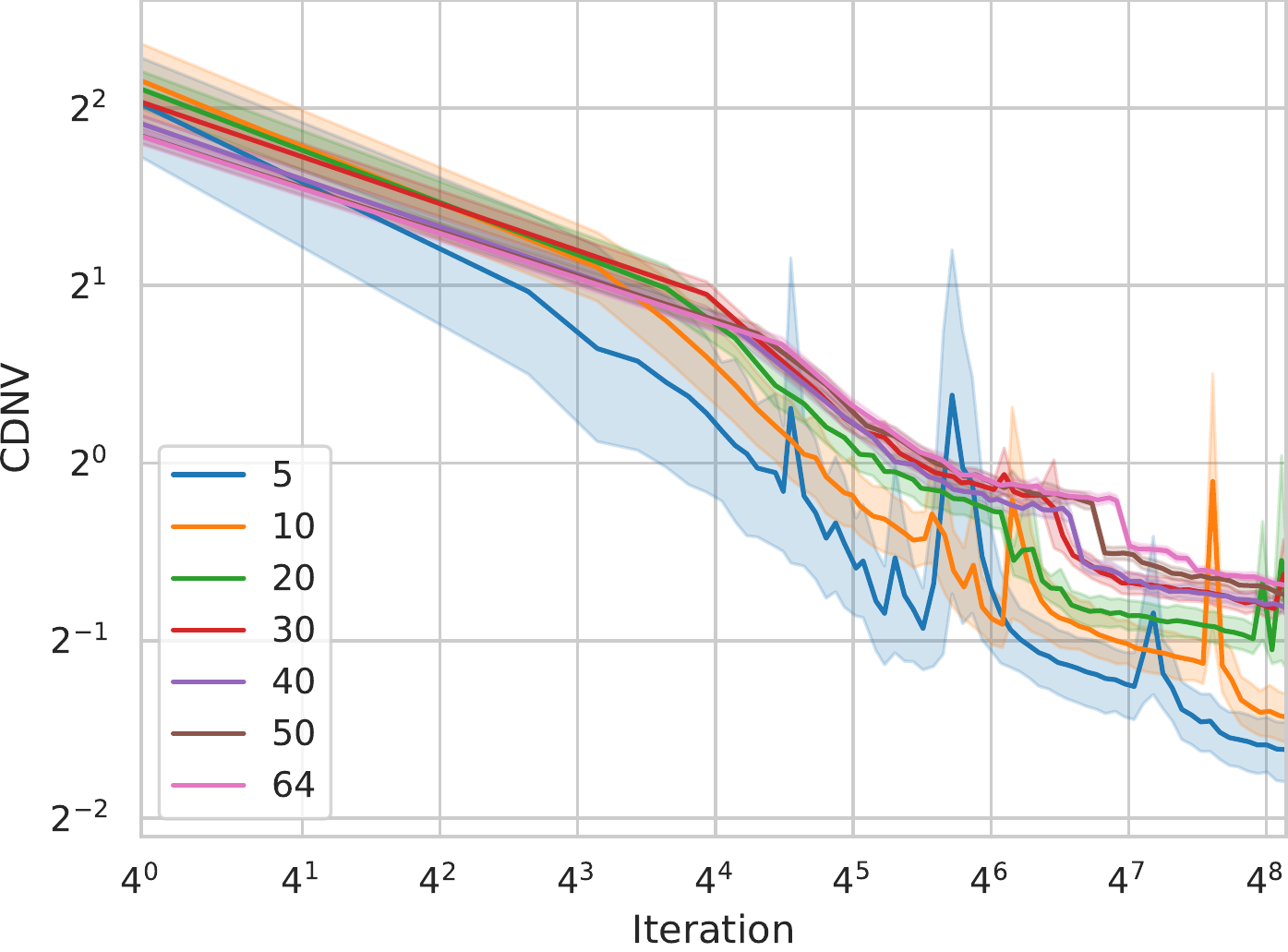}&
\includegraphics[width=.23\linewidth]{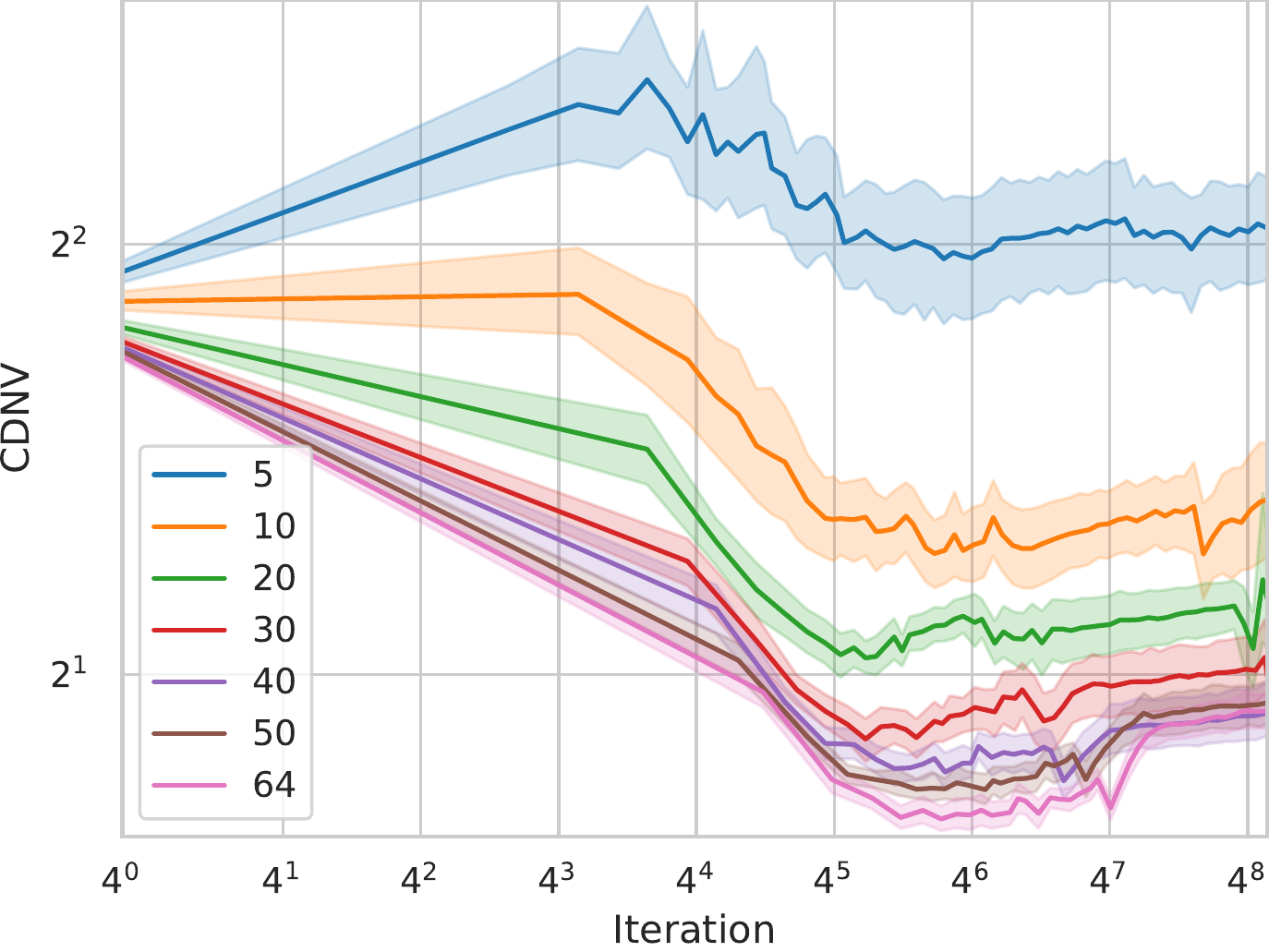}\\
{\bf (a)} CDNV, train & {\bf(b)} CDNV, test & {\bf(c)} CDNV, target \\ 
\includegraphics[width=.23\linewidth]{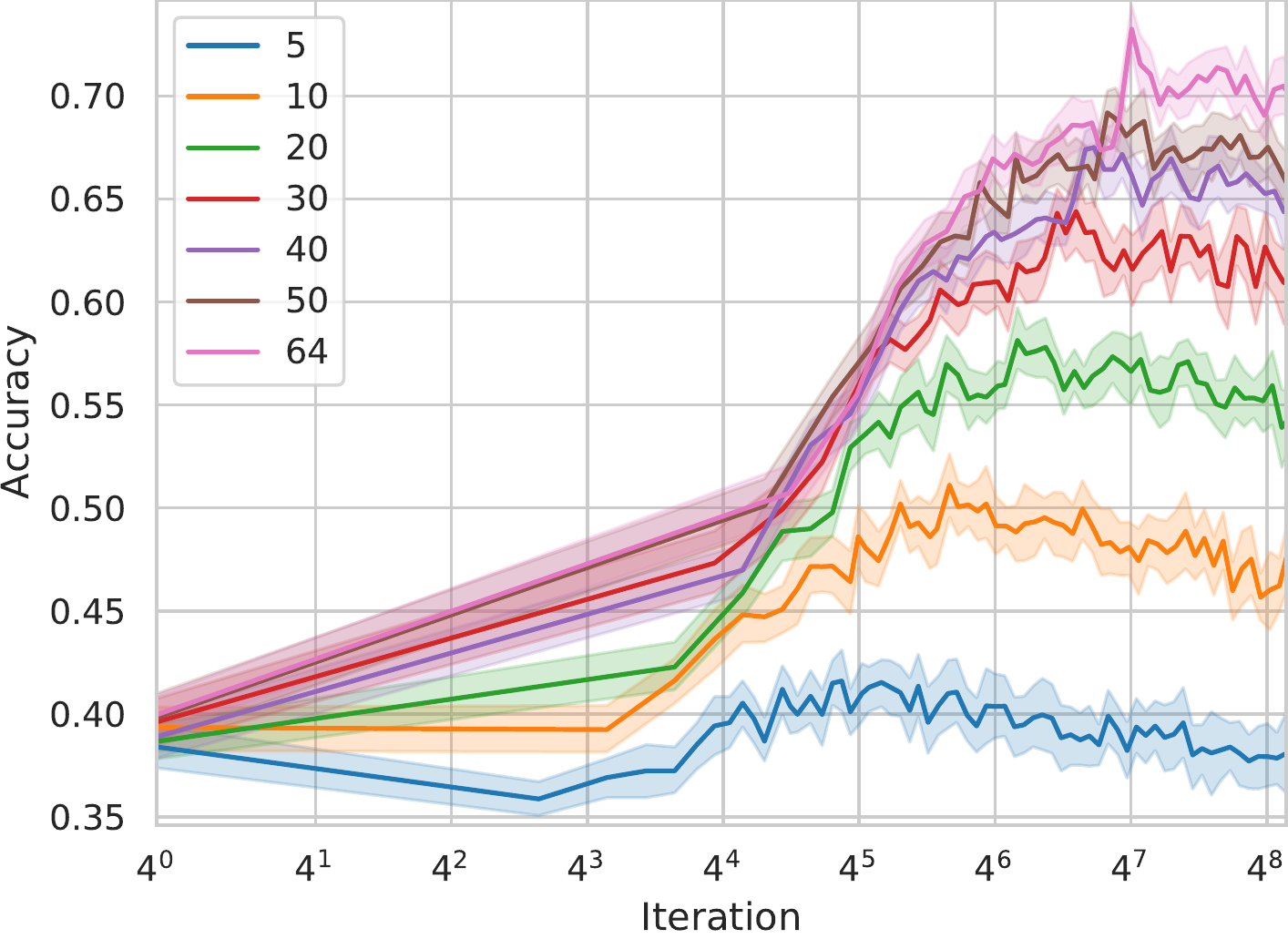}& 
\includegraphics[width=.23\linewidth]{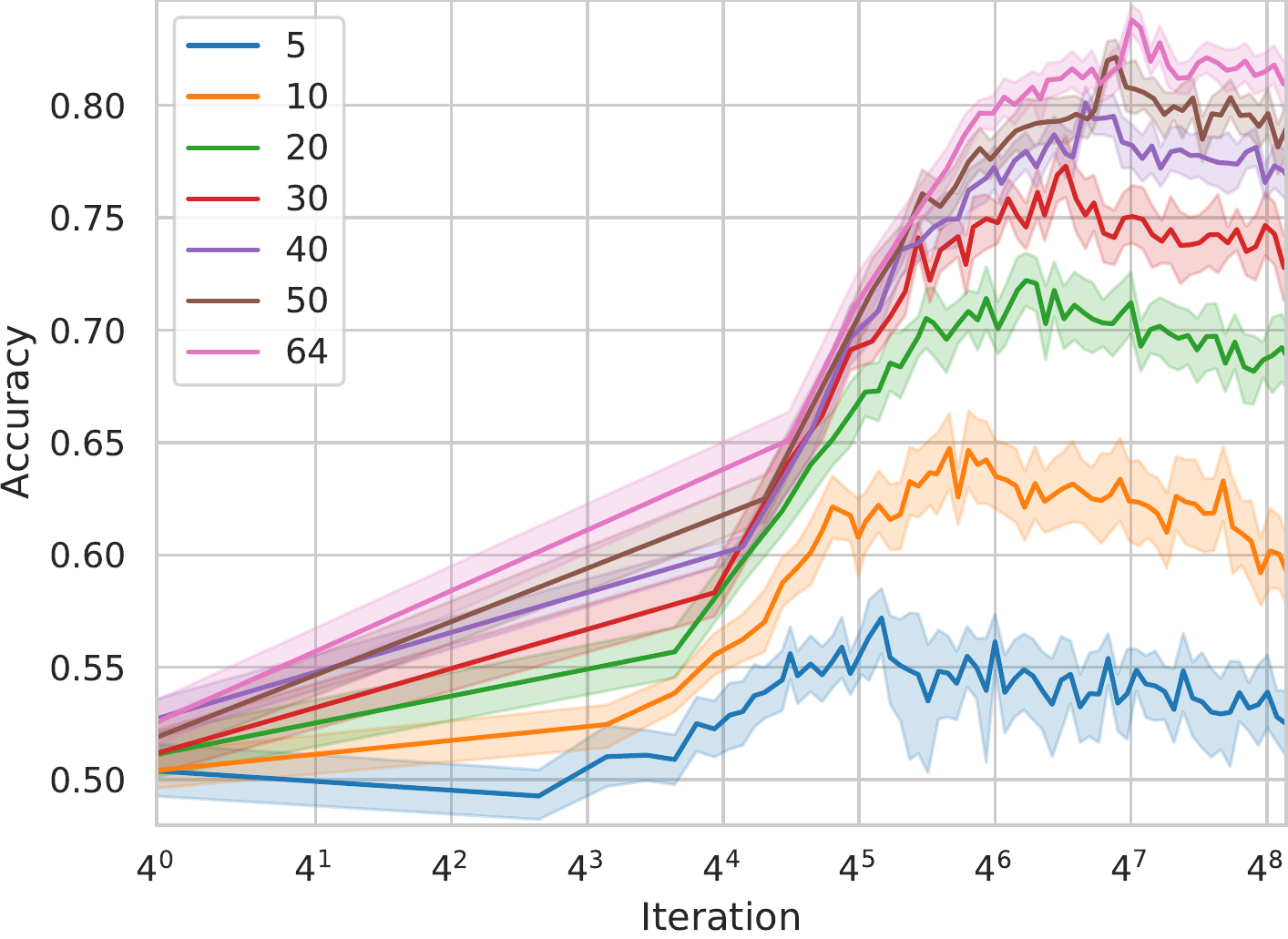}& 
\includegraphics[width=.23\linewidth]{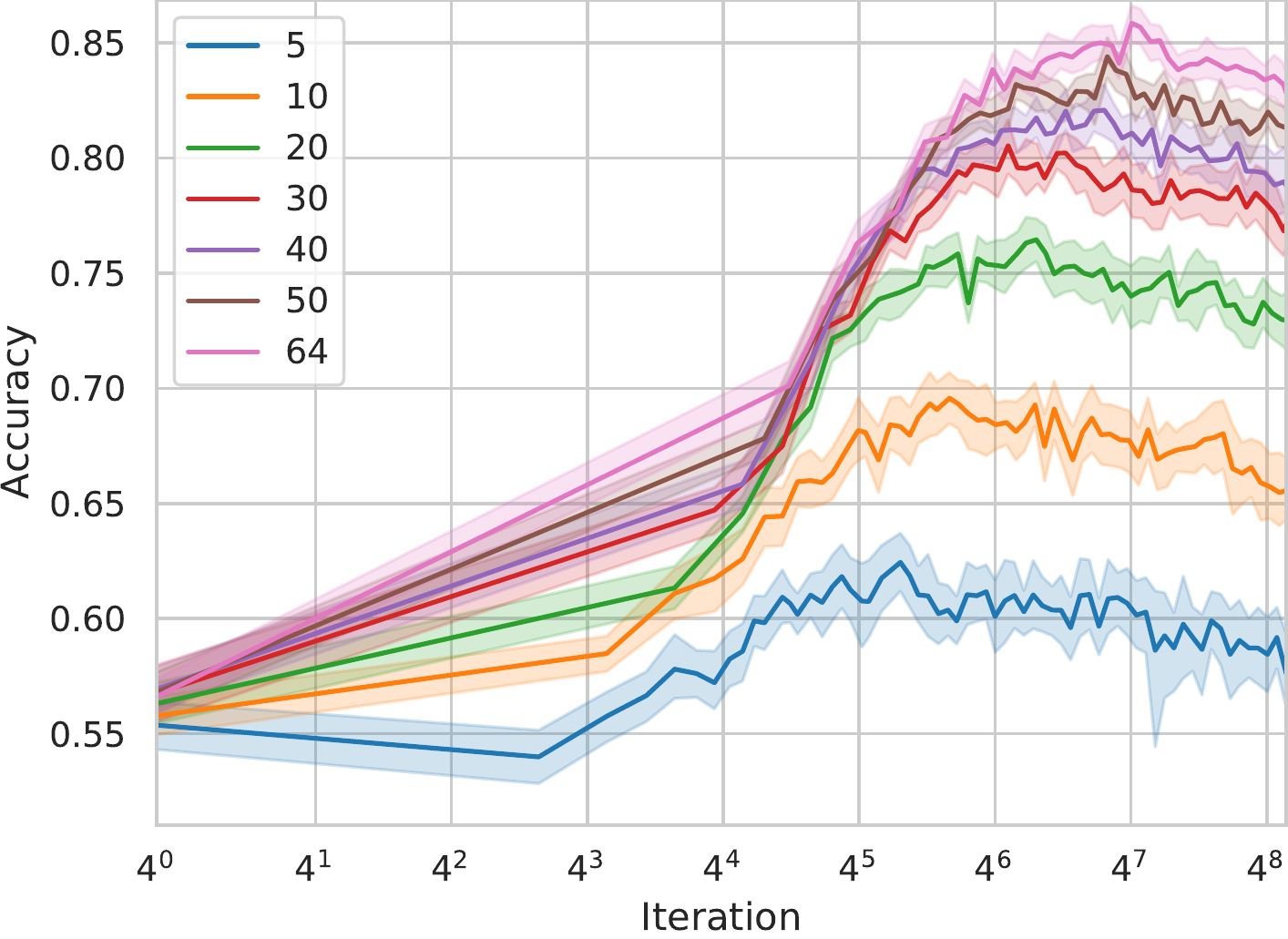}&
\includegraphics[width=.23\linewidth]{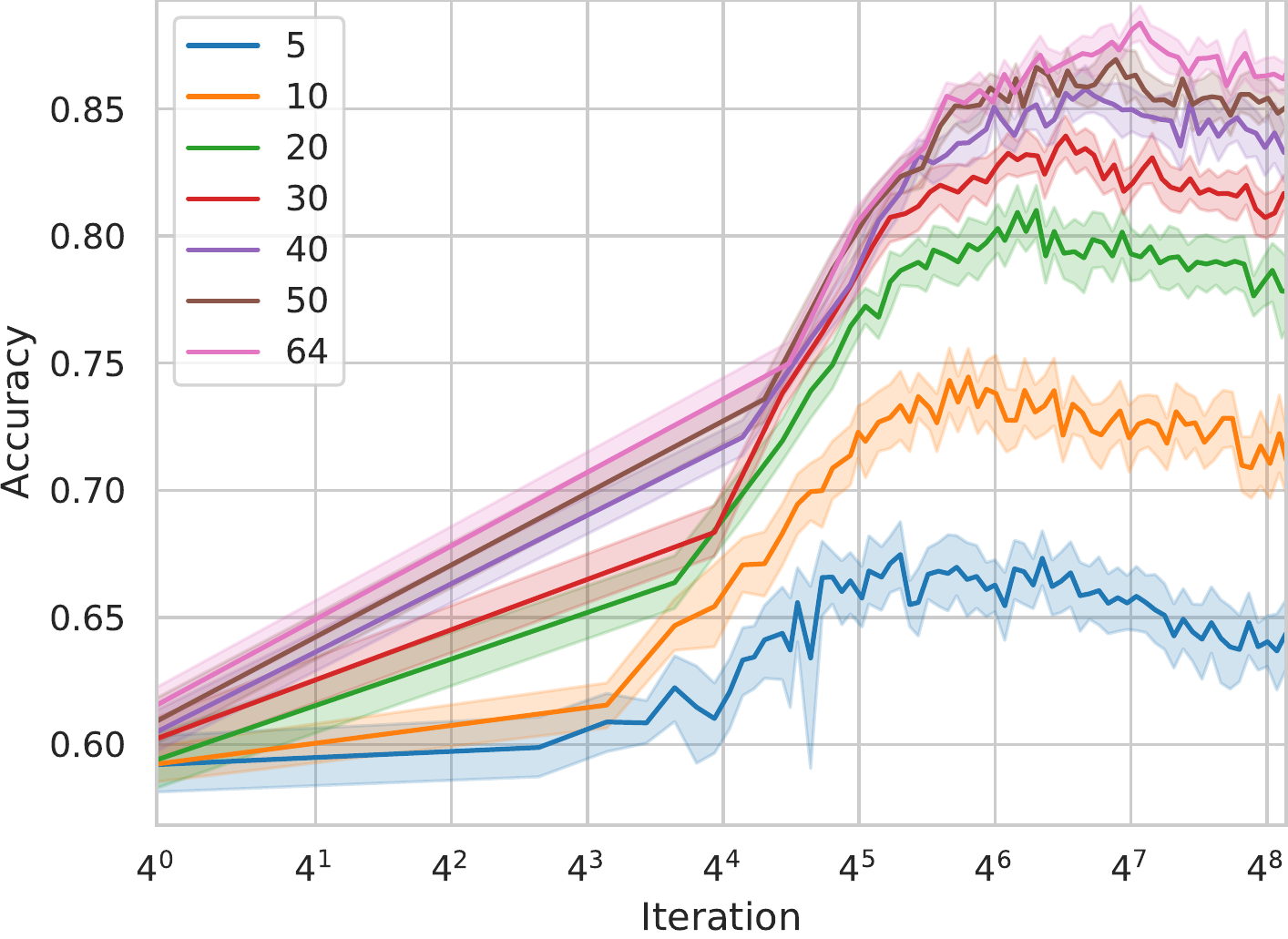}\\
{\bf(d)} 1-shot accuracy & {\bf(e)} 5-shot accuracy & {\bf(f)} 10-shot acc & {\bf(g)} 20-shot acc \\
\multicolumn{4}{c}{CIFAR-FS} \\[0.5em]
\includegraphics[width=.23\linewidth]{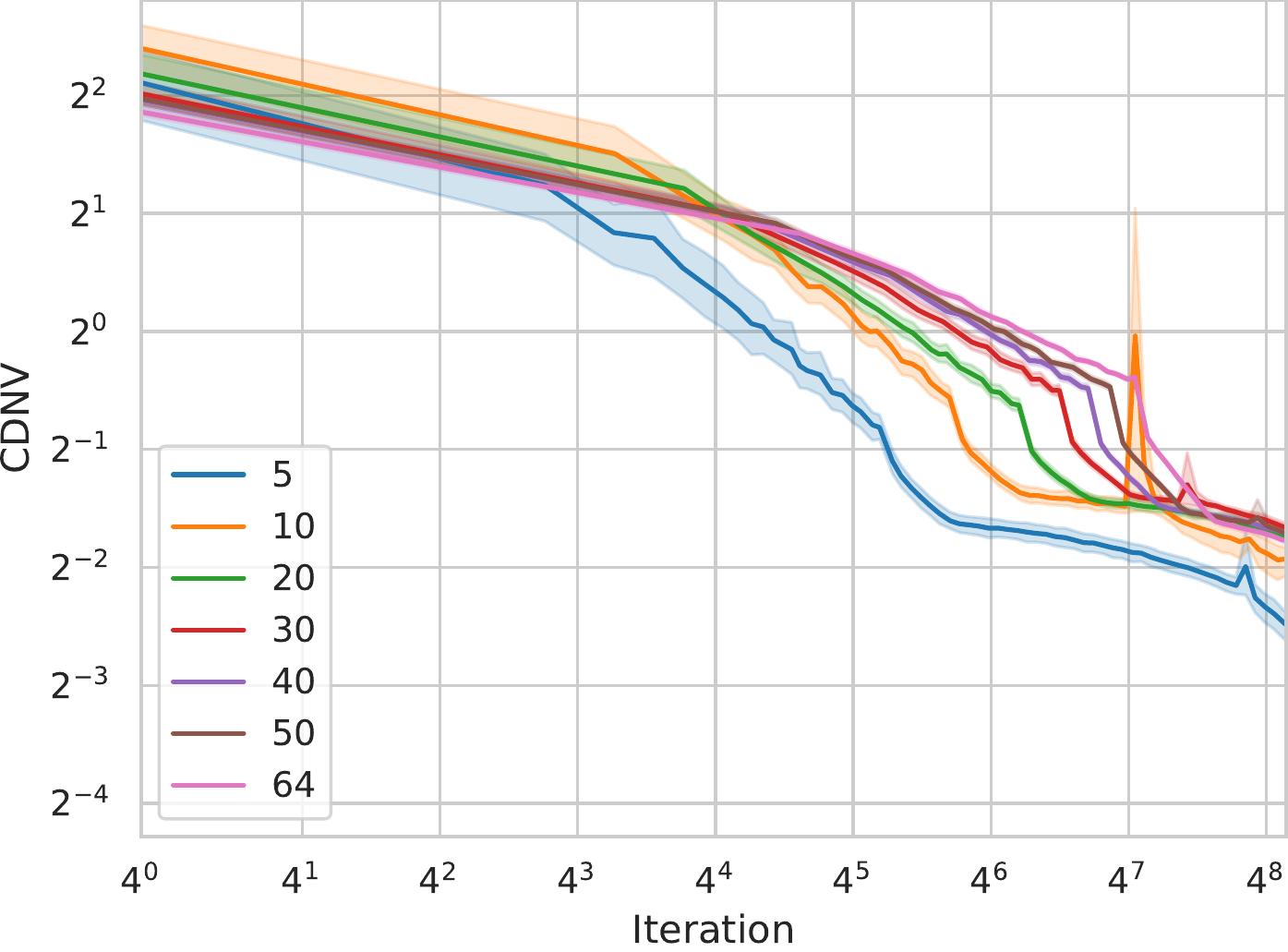}&
\includegraphics[width=.23\linewidth]{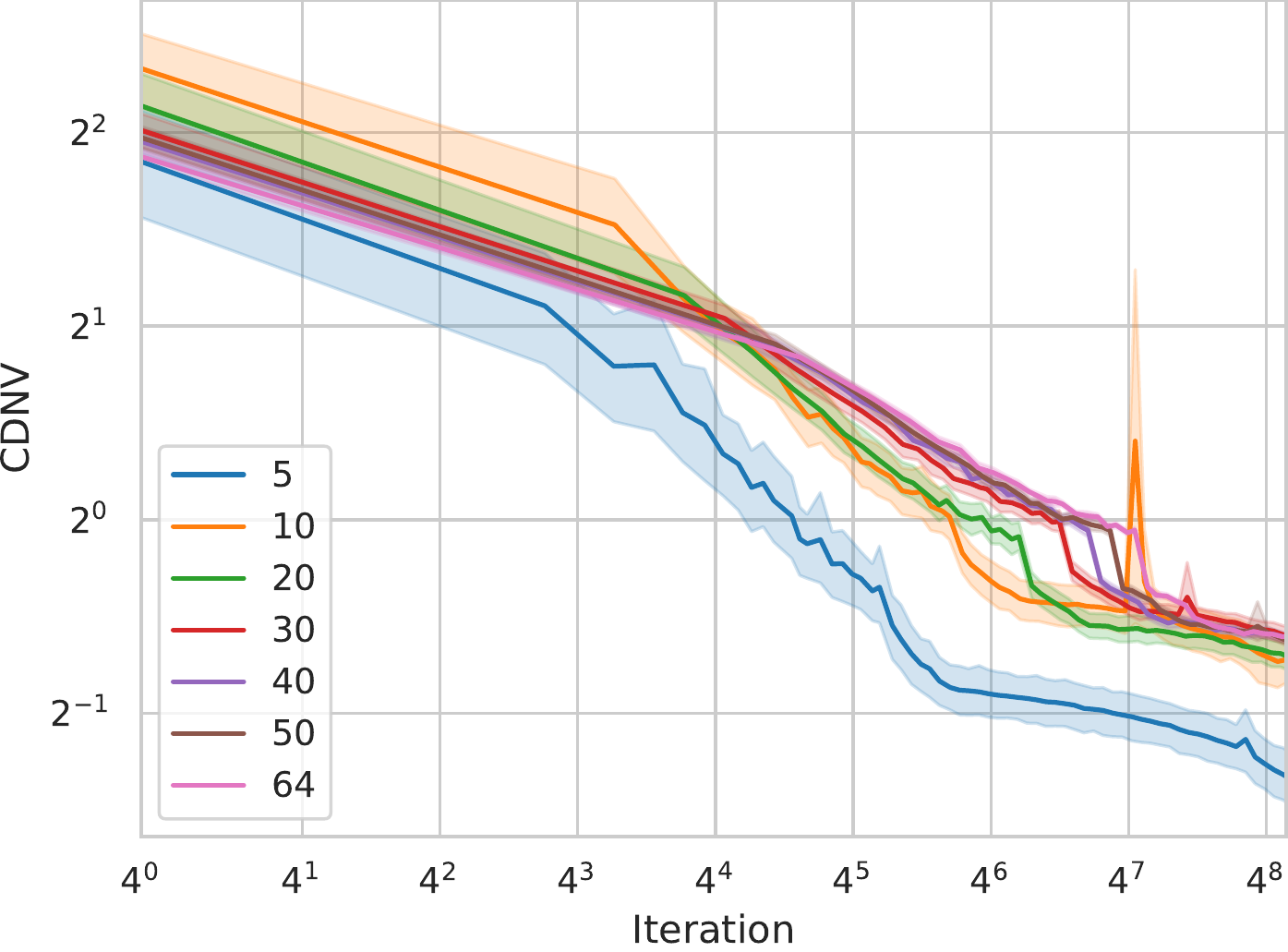}&
\includegraphics[width=.23\linewidth]{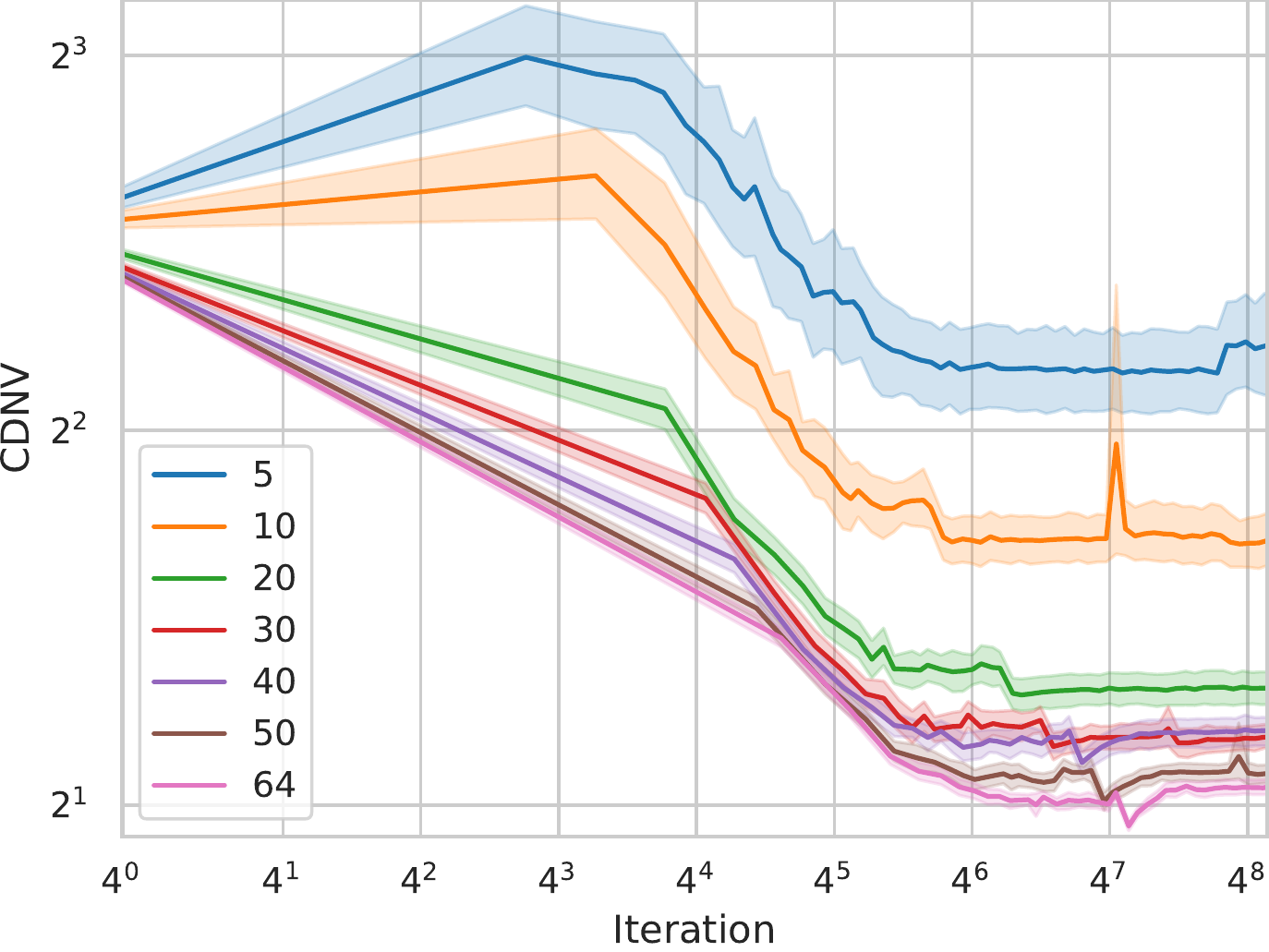}\\
{\bf (a)} CDNV, train & {\bf(b)} CDNV, test & {\bf(c)} CDNV, target \\
\includegraphics[width=.23\linewidth]{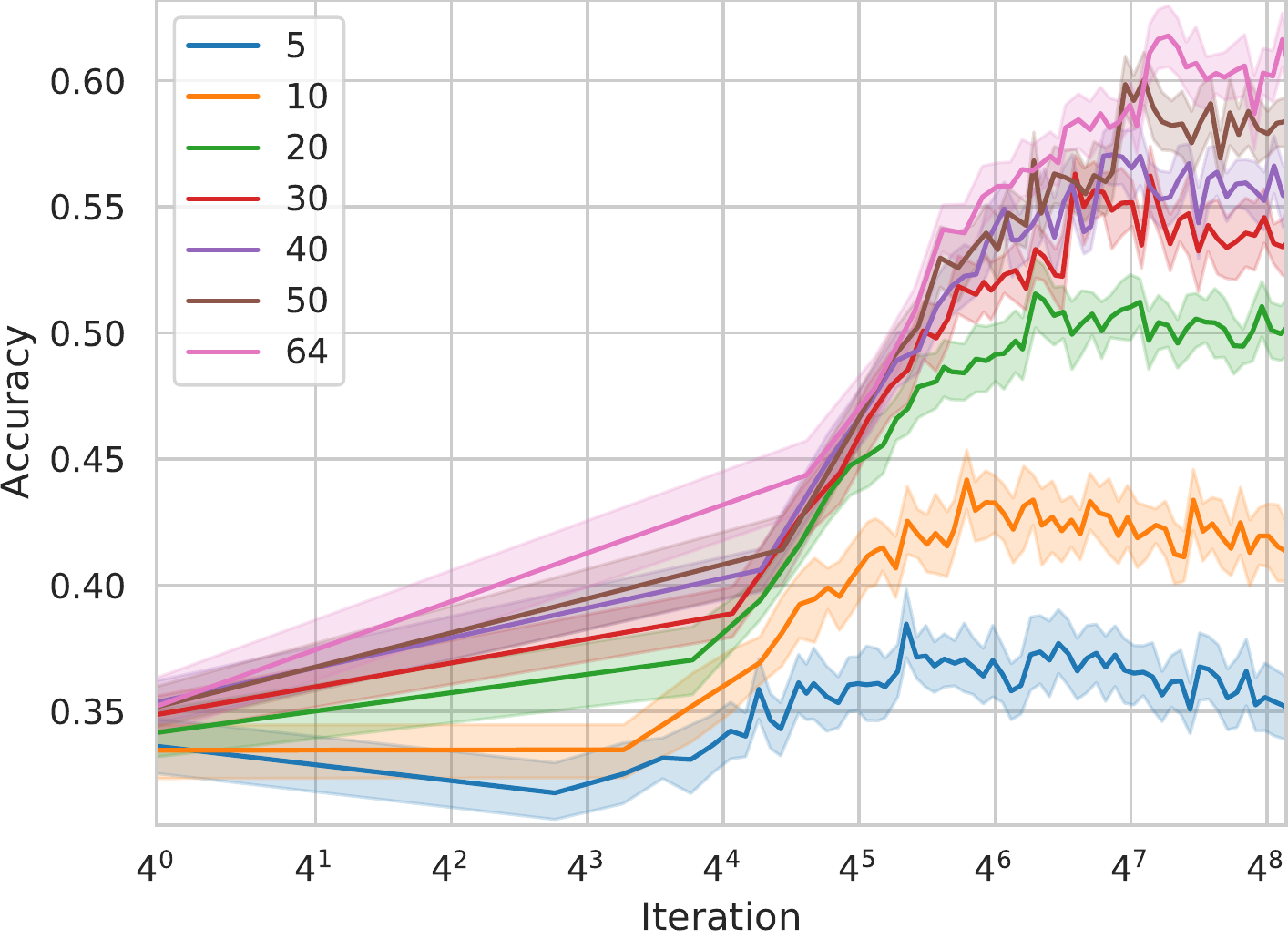}& 
\includegraphics[width=.23\linewidth]{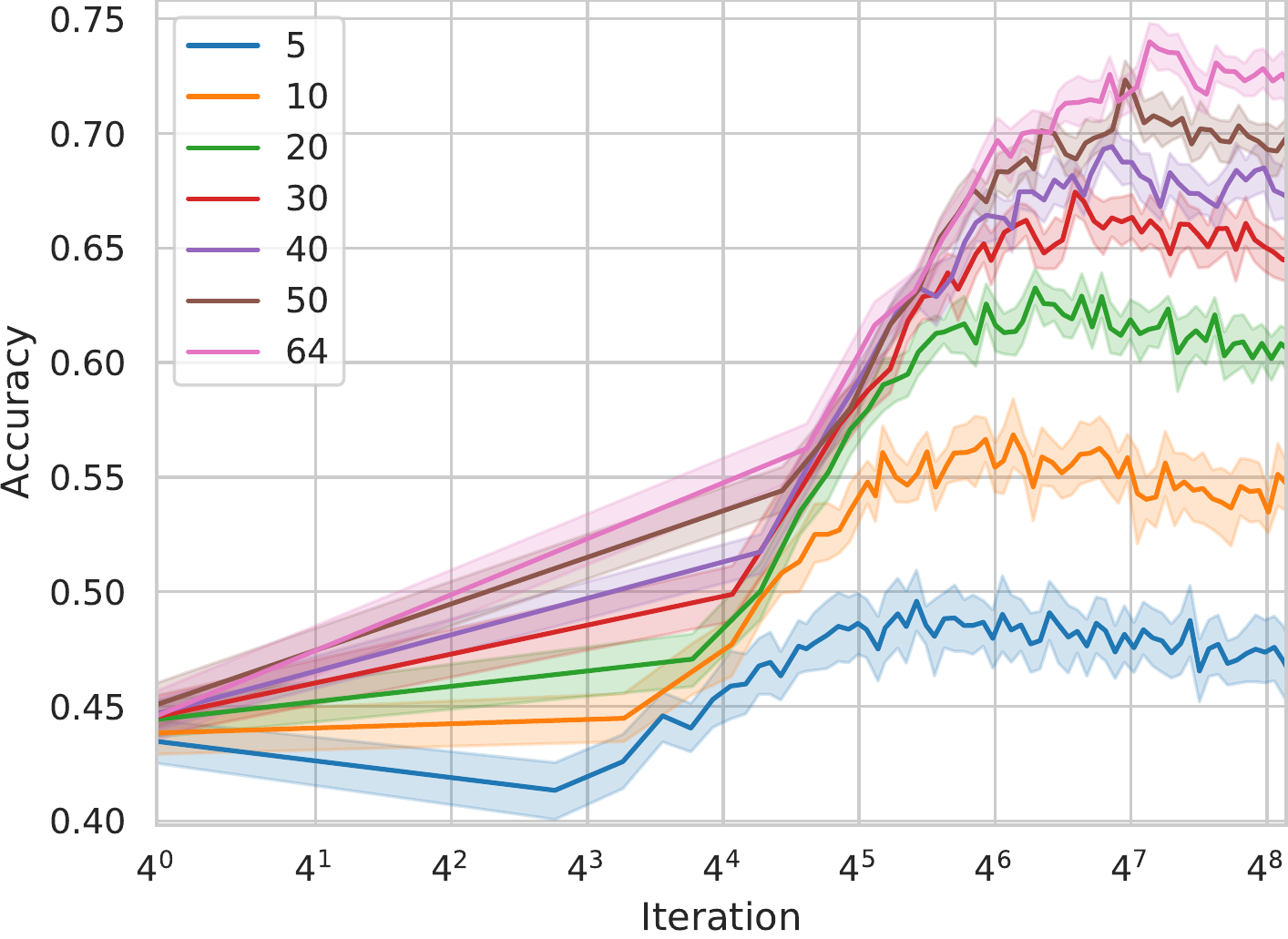}& 
\includegraphics[width=.23\linewidth]{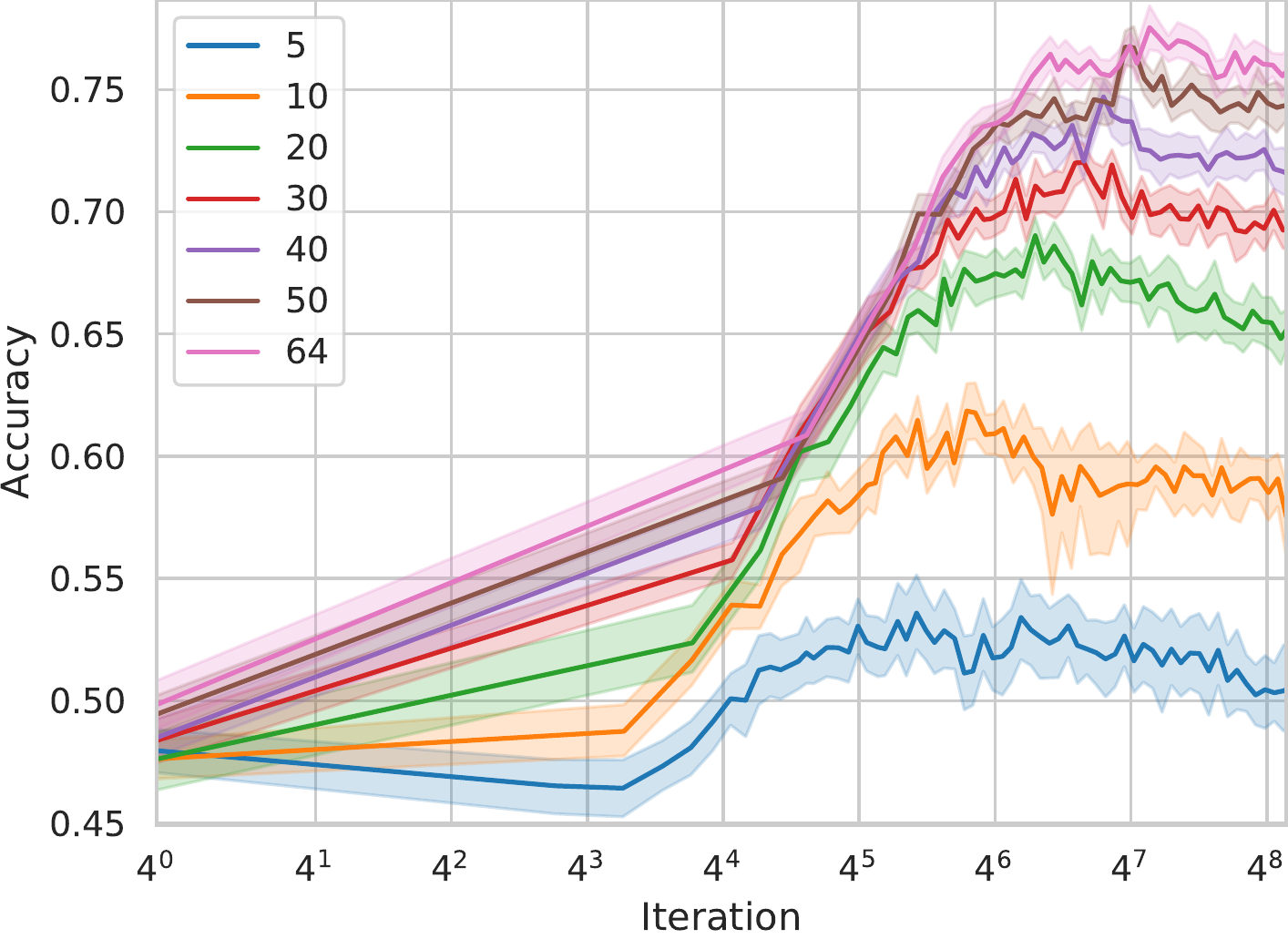}&
\includegraphics[width=.23\linewidth]{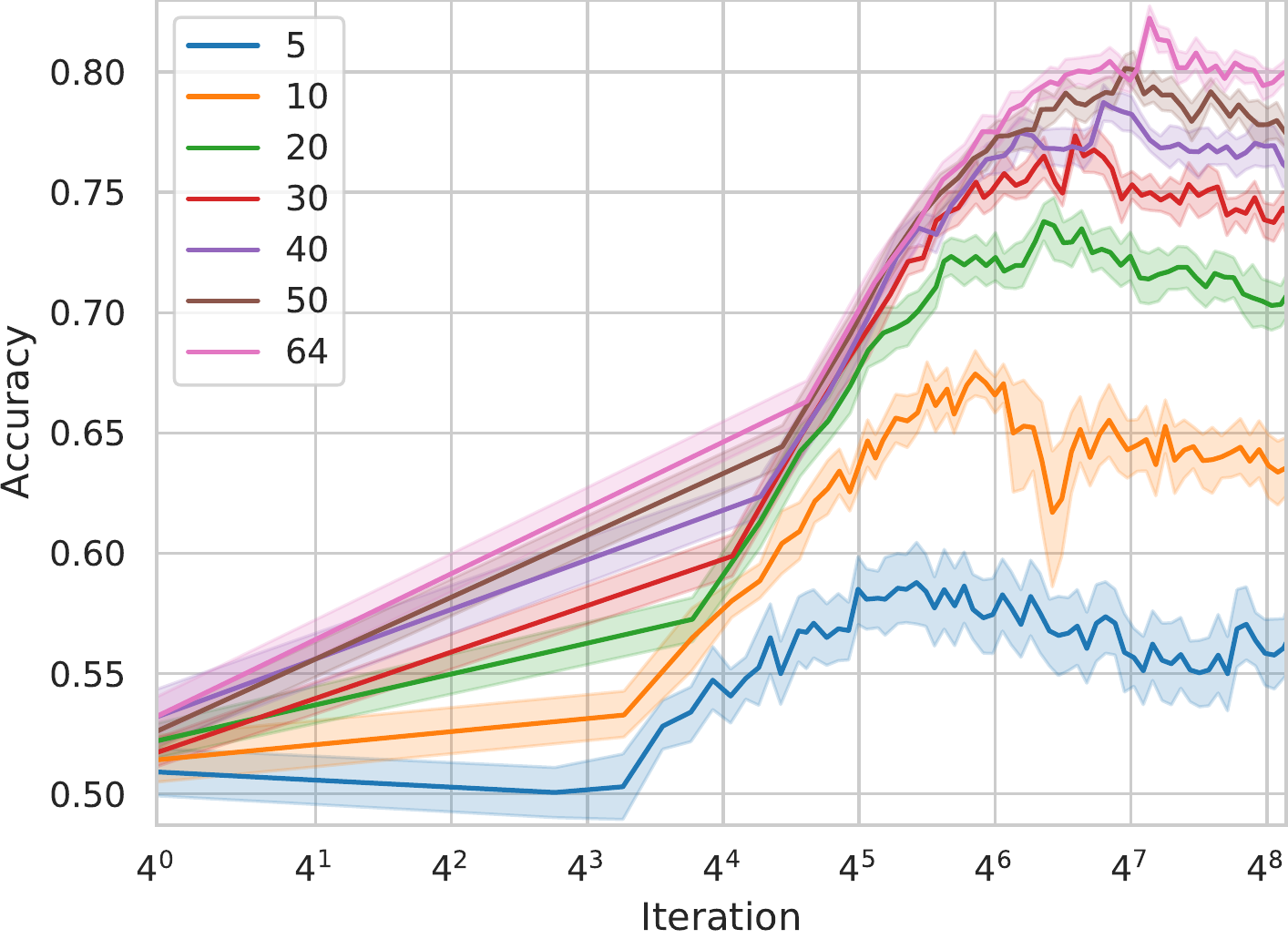}\\
{\bf(d)} 1-shot acc & {\bf(e)} 5-shot acc & {\bf(f)} 10-shot acc & {\bf(g)} 20-shot acc \\
\multicolumn{4}{c}{Mini-ImageNet} \\[0.5em]
\end{tabular}
    \caption{ {\bf Class-features variability collapse and few-shot performance with learning rate scheduling.} We trained WRN-28-4 using SGD with learning rate scheduling on $l\in \{5,10,20,30,40,50,64\}$ source classes (as indicated in the legend). For each dataset, in {\bf(a-c)} we plot the CDNV on the training and test data and the target classes (resp.). In {\bf (d-g)} we plot the 1,5,10 and 20-shot accuracy rates (resp.).
    } \label{fig:lr_scheduling}
\end{figure}

To further demonstrate the relationship between neural collapse and generalization to new classes, we conducted an experiment to examine the effect of varying the number of target samples. In this experiment, we trained WRN-28-4 using SGD and reported the CDNV on the source training data, the source test data (i.e., unseen test samples from the source classes), and the target classes (resp.). In Figure~\ref{fig:multi_shots}(d-g), we plotted the 1, 5, 10 and 20-shot accuracy rates of the network during training time. We also experimented with a learning rate scheduling, starting with learning rate $\eta=0.05$ and decreasing it by a factor of 0.1 every 30 epochs twice. The results of this experiment are provided in Figure~\ref{fig:lr_scheduling}.

Similarly to Figure~\ref{fig:var_nc_main}, as expected, we observed that the few-shot performance improved when the number of source classes increased, while the CDNV on the target classes tended to decrease. Additionally, in line with our theory, the CDNV on the target classes was negatively correlated with the few-shot performance, meaning that better neural collapse resulted in better performance. For example, in Figure~\ref{fig:lr_scheduling}(d-g), it is clear that the peak performance on all few-shot experiments for the case of 64 source classes was achieved around the minimal value of the CDNV on the target classes in Figure~\ref{fig:lr_scheduling}(c). This occurred around iteration 16,000 for CIFAR-FS and around iteration 20,000 for Mini-ImageNet, slightly after the first learning rate decay (at 15,000 and 18,000 steps, respectively). As shown, the performance slightly decreased after the peak iteration, while the CDNV on the target classes slightly increased, as the training started to overfit to the source data and classes. This effect can be mitigated by selecting the final network based on the performance on the source test data, as demonstrated in Table~\ref{tab:main_results}.

%%%%%%%%%%%%%%%%%%%%%%%%%%%%%%%%%%%%%%%%%%%%%%%%%%%%%%%%%%%%%%%
\section{Conclusions}\label{sec:conclusions}
%%%%%%%%%%%%%%%%%%%%%%%%%%%%%%%%%%%%%%%%%%%%%%%%%%%%%%%%%%%%%%%

The use of pre-trained models for transfer learning is a successful approach for addressing overfitting in the low-data regime, but the reasons for this success are not clear. Recently, \citet{galanti2022on} proposed a new perspective on this problem using the concept of neural collapse. However, their theoretical analysis was based on several unrealistic assumptions. In this work, we provide a more comprehensive theoretical analysis by significantly relaxing these assumptions. We study the transfer error of pre-trained feature maps, which measures the expected error of random classifiers trained on few samples using the feature maps for randomly generated classification tasks. We derive a generalization bound on the transfer error of a pre-trained feature map. This bound is decomposed into a measure of the variability collapse between the training samples and additional terms that scale as $\mathcal{O}(m^{-1/2}\textnormal{poly}(n))$ and $\mathcal{O}(l^{-1/2})$, where $m$ is the number of training samples per source class, $n$ is the number of samples per target class, and $l$ is the number of source classes. We show that if neural collapse occurs when training with many source classes and training samples per class, then it requires very few samples to train a linear classifier on top of the learned feature representation that accurately predicts the new classes. These results provide a justification for the recent successes of transfer learning in few-shot tasks, as observed by \citet{DBLP:conf/eccv/TianWKTI20} and \citet{Dhillon2020A}.

\renewcommand\theHsection{\thesection}
\appendix

\newpage

%%%%%%%%%%%%%%%%%%%%%%%%%%%%%%%%%%%%%%%%%%%%%%%%%%%%%%%%%%%%%%%
\section{Proof of Lemma~\ref{prop:marginBoundA}}\label{app:marginBoundA}
%%%%%%%%%%%%%%%%%%%%%%%%%%%%%%%%%%%%%%%%%%%%%%%%%%%%%%%%%%%%%%%

To prove Lemma~\ref{prop:marginBoundA}, we apply Corollary~3 of \citet{Maurer2019UniformCA}, which provides generalization bounds for settings where the total loss over a dataset is a nonlinear function of the dataset (and not the usual average of losses computed independently for every data point). We will apply this result to the average pairwise error for distinguishing two classes (note that although we will consider average loss functions, here the terms in the sum are not independent as there is overlap among the classes used in each pair).  In order to do so,  we use the following definitions (their Definition~1):
Let $\|\cdot \|$ be the Euclidean norm, $\cZ \subset \R^{d}$ be any set and $A:\cZ^l \to \mathbb{R}$. For $\bz = (z_1,\dots,z_l)$, $j \in \{1, \dots, l\}$ and $z'_j \in \cZ$, we define the $j$-th partial difference operator on $A$ as
\begin{equation*}
D^{j}_{z'_j}A(\bz) ~:=~ A(z_1,\dots,z_{j-1},z_j,z_{j+1}\dots,z_l) - A(z_1,\dots,z_{j-1},z'_j,z_{j+1},\dots,z_l).
\end{equation*}
In addition, we denote
\begin{equation*}
M(A) ~:=~ \max_{j \in [l]} \sup_{\substack{\bz \in \cZ^{l} \\ z'_j \in \cZ:~z'_j\neq z_j}} \frac{D^{j}_{z'_j}A(\bz)}{\|z_j-z'_j\|},
\end{equation*}
and 
\begin{equation*}
J(A) ~:=~ l\cdot \max_{\substack{r,j \in [l]\\ r\neq j}}\sup_{\substack{\bz \in \cZ^l\\ z'_j\in \cZ:~ z'_j\neq z_j \\ z'_r \in \cZ:~ z'_r \neq z_r}}  \frac{D^{r}_{z'_r}D^{j}_{z'_j} A(\bz)}{\|z_j-z'_j\|},
\end{equation*}
and the maximal difference
\begin{equation*}
K(A) ~:=~ \max_{j \in [l]} \sup_{\substack{\bz \in \cZ^{l} \\ z'_j \in \cZ:~z'_j\neq z_j}} D^{j}_{z'_j}A(\bz).
\end{equation*}

In addition, we need the notion of Rademacher complexity: the Rademacher complexity of a set of vectors $Y \subset \R^s$ is defined as $R(Y) := \E_{\epsilon}\sup_{y\in Y} \left[\langle \epsilon, y\rangle\right]$, where $\epsilon = (\epsilon_1,\dots,\epsilon_s) \sim U[\{\pm 1\}]^s$ \citep[see, e.g.,][]{MohriRostamizadehTalwalkar18}.

Now assume that $W=(W_1,\ldots,W_l)$ is an i.i.d.\ random vector with each $W_i, i \in [l]$ taking values in a set $\cW$, and let $\cH$ be a finite class of functions $H: \cW \to \cZ$. When $H$ is applied to the vector $W$, it is applied coordinatewise, that is, $H(W):=(H(W_1),\ldots,H(W_l))$, and we denote $\cH(W):=\{H(W): H \in \cH\}$.
For this setting, the following generalization bound is proved in Corollary~3 of \citet{Maurer2019UniformCA}:\footnote{The original statement makes use of Gaussian complexity instead of Rademacher complexity. The Gaussian complexity of a set $Z\subset \R^l$ is upper bounded by $2\sqrt{\log(l)} R(Z)$ \citep[Exercise~5.5]{wainwright19}.} for any $\delta' \in (0,1)$, with probability at least $1-\delta'$ over the selection of $W$, for all $H \in \cH$ we have
\begin{equation}\label{eq:maurer-ext}
\begin{aligned}
\E_{W}\left[A(H(W))\right] ~&\le~ A(H(W)) + 2\sqrt{2\pi\log(l)} \big(2M(A) + J(A)\big)\cdot \E_{W}[R(\cH(W))] \\
& \qquad + K(A)\cdot \sqrt{l\log\left(\fr{1}{\delta'}\right)}.
\end{aligned}
\end{equation}
Importantly, the above inequality does not depend on the size of $\cH$, which is only required to be finite. Therefore, the inequality can be easily extended to any infinite function class $\cH$ that admits a sequence of finite covers $\cH_t \subset \cH, t=1,2,\ldots$ such that for every $H \in \cH$ there exists a $H_t \in \cH_t$ such that $|H(w)-H_t(w)| \le 1/t$ for any $w \in \cW$. Applying \eqref{eq:maurer-ext} to $\cH_t$ and letting $t$ to infinity then implies that \eqref{eq:maurer-ext} also holds for $\cH$, since
$R(\cH_t(W)) \le R(\cH(W))$ for any $W \in \cW^l$ by definition.
In the sequel, we use this slightly extended version of Corollary~3 of \cite{Maurer2019UniformCA}.

Now we are ready to proceed with the proof of Lemma~\ref{prop:marginBoundA}; we start with repeating the statement for convenience.

\marginBoundA*

\begin{proof} We start the proof by reducing the error bound to a classification problem between two classes: For any $f \in \cF$,
\begin{align*}
\mathcal{L}_{\cD}(f) ~&=~ \E_{P_1,\dots,P_k}\E_{S_1,\dots,S_k}\E_{(x,y) \sim P}\big[\bI[\hnn(x) \neq y] \big]\\
~&=~ \E_{P_1,\dots,P_k}\E_{S_1,\dots,S_k}\E_{(x,y) \sim P}\big[\bI\left[\exists i \neq y:~\|f(x) - \mu_f(S_i)\| \le \|f(x) - \mu_f(S_y)\|\right] \big]\\
~&\le~ \E_{P_1,\dots,P_k}\E_{S_1,\dots,S_k}\E_{(x,y) \sim P}\left[\sum_{i\neq y}\bI\left[\|f(x) - \mu_f(S_i)\| \le \|f(x) - \mu_f(S_y)\|\right]\right] \\
~&=~ \E_{P_1,\dots,P_k}\E_{S_1,\dots,S_k}\E_{y \sim U[\{1,\ldots,k\}]}\E_{x \sim P_y}\left[\sum_{i\neq y}\bI\left[\|f(x) - \mu_f(S_i)\| < \|f(x) - \mu_f(S_y)\|\right]\right] \\
~&=~ \E_{y \sim U[\{1,\ldots,k\}]}\sum_{i\neq y}\E_{P_i,P_y}\E_{S_i,S_y}\E_{x \sim P_y}\big[\bI\left[\|f(x) - \mu_f(S_y)\| < \|f(x) - \mu_f(S_i)\|\right]\big] \\
~&=~ (k-1) \cdot \E_{P_1,P_2}\E_{S_1,S_2}\E_{x_1 \sim P_1}\big[\bI\left[\|f(x_1) - \mu_f(S_2)\| < \|f(x_1) - \mu_f(S_1)\|\right] \big]\\
\end{align*}
where $P_i \sim \mathcal{D}$ and $S_i \sim P^n_i$. To further bound the right-hand side above, we use the \emph{soft-margin} relaxation $\ell_\Delta$ (defined in \eqref{eq:ellDelta}):
\begin{align*}
\bI\left[\|f(x_1) - \mu_f(S_2)\| < \|f(x_1) - \mu_f(S_1)\|\right]
&\le \ell_\Delta\left(\|f(x_1) - \mu_f(S_1)\| - \|f(x_1) - \mu_f(S_2)\|\right),   
\end{align*}
implying
\begin{equation}
\label{eq:LD1}
   \mathcal{L}_{\cD}(f) \leq  (k-1) \cdot \E_{P_1,P_2}\E_{S_1,S_2}\E_{x_1 \sim P_1}\big[
   \ell_\Delta\left(\|f(x_1) - \mu_f(S_1)\| - \|f(x_1) - \mu_f(S_2)\|\right)
   \big].
\end{equation}
By the random selection of $P_1, P_2$, $\tP_c, c \in [l]$, the main term on the right-hand side above (i.e., omitting the $(k-1)$ factor) is equal to
\begin{equation}
\label{eq:A-motivation}
\Avg_{i \neq j \in [l]} \E_{\tP_i,\tP_j}\E_{\hS_i,\hS_j}\E_{x_i \sim \tP_i}[
   \ell_\Delta(\|f(x_i) - \mu_f(\hS_i)\| - \|f(x_i) - \mu_f(\hS_j)\|)].
\end{equation}
To write this averaging succinctly, we introduce some notation.
Let $\hU_c := (x_c,\hS_c)$ for $c \in [l]$ where $x_c \sim \tP_c$ and $\hS_c \sim \tP_c^n$, independently for all $c$, and $\hcU := \{\hU_c\}^{l}_{c=1}$ (recall that this also depends on the random choice of the $\tP_c$). Furthermore, let $\mathcal{Q} := \{(x,S) \mid x \in \cX,~ S \in \cX^{n} \}$ denote the set of all point-sample pairs (note that $\hU_c \in \mathcal{Q}$).
For any $f \in \cF$, define $H_f: \mathcal{Q} \to \R$ as $H_f(x,S) = (f(x),\mu_f(S))$ and for any $\cF' \subset \cF$ and $\mathcal{Q}' \subset \mathcal{Q}$, let $H_{\cF'}:=\{ H_f: f \in \cF'\}$ and $H_{\cF'}(\mathcal{Q}'):=\{H_f(x,S): (x,S) \in \mathcal{Q}', f \in \cF'\}$. Finally define the aggregation function $A_\Delta:H_{\cF}(\mathcal{Q})^{l} \to \R$ as
\begin{equation*}
A_\Delta(\bz) ~:=~ \Avg_{i\neq j \in [l]}[\ell_\Delta(\|u_{i}-v_{i}\|-\|u_{i}-v_{j}\|)],
\end{equation*}
where $\bz = (z_1,\dots,z_l)$ such that $z_i = (u_i,v_i) \in H_{\cF}(\mathcal{Q})$ for all $i \in [l]$.
Then 
\begin{equation*}
A_\Delta(H_f(\hcU)) = \Avg_{i\neq j \in [l]} [\ell_\Delta(\|f(x_i)-\mu_f(\hS_i)\| - \|f(x_i)-\mu_f(\hS_j)\|)]
\end{equation*}
by definition, and \eqref{eq:A-motivation} can be rewritten as $\E_{\hcU}[A_\Delta(H_{f}(\hcU))]$ and from \eqref{eq:LD1} we get 
\begin{equation}
\label{eq:LD2}
  \mathcal{L}_{\cD}(f) \leq  (k-1) \cdot \E_{\hcU}[A_\Delta(H_{f}(\hcU))].
\end{equation}
Our next step is to bound the expectation on the right-hand side above with the randomly selected samples of class-conditional distributions $\tP_c, c\ \in [l]$ for a fixed $\Delta$. Let $\delta' \in (0,1)$ and $\cF' \subset \cF$ with $\alpha=\sup_{f \in \cF'} \cC(f)<\infty$. It is easy to see that $\cF'$ can be approximated arbitrarily closely with a finite subset $\cF''$ of $\cF$ in $L_\infty$ distance, and hence $H_{\cF}$ can also be arbitrarily closely approximated by the corresponding finite set $H_{\cF''}$. Then, by the extension of Corollary~3 of \cite{Maurer2019UniformCA} discussed after \eqref{eq:maurer-ext}, 
with probability at least $1-\delta'$ over the selection of $\hcU$, for all $f \in \cF'$ we have
\begin{equation}\label{eq:maurer2}
\begin{aligned}
\E_{\hcU}\left[A_\Delta(H_{f}(\hcU))\right] &\le A_\Delta(H_{f}(\hcU)) + 2\sqrt{2\pi\log(l)} \big(2M(A_\Delta) + J(A_\Delta)\big)\cdot \E_{\hcU}[R(H_{\cF'}(\hcU))] \\
& \qquad + K(A_\Delta)\cdot \sqrt{l\log\left(\fr{1}{\delta'}\right)}.
\end{aligned}
\end{equation}
Next we bound the terms $K(A_\Delta)$, $M(A_\Delta)$, and $J(A_\Delta)$.

\paragraph{Bounding $K(A_\Delta)$, $M(A_\Delta)$ and $J(A_\Delta)$.} Let $\Delta > 0$, $j\in [l]$ and $z_i = (u_i,v_i) \in \R^{p}\times \R^{p}$ for $i\in [l]$. Let $\bz = (z_1,\dots,z_l)$ and $\bz = (z'_1,\dots,z'_l)$, where $z_i = z'_i$ for all $i\neq j$. We have
\begin{align}
\lefteqn{D^j_{z'_j}A_\Delta(\bz) = A_\Delta(\bz)-A_\Delta(\bz')} \nonumber \\ 
&= \frac{1}{l(l-1)}\sum_{i_1\neq i_2\in [l]}\left(\ell_\Delta(\|u_{i_1} - v_{i_1}\| - \|u_{i_1} - v_{i_2}\|) - \ell_\Delta(\|u'_{i_1} - v'_{i_1}\| - \|u'_{i_1} - v'_{i_2}\|)\right) \nonumber  \\
&= \frac{1}{l(l-1)}\sum_{i\in [l]:~i\neq j}\left(\ell_\Delta(\|u_i - v_i\| - \|u_i - v_j\|) - \ell_\Delta(\|u_i - v_i\| - \|u_i - v'_j\|)\right) \nonumber \\
&\quad+ \frac{1}{l(l-1)}\sum_{i\in [l]:~i\neq j}\left(\ell_\Delta(\|u_j - v_j\| - \|u_j - v_i\|) - \ell_\Delta(\|u'_j - v'_j\| - \|u'_j - v_i\|)\right). \label{eq:lipsc0}
\end{align}
Since $\ell_\Delta$ is $\Delta^{-1}$-Lipschitz and by the triangle inequality
\begin{align}
\lefteqn{
\ell_\Delta(\|u_i - v_i\| - \|u_i - v_j\|) - \ell_\Delta(\|u_i - v_i\| - \|u_i - v'_j\| )} \nonumber \\
&\le \Delta^{-1}\Big\vert \|u_i - v_i\| - \|u_i - v_j\| - \|u_i - v_i\| + \|u_i - v'_j\| \Big\vert \nonumber \\
&= \Delta^{-1}\Big\vert \|u_i - v'_j\| - \|u_i - v_j\| \Big\vert \nonumber \\
&\le \Delta^{-1}\|v_j - v'_j\| \label{eq:lipsc1} \\
&\le \Delta^{-1} \|(u_j,v_j) - (u'_j,v'_j)\| \nonumber \\
&= \Delta^{-1}\|z_j-z'_j\|. \label{eq:lipsc2}
\end{align}
Similarly,
\begin{align}
\lefteqn{\ell_\Delta(\|u_j - v_j\| - \|u_j - v_i\|) - \ell_\Delta(\|u'_j - v'_j\| - \|u'_j - v_i\|)} \nonumber \\
&\le \Delta^{-1} \Big\vert \|u_j - v_j\| - \|u_j - v_i\| - \|u'_j - v'_j\| + \|u'_j - v_i\| \Big\vert \nonumber \\
&\le \Delta^{-1} \left(\Big\vert \|u_j - v_j\| - \|u'_j - v'_j\| \Big\vert + \Big\vert \|u_j - v_i\| - \|u'_j - v_i\| \Big\vert\right) \nonumber \\
&\le \Delta^{-1} \left(\|u_j -u'_j\| + \|v_j - v'_j\| + \|u_j - u'_j\| \right) \nonumber \\
&\le \Delta^{-1} \left(\|u_j -u'_j\| + \|v_j - v'_j\| + \|(u_j,v_j) - (u'_j,v'_j)\| \right) \nonumber \\
&\le (1+\sqrt{2})\Delta^{-1} \|(u_j,v_j) -(u'_j,v'_j)\| \label{eq:lipsc3} \\
&= (1+\sqrt{2})\Delta^{-1} \|z_j-z'_j\|,
\label{eq:lipsc4}
\end{align}
where the last inequality follows from the arithmetic mean--quadratic mean inequality. Combining \eqref{eq:lipsc0}, \eqref{eq:lipsc2} and \eqref{eq:lipsc4}, we obtain that 
\begin{equation}\label{eq:Mbound}
M(A_\Delta) = \max_{j \in [l]} \sup_{\substack{\bz \in H_{\cF}(\mathcal{Q})^{l} \\ z'_j \in H_{\cF}(\mathcal{Q}):~z'_j\neq z_j}} \frac{D^j_{z'_j} A(\bz)}{\|z_j - z'_j\|} \leq \frac{2+\sqrt{2}}{\Delta l}. 
\end{equation}
Using that $\|u_j-u'_j\|$ and $\|v_j - v'_j\|$ are both bounded by $2\sup_{x \in \cX, f\in \cF'} \|f(x)\| \le 2 \alpha B$ for any $\bz,\bz'$, we obtain from \eqref{eq:lipsc0}, \eqref{eq:lipsc1} and \eqref{eq:lipsc3} that
\[
D^j_{z'_j} A_\Delta(\bz) \le \frac{6+4\sqrt{2}}{\Delta l} \sup_{x \in \cX, f\in \cF'} \|f(x)\| \leq \frac{(6+4\sqrt{2}) \alpha B}{\Delta l}.
\]
Hence,
\begin{equation}\label{eq:Kbound}
K(A_\Delta) \le \frac{(6+4\sqrt{2}) \alpha B}{\Delta l}.
\end{equation}
Finally, we bound $J(A_\Delta)$. Let $j,r\in [l]$ with $j\neq r$ and $\bz = (z_1,\dots,z_l) \in H_{\cF}(\mathcal{Q})^{l}$. We have
\begin{align}
\lefteqn{D^r_{z'_r}D^j_{z'_j} A_\Delta(\bz) } \nonumber \\
&= \frac{1}{l(l-1)}\sum_{i\in [l]:~i\neq j}D^r_{z'_r}\left(\ell_\Delta(\|u_i - v_i\| - \|u_i - v_j\|) - \ell_\Delta(\|u_i - v_i\| - \|u_i - v'_j\|)\right) \nonumber \\
&\quad+ \frac{1}{l(l-1)}\sum_{i\in [l]:~i\neq j} D^r_{z'_r}\left(\ell_\Delta(\|u_j - v_j\| - \|u_j - v_i\|) - \ell_\Delta(\|u'_j - v'_j\| - \|u'_j - v_i\|)\right) \nonumber \\
&= \frac{1}{l(l-1)}\cdot D^r_{z'_r}\left(\ell_\Delta(\|u_r - v_r\| - \|u_r - v_j\|) - \ell_\Delta(\|u_r - v_r\| - \|u_r - v'_j\|)\right) \nonumber \\
&\quad+ \frac{1}{l(l-1)}\cdot D^r_{z'_r}\left(\ell_\Delta(\|u_j - v_j\| - \|u_j - v_r\|) - \ell_\Delta(\|u'_j - v'_j\| - \|u'_j - v_r\|)\right), \label{eq:Dr0}
\end{align}
where the second equality follows as all terms are $0$ for which $i \neq r$, since the function in that case is constant with respect to $u_r$ (since $j\neq r$ by our definition). 
Using again that $\ell_\Delta$ is $\Delta^{-1}$-Lipschitz, we have
\begin{align}
\lefteqn{D^r_{z'_r}\left(\ell_\Delta(\|u_r - v_r\| - \|u_r - v_j\|) - \ell_\Delta(\|u_r - v_r\| - \|u_r - v'_j\|)\right)} \nonumber \\
&= \ell_\Delta(\|u_r - v_r\| - \|u_r - v_j\|) - \ell_\Delta(\|u_r - v_r\| - \|u_r - v'_j\|) \nonumber \\
&\quad- \ell_\Delta(\|u'_r - v'_r\| - \|u'_r - v_j\|) + \ell_\Delta(\|u'_r - v'_r\| - \|u'_r - v'_j\|) \nonumber \\
&\le 2\Delta^{-1} \|v_j-v'_j\|. \label{eq:Dr1}
\end{align}
Similarly, 
\begin{align}
\lefteqn{
D^r_{z'_r}\left(\ell_\Delta(\|u_j - v_j\| - \|u_j - v_r\|) - \ell_\Delta(\|u'_j - v'_j\| - \|u'_j - v_r\|)\right)} \nonumber \\
&= \ell_\Delta(\|u_j - v_j\| - \|u_j - v_r\|) - \ell_\Delta(\|u'_j - v'_j\| - \|u'_j - v_r\|) \nonumber \\
&\quad - \ell_\Delta(\|u_j - v_j\| - \|u_j - v'_r\|) + \ell_\Delta(\|u'_j - v'_j\| - \|u'_j - v'_r\|) \nonumber \\
&\le \vert \ell_\Delta(\|u_j - v_j\| - \|u_j - v_r\|) -
\ell_\Delta(\|u'_j - v'_j\| - \|u'_j - v_r\|) \vert \nonumber \\
&\quad + \vert \ell_\Delta(\|u_j - v_j\| - \|u_j - v'_r\|) 
 - \ell_\Delta(\|u'_j - v'_j\| - \|u'_j - v'_r\|) \nonumber \\
&\le 2\Delta^{-1} \left(\|u_j-u'_j\| + \|u_j-u'_j + v'_j-v_j\|\right) \nonumber \\
& \le 2\Delta^{-1} \left(2\|u_j-u'_j\| + \|v_j-v'_j\|\right)
\label{eq:Dr2}
\end{align}
Combining \eqref{eq:Dr0}--\eqref{eq:Dr2} with the arithmetic mean--quadratic mean inequality, we obtain
\[
\frac{D^r_{z'_r}D^j_{z'_j} A_\Delta(\bz)}{\|z_j-z'_j\|}
\leq \frac{4\left(\|u_j-u'_j\| + \|v_j-v'_j\|\right)}{l(l-1)\Delta \|z_j-z'_j\|}
\leq \frac{4\sqrt{2}}{l(l-1)\Delta},
\]
which gives
\begin{equation}\label{eq:Jbound}
J(A_\Delta) \leq \frac{4\sqrt{2}}{\Delta(l-1)}.    
\end{equation}
Substituting \eqref{eq:Mbound}, \eqref{eq:Kbound}, and \eqref{eq:Jbound} into \eqref{eq:maurer2}, we get that with probability at least $1-\delta'$ over the selection of $\hcU$ for any $f \in \cF'$ with $\sup_{f \in \cF'} \cC(f) \le \alpha$,
\begin{equation}\label{eq:maurer3}
\begin{aligned}
\E_{\hcU}\left[A_\Delta(H_{f}(\hcU))\right] ~&\le~ A_\Delta(H_{f}(\hcU)) +  \frac{8(3+\sqrt{2}) \sqrt{\pi\log(l)} \cdot \E_{\hcU}[R(H_{\cF'}(\hcU))]}{\Delta (l-1)}\\
&\quad + \frac{(6+4\sqrt{2}) \alpha B}{\Delta \sqrt{l}} \sqrt{\log\left(\fr{1}{\delta'}\right)}.
\end{aligned}
\end{equation}
To be able to adapt to the complexity $\cC(f)$ of $f$ and prove a bound for all $\Delta$, we apply the above to $\cF'=\cF^\alpha := \{f \in \cF: \cC(f) \le \alpha\}$ with positive integers $\alpha$ and for an exponential grid for $\Delta$ defined as $\Delta_t = 2^{-t}B$ for any integer $t$ and $\delta'=\delta_{t,\alpha}$ where 
$\delta_{t,\alpha}:=\fr{\delta}{3\alpha(\alpha+1)\vert t\vert (\vert t \vert +1)}$ for all $t \neq 0$ and 
$\delta_{0,\alpha}:=\fr{\delta}{3\alpha(\alpha+1)}$. Note that by definition, 
$\sum_{t=-\infty}^\infty\sum_{\alpha=1}^\infty \delta_{t,\alpha} = \delta$.

\paragraph{Bounding $\E_{\hcU}[R(H_{\cF^\alpha}(\hcU))]$.} As a next step, we upper bound $\E_{\hcU}[R(H_{\cF^\alpha}(\hcU))]$. Let $\epsilon$ be the collection of i.i.d.\ Rademacher random variables $\epsilon^1_{cj}$ and $\epsilon^2_{cj}$ for $c\in [l]$ and $j\in [p]$. Then
\begin{align}
\lefteqn{\E_{\hcU}[R(H_{\cF^\alpha}(\hcU))]} \nonumber \\
&= \E_{\hcU}\E_{\epsilon} \left[\sup_{f \in \cF^\alpha} \sum^{l}_{c=1} \sum^{p}_{j=1} \left(f(x_c)_j \epsilon^1_{cj} + \mu_f(\hS_c)_j \epsilon^2_{cj}\right)\right] \nonumber \\
&\leq \E_{\hcU}\E_{\epsilon} \left[\sup_{f \in \cF^\alpha} \sum^{p}_{j=1}\sum^{l}_{c=1}  f(x_c)_j \epsilon^1_{cj}\right]
+ \E_{\hcU}\E_{\epsilon} \left[\sup_{f \in \cF^\alpha} \sum^{p}_{j=1} \sum^{l}_{c=1} \mu_f(\hS_c)_j \epsilon^2_{cj}\right] \nonumber \\
&= \E_{\hcU}\E_{\epsilon} \left[\sup_{f \in \cF^\alpha} \sum^{p}_{j=1}\sum^{l}_{c=1}  f(x_c)_j \epsilon^1_{cj}\right] 
+ \E_{\hcU}\E_{\epsilon} \left[\sup_{f \in \cF^\alpha} \sum^{p}_{j=1}\sum^{l}_{c=1} \E_{x_c \sim U[\hS_c]} f(x_c)_j \epsilon^2_{cj}\right] \nonumber \\
&\leq \E_{\hcU}\E_{\epsilon} \left[\sup_{f \in \cF^\alpha} \sum^{p}_{j=1}\sum^{l}_{c=1}  f(x_c)_j \epsilon^1_{cj}\right] 
+\E_{\hcU}\E_{x_c \sim U[\hS_c]}\E_{\epsilon} \left[\sup_{f \in \cF^\alpha} \sum^{p}_{j=1}\sum^{l}_{c=1} f(x_c)_j \epsilon^2_{cj}\right] \nonumber \\
&= 2\E_{x_1,\dots,x_l}\left[R(\cF^\alpha(\{x_c\}^{l}_{c=1}))\right]. \label{eq:rad_bound_thm0}
\end{align}
Due to the homogenity of the ReLU nonlinearities, we can rescale the layers of a ReLU network so that all layer matrix have norm at most $1$ except for the last layer, that is, any $f \in \cF^\alpha$ can be written in a canonical form as $f=W^q \tilde{f}$ with $\|W^q\|_F \le \alpha$ and some $\tilde{f} \in \tilde\cF := \{\sigma(W^{q-1} (\dots W^2\sigma(W^1x) \dots ) \mid \forall i \leq q-1:~\|W^i\|_F\leq 1\}$. Denoting the $j$th row of $W^q$ by $W^q_{(j)}$, for any $\{x_1,\dots,x_l\} \subset \mathcal{X}$ we can bound the Rademacher complexity as
\begin{align*}
R(\cF^\alpha(\{x_c\}^{l}_{c=1})
&= \E_{\epsilon} \left[\sup_{f \in \cF^\alpha} \sum^{p}_{j=1}\sum^{l}_{c=1} f(x_c)_j \epsilon_{cj}\right] \\
&= \E_{\epsilon} \left[\sup_{\substack{W^q:~\|W^q\|_F \leq \alpha\\ \tilde{f} \in \tilde\cF}} \sum^{p}_{j=1}\sum^{l}_{c=1} W^q_{(j)} \tilde{f}(x_{c}) \epsilon_{cj}\right] \nonumber\\
&\leq \alpha \E_{\epsilon} \left[\sup_{g \in \cG} \max_{j \in [p]} \left\|\sum^{l}_{c=1} \tilde{f}(x_{c}) \epsilon_{cj}\right\|\right] \nonumber\\
&\leq \frac{\alpha}{\lambda} \log\left( \E_{\epsilon} \left[\max_{j \in [p]}\sup_{\tilde{f} \in \tilde\cF} \exp \left(\lambda \left\|\sum^{l}_{c=1} \tilde{f}(x_{c}) \epsilon_{cj}\right\|\right)\right]\right) \nonumber \\
&\leq \frac{\alpha}{\lambda} \log\left( \E_{\epsilon} \left[\sum_{j \in [p]}\sup_{\tilde{f} \in \tilde\cF} \exp \left(\lambda \left\|\sum^{l}_{c=1} \tilde{f}(x_{c}) \epsilon_{cj}\right\|\right)\right]\right) \nonumber \\
&= \frac{\alpha}{\lambda} \log\left(p\cdot \E_{\epsilon} \left[\sup_{\tilde{f} \in \tilde\cF} \exp \left(\lambda \left\|\sum^{l}_{c=1} \tilde{f}(x_{c}) \epsilon_{c1}\right\|\right)\right]\right), \nonumber
\end{align*}
where (i) the first inequality follows by moving all the norms of $W^q$ to the row maximizing $W^q_{(j)}\sum_{c=1}^l W^q_{(j)} \tilde{f}(x_{c}) \epsilon_{cj}$ and applying the Cauchy-Schwartz inequality; (ii) the second inequality follows for arbitrary $\lambda>0$ (to be chosen later) by Jensen's inequality and flipping the order of the $\sup$ and the $\max$; (iii) the thrid inequality follows by upper bounding the maximum over $j$ by the sum; while (iv) the last equality holds by exchanging the expectation and the summation over $j$, and noticing that the resulting expectation is the same for all $j$.
Finally, by the proof of Theorem~1 of \cite{golowich}, when choosing $\lambda = \fr{\sqrt{2(q\log(2)+\log(p))}}{\sqrt{\sum^{l}_{c=1}\|x_c\|^2}}$, we have
\begin{equation}\label{eq:RFalpha}
\begin{aligned}
R\left(\cF^\alpha(\{x_c\}^{l}_{c=1})\right) &\le \alpha \sqrt{2l(q\log(2)+\log(p))} \cdot \max_{c\in [l]} \|x_c\|~.
\end{aligned}
\end{equation}
Plugging into \eqref{eq:rad_bound_thm0} (and using that $B=\sup_{x\in\cX}\|x\|$ by definition) gives
\begin{equation}\label{eq:rad_bound_thm}
\E_{\hcU}[R(H_{\cF^\alpha}(\hcU))] \leq 2\alpha B \sqrt{2l(q\log(2)+\log(p))}.
\end{equation} 
Substituting this into \eqref{eq:maurer3} and using that $l-1 \ge l/2$ for $l \ge 2$ and the union bound, we obtain that with probability at least $1-\delta$, for all integers $t$, positive integers $\alpha$, and $f \in \cF^\alpha$, 
\begin{align*}
\lefteqn{\E_{\hcU}\left[A_{\Delta_t}(H_{f}(\hcU))\right]}\\
&\le A_{\Delta_t}(H_{f}(\hcU)) +  
\frac{\alpha B}{\Delta_t \sqrt{l}}
\left(355 \sqrt{\log(l)(q\log(2)+\log(p))} + (6\!+\!4\sqrt{2}) \sqrt{\log\!\left(\!\frac{3(\alpha+1)^2(|t|+1)^2}{\delta}\right)}\right).
\end{align*}
Now, for any $\Delta>0$ there exists a $t$ such that $\Delta \le \Delta_t \le 2\Delta$; then clearly $|t| \le |\log (\fr{\Delta}{B})| + 1$. Since $\ell_\Delta$ and $A_\Delta$ monotonically increase in $\Delta$, the above inequality with $\alpha=\lceil \cC(f) \rceil$ implies that with probability at least $1-\delta$ over the choice of $\hcU$, for all $f \in \cF$ we have
\begin{small}
\begin{equation}\label{eq:CfBound}
\begin{aligned}
\lefteqn{\E_{\hcU}\left[A_\Delta(H_{f}(\hcU))\right]}\\
&\le A_{2\Delta}(H_{f}(\hcU)) +
\frac{\lceil \cC(f) \rceil B}{\Delta \sqrt{l}}
\left(355 \sqrt{\log(l)(q\log(2)+\log(p))} + 16.5\sqrt{\log\left(\frac{\sqrt{3}(\lceil \cC(f) \rceil+1)(|\log (\fr{\Delta}{B})|+2)}{\sqrt\delta}\right)}\right).
\end{aligned}
\end{equation}
\end{small}
Let $\cZ_\delta$ denote the event that the above inequality (Eq.~\ref{eq:CfBound}) does not hold; then we have that $\Pr[\cZ_\delta] \le \delta$.
Let $p_{\delta^2,\tcP}:=\Pr[\cZ_{\delta^2}|\tcP]$ denote the conditional probability that \eqref{eq:CfBound} holds with $\delta^2$ in place of $\delta$ given the random choice of $\tcP$. Then, 
\[
\E_{\tcP} [p_{\delta^2,\tcP}] 
= \Pr[\cZ_{\delta^2}] \leq \delta^2
\]
and hence by Markov's inequality,
\[
\Pr[p_{\delta^2,\tcP} \ge \delta]
\leq \frac{\E_{\tcP} [p_{\delta^2,\tcP}]}{\delta} 
\leq \delta.
\]
That is, with probability at least $1-\delta$ over the choice of $\tcP$, with probability at least $1-\delta^2$ over the choice of $\hcU$ (given $\tcP$), for all $f \in \cF$ we have
\begin{small}
\begin{equation}\label{eq:CfBound2}
\begin{aligned}
\lefteqn{\E_{\hcU}\left[A_\Delta(H_{f}(\hcU))\right]}\\
~&\le~ A_{2\Delta}(H_{f}(\hcU)) +
\frac{\lceil \cC(f) \rceil B}{\Delta \sqrt{l}}
\left(355 \sqrt{\log(l)(q\log(2)+\log(p))} + 16.5\sqrt{\log\left(\frac{\sqrt{3}(\lceil \cC(f) \rceil+1)(|\log (\fr{\Delta}{B})|+2)}{\delta}\right)}\right).
\end{aligned}
\end{equation}
\end{small}
Taking expectation with respect to $\hcU$ given $\tcP$, we have
\begin{align*}
\E_{\hcU}[A_{2\Delta}(H_{f}(\hcU)) \mid \tcP] 
~&=~ \Avg_{i\neq j\in [l]}\E_{\hcU}\left[\ell_{2\Delta}(\|f(x_i) - \mu_{f}(\hS_j)\| - \|f(x_i) - \mu_{f}(\hS_i)\|) \mid \tcP\right] \\
~&=~ \Avg_{i\neq j\in [l]}\E_{\hS_i,\hS_j}[\ell_{2\Delta}(\|f(x_i) - \mu_{f}(\hS_j)\| - \|f(x_i) - \mu_{f}(\hS_i)\|) \mid \tcP] \\
~&=~ \Avg_{i\neq j\in [l]}\E_{\hS_i,\hS_j}[\ell_{2\Delta}(f;\tP_i,\tP_j)] \\
\end{align*}
Combining with \eqref{eq:CfBound2} and also that $A_\Delta(H_{f}(\hcU)) \le 1$ for all $\hcU$ yields
that with probability at least $1-\delta$ over the choice of $\tcP$, for all $f \in \cF$
\begin{small}
\begin{align*}
\E_{\hcU}\left[A_\Delta(H_{f}(\hcU))\right] 
~&\le~ \Avg_{i\neq j\in [l]}\E_{\hS_i,\hS_j}[\ell_{2\Delta}(f;\tP_i,\tP_j)] +  \delta^2 \\
&\quad +
\frac{\lceil \cC(f) \rceil B}{\Delta \sqrt{l}}
\left(355 \sqrt{\log(l)(q\log(2)+\log(p))} + 16.5\sqrt{\log\left(\fr{\sqrt{3}(\lceil \cC(f) \rceil+1)(|\log (\fr{\Delta}{B})|+2)}{\delta}\right)}\right).
\end{align*}
\end{small}
The statement of the lemma now follows by combining this inequality with \eqref{eq:LD2}.
\end{proof}

\section{Proof of Lemma~\ref{prop:marginBoundB}}\label{app:marginBoundB}

\marginBoundB*

\begin{proof}
As in the proof of Lemma~\ref{prop:marginBoundA}, we proceed by considering feature maps $f$ belonging to the function class $\cF^\alpha := \{f \in \cF: \cC(f) \le \alpha\}$ for  a sequence of $\alpha$ growing to infinity. For now, fix $\alpha>0$, $f \in \cF^\alpha$, and $\Delta>0$.
Our plan is to rewrite $\ell_{\Delta}(f;\tS_i,\tS_j)$ as an average so that we can apply our extension \eqref{eq:maurer-ext} of Corollary~3 of~\cite{Maurer2019UniformCA} to get a generalization bound. 

To this end, consider a vector $\bz=(\bz_i,\bz_j) \in \cF(\cX)^m \times \cF(\cX)^m$ with $\bz_r=(z_{r1},\ldots,z_{rm})$ for $r \in \{i,j\}$ (representing $m$ samples from classes $i$ and $j$, respectively), and multi-sets $I_i,I_j \subset [m]$ of size $n$ representing the indices of $n$ samples selected from $\bz_i$ and $\bz_j$ without replacement. For $r \in \{i,j\}$, let $\mu_{r,I_r}(\bz)=\Avg_{b \in I_r} [z_{rb}]$ denote the mean of the selected subsamples, and for any $I=(I_i,I_j)$ and $b \in [m]$, denote the soft-margin loss of $z_{ib}$ of the nearest neighbor classifier defined by $\mu_{i,I_i}(\bz)$ and $\mu_{j,I_j}(\bz)$ as  
\[
Q_{\Delta}(\bz;I,b) := \ell_\Delta(\|z_{ib} - \mu_{i,I_i}(\bz)\|-\|z_{ib}-\mu_{j,I_j}(\bz)\|).    
\]
Finally, define the aggregation function $B_\Delta:\cF(\cX)^{2m} \to \R$ as
\[
B_\Delta(\bz) := B_\Delta(\bz_i,\bz_j) := \Avg_{b, I}[Q_{\Delta}(\bz; I,b)],
\]
where in the average $b$ ranges over $[m]$ and $I=(I_i,I_j)$ over any pairs of multi-sets $I_i,I_j \subset [m]$ of cardinality $n$.
Then clearly
\begin{equation}
\label{eq:l4-B=ell}
B_\Delta(f(\tS_i), f(\tS_j)) = \ell_{\Delta}(f;\tS_i,\tS_j).
\end{equation}
Denoting $\tS_i=\{\tx_{i1},\ldots,\tx_{im}\}$
and $\tS_j=\{\tx_{j1},\ldots,\tx_{jm}\}$, we can think of $(f(\tS_i),f(\tS_j))$ as a collection of $m$ i.i.d.\ vectors $(f(\tx_{ib}),f(\tx_{jb}))_{b \in [m]}$, hence
by extension \eqref{eq:maurer-ext} of Corollary~3 of~\cite{Maurer2019UniformCA}, for any $\cF' \subset \cF^\alpha$ with $\alpha=\sup_{f \in \cF'} \cC(f)$ and any $\delta' \in (0,1)$, with probability at least $1-\delta'$ over the selection of $\tS_i$ and $\tS_j$ we have
\begin{equation}
\label{eq:maurer-l4}
\begin{aligned}
&\E_{\tS_i,\tS_j}[B_\Delta(f(\tS_i), f(\tS_j))]\\ &\le B_\Delta(f(\tS_i), f(\tS_j))  + 2\sqrt{2\pi\log(m)} \big(2M(B_\Delta) + J(B_\Delta)\big)\cdot \E_{\tS_i,\tS_j}[R(\cF'(\tS_i\cup \tS_j))] \\
& \qquad + K(B_\Delta)\cdot \sqrt{m\log\left(\fr{1}{\delta'}\right)}.
\end{aligned}
\end{equation}
Next we explore how $\E_{\tS_i,\tS_j}[B_\Delta(f(\tS_i), f(\tS_j))]$ and $\ell_{\Delta}(f;\tP_i,\tP_j)$ are related.
\begin{equation*}
\begin{aligned}
&\E_{\tS_i, \tS_j}[B_\Delta(f(\tS_i), f(\tS_j))]\\ 
~&=~ \E_{\tS_i, \tS_j}[\ell_{\Delta}(f;\tS_i,\tS_j)] \\
~&=~ \E_{\tS_i, \tS_j}\E_{\bar{S}_i \sim U[\tS_i]^n}\E_{\bar{S}_j \sim U[\tS_j]^n}\E_{\tx_i \sim U[\tS_i]}[\ell_{\Delta}(\|f(\tx_i) - \mu_f(\bar{S}_i)\| - \|f(\tx_i) -\mu_f(\bar{S}_j)\|)] \\
~&=~ \E_{\tS_i, \tS_j}\E_{\bar{S}_i \sim U[\tS_i]^n}\E_{\bar{S}_j \sim U[\tS_j]^n}\E_{b \sim U[m]}[\ell_{\Delta}(\|f(\tx_{ib}) - \mu_f(\bar{S}_i)\| - \|f(\tx_{ib}) -\mu_f(\bar{S}_j)\|)]. \\
\end{aligned}    
\end{equation*}
Note that by replacing $\tx_i \sim U[\tS_i]$ with $\tx_i \sim \tP_i$, the above expression would be equal to 
\begin{equation*}
\ell_{\Delta}(f;\tP_i,\tP_j) = \E_{\hS_i \sim \tP^n_i}\E_{\hS_j \sim \tP^n_j}\E_{\tx_i\sim \tP_i}[\ell_{\Delta}(\|f(\tx_i) - \mu_f(\hS_i)\| - \|f(\tx_i) -\mu_f(\hS_j)\|)],    
\end{equation*}
which is the term that we would like to bound. The problem is that because both $\tx_i$ and $\bar{S}_i$ are selected from $\tS_i$, they are not independent. However, since $n \ll m$, it is likely that $\tx_i$ is not an element of $\bar{S}_i$; to put it differently,  the index $b$ of $\tx_i$ (i.e., $\tx_i=\tx_{ib}$) does not belong to the multi-set of indices $I_i$ defining $\bar{S}_i$ (i.e., $\bar{S}_i=\{\tx_{ib'}: b' \in I_i\}$. Under this assumption, $\tx_i$ and $\bar{S}_i$ become independent with $\tx_i$ following distribution $\tP_i$.
Since $I_i$ contains at most $n$ different indices, we indeed have $\Pr[b \notin I_i]\ge 1-\fr{n}{m}$. These, together with the nonnegativity of $\ell_\Delta$ imply that
\begin{align}
&\E_{\tS_i,\tS_j}[B_\Delta(f(\tS_i), f(\tS_j))] \nonumber \\
~&\ge~ 
\E_{\tS_i,\tS_j}\E_{\bar{S}_i,\bar{S}_j}\E_{b \sim U[m]}[\ell_{\Delta}(\|f(\tx_{ib}) - \mu_f(\bar{S}_i)\| - \|f(\tx_{ib})  -\mu_f(\bar{S}_j)\|) \mid b \notin I_i] \cdot \Pr[b \notin I_i] \nonumber \\
~&\ge~ 
\E_{\tS_i,\tS_j}\E_{\bar{S}_i,\bar{S}_j}\E_{b \sim U[m]}[\ell_{\Delta}(\|f(\tx_{ib}) - \mu_f(\bar{S}_i)\| - \|f(\tx_{ib})  -\mu_f(\bar{S}_j)\|) \mid b \notin I_i] \cdot \left(1-\fr{n}{m}\right) \nonumber \\
~&=~ 
\E_{\tS_i,\tS_j}\E_{\bar{S}_i,\bar{S}_j}\E_{b \sim U[[m]\setminus I_i]}[\ell_{\Delta}(\|f(\tx_{ib}) - \mu_f(\bar{S}_i)\| - \|f(\tx_{ib})  -\mu_f(\bar{S}_j)\|)] \cdot \left(1-\fr{n}{m}\right) \nonumber \\
~&=~ \E_{\tS_i,\tS_j}\E_{\bar{S}_i,\bar{S}_j}\E_{\tx_i\sim \tP_i}[\ell_{\Delta}(\|f(\tx_{i}) - \mu_f(\bar{S}_i)\| - \|f(\tx_{i}) -\mu_f(\bar{S}_j)\|)]\left(1-\fr{n}{m}\right). \label{eq:BDelta-lower1}
\end{align}
Finally, we would like to replace the distributions of $\bar{S}_i$, and $\bar{S}_j$ above from $U[\tS_i]^n$ and $U[\tS_j]^n$ to $\tP_i^n$ and $\tP_j^n$, respectively. Note that sampling $\bar{S}_i\sim \tP^n_j$ is equivalent to sampling $\tS_i \sim \tP^m_i$ and then $\bar{S}_i\sim U[\tS_i]^n$ conditioned on the event that the indices of the samples we choose to put in $\bar{S}_i$ are unique. This event happens with probability $\fr{m(m-1)\cdots(m-n+1)}{m^n}$, 
and hence
\begin{align*}
& \E_{\tS_i,\tS_j}\E_{\bar{S}_i \sim U[\tS_i]^n,\bar{S}_j \sim U[\tS_j]^n}\E_{\tx_i\sim \tP_i}[\ell_{\Delta}(\|f(\tx_{i}) - \mu_f(\bar{S}_i)\| - \|f(\tx_{i}) -\mu_f(\bar{S}_j)\|)] \\
~&\ge~ 
\E_{\tS_i,\tS_j}\E_{\bar{S}_i \sim \tP_i^n,\bar{S}_j \sim \tP_j^n}\E_{\tx_i\sim \tP_i}[\ell_{\Delta}(\|f(\tx_{i}) - \mu_f(\bar{S}_i)\| - \|f(\tx_{i}) -\mu_f(\bar{S}_j)\|)]
\cdot \fr{m(m-1)\cdots(m-n+1)}{m^n} \\
~&=~\E_{\bar{S}_i \sim \tP_i^n}\E_{\bar{S}_j \sim \tP_j^n}\E_{\tx_i\sim \tP_i}[\ell_{\Delta}(\|f(\tx_{i}) - \mu_f(\bar{S}_i)\| - \|f(\tx_{i}) -\mu_f(\bar{S}_j)\|)]
\cdot \fr{m(m-1)\cdots(m-n+1)}{m^n} \\
~&\ge~ \fr{1}{4} \ell_{\Delta}(f;\tP_i,\tP_j),
\end{align*}
where in the last step we used that
$\fr{m(m-1)\cdots(m-n+1)}{m^n} \ge 2/9$ which is easy to see: for $m \ge 4$, we have
\[
\fr{m(m-1)\cdots(m-n+1)}{m^n} = \left(1-\fr{1}{m}\right)\cdots\left(1-\fr{n-1}{m}\right)
\ge \left(1-\fr{1}{\sqrt{m}}\right)^{\sqrt{m}} \ge \fr{1}{4}
\]
because $n \le \sqrt{m}$, and the bound follows by inspection for $m \in \{1,2,3\}$.

Combining with \eqref{eq:l4-B=ell}, \eqref{eq:maurer-l4}, and \eqref{eq:BDelta-lower1}, we obtain
\begin{equation}
\begin{aligned}
\ell_{\Delta}(f;\tP_i,\tP_j) ~&\le~ \fr{4m}{m-n} \cdot \ell_{\Delta}(f;\tS_i,\tS_j) \\
&\quad+ \frac{8m\sqrt{2\pi\log(2m)}}{m-n} \big(2M(B_\Delta) + J(B_\Delta)\big)\cdot \E_{\tS_i,\tS_j}[R(\cF'(\tS_i\cup \tS_j)))] \\
&\quad+ \frac{4m}{m-n} K(B_\Delta)\cdot \sqrt{2m\log(\fr{1}{\delta'})}.
\label{eq:marginboundB-middle}
\end{aligned}
\end{equation}

It remains to bound $M(B_{\Delta})$, $J(B_{\Delta})$ and $K(B_{\Delta})$.

\paragraph{Bounding $K(B_{\Delta})$ and $M(B_{\Delta})$.} Let $z_{rt} \in \R^{p}$ for $r \in \{i,j\}$, $t\in [m]$ with $\bz_r = (z_{r1},\dots,z_{rm})$ and $\bz'_r = (z'_{r1},\dots,z'_{rm})$, where $z_{rb} = z'_{rb}$ for all $b\neq t$. By definition, for $\bz=(\bz_i,\bz_j)$ and $\bz'=(\bz'_i,\bz'_j)$,
\begin{align}
D^{t}_{(z'_{it},z'_{jt})}B_\Delta(\bz) &= B_\Delta(\bz)-B_\Delta(\bz')
= \Avg_{b,I}\left[Q_{\Delta}(\bz;I,b) - Q_{\Delta}(\bz';I,b)\right].
\label{eq:D_rt}
\end{align}
Note that the actual organization of the variables in $B_\Delta(\bz)$ is according to pairs $(z_{ib},z_{jb})_{b \in [m]}$, so in the partial difference operator above we consider changing a pair.
To bound \eqref{eq:D_rt}, we consider the difference of the $Q_\Delta$ functions above. Using that $\ell_\Delta$ is $\Delta^{-1}$-Lipschitz, the triangle inequality and that $\bz$ and $\bz'$ only differ in a single coordinate $(r,t)$, we obtain the following:
\begin{align}
&Q_{\Delta}(\bz;I,b) - Q_{\Delta}(\bz';I,b) \nonumber \\
~&=~\ell_\Delta(\|z_{ib} - \mu_{i,I_i}(\bz)\|-\|z_{ib}-\mu_{j,I_j}(\bz)\|) - \ell_\Delta(\|z'_{ib} - \mu_{i,I_i}(\bz')\|-\|z'_{ib}-\mu_{j,I_j}(\bz')\|) \nonumber \\
~&\le~ \Delta^{-1}\left\vert \|z_{ib} - \mu_{i,I_i}(\bz)\| - \|z_{ib}-\mu_{j,I_j}(\bz)\| - \|z'_{ib} - \mu_{i,I_i}(\bz')\| + \|z'_{ib}-\mu_{j,I_j}(\bz')\| \right\vert \nonumber \\
~&\le~ \Delta^{-1}\left(\left\vert \|z_{ib} - \mu_{i,I_i}(\bz)\| - \|z'_{ib} - \mu_{i,I_i}(\bz')\| \right\vert + \left\vert \|z_{ib}-\mu_{j,I_j}(\bz)\| - \|z'_{ib}-\mu_{j,I_j}(\bz')\| \right\vert \right) \nonumber \\
~&\le~ \Delta^{-1} \left(\|z_{ib} - z'_{ib}\| + \| \mu_{i,I_i}(\bz) - \mu_{i,I_i}(\bz')\| + \|z_{ib}- z'_{ib}\| + \| \mu_{j,I_j}(\bz) -\mu_{j,I_j}(\bz')\|\right) \nonumber \\
~&=~ \Delta^{-1}\left(2\|z_{ib} - z'_{ib}\| + \|\mu_{i,I_i}(\bz) - \mu_{i,I_i}(\bz')\| + \|\mu_{j,I_j}(\bz) - \mu_{j,I_j}(\bz')\|\right) \nonumber \\
~&=~ 
\frac{\|z_{it} - z'_{it}\|}{\Delta}\left(
2 \bI[b=t] 
+\frac{\bI[t \in I_i] }{n} \right)
+\frac{\|z_{jt} - z'_{jt}\|\cdot \bI[t \in I_j]}{\Delta n} 
\label{eq:Ib1}  \\
~&\le~
\frac{\|(z_{it},z_{jt}) - (z'_{it},z'_{jt})\|}{\Delta}
\left(2 \bI[b=t] 
+\frac{\bI[t \in I_i] }{n} +\frac{\bI[t \in I_j]}{n}\right), 
\label{eq:Ib}  
\end{align}
Now trivially
\begin{align}
\label{eq:B-ave1}
\Avg_{b,I}\Big[\bI[b=t]\Big]
&= \Avg_{b}\Big[\bI[b=t]\Big]
= \frac{1}{m},
\end{align}
and for $r \in \{i,j\}$,
\begin{align}
\label{eq:B-ave2}
\Avg_{b,I}\Big[\bI[t \in I_r]\Big]
& = \Avg_{I_r} \Big[\bI[t \in I_r]\Big]
= 1-\frac{(m-1)^n}{m^n} 
\le \frac{n}{m},
\end{align}
where in the last step we used that
$(1-\fr{1}{m})^n \ge 1-\fr{n}{m}$ for $n \le m$.

Combining with \eqref{eq:D_rt} and \eqref{eq:Ib} implies
$\frac{D^{t}_{(z'_{it},z'_{jt})}B_\Delta(\bz)}{\|(z_{it},z_{jt}) - (z'_{it},z'_{jt})\|} \leq 
\fr{4}{\Delta m}$
and hence
\begin{equation}\label{eq:Mbound2}
M(B_{\Delta}) = \max_{t \in [m]} \sup_{\substack{\bz_i,\bz_j\in \cF(\cX)^{m} \\ (z'_{it},z'_{jt}) \neq (z_{it},z_{jt})}} \frac{D^{t}_{(z'_{it},z'_{jt})}B_\Delta(\bz)}{\|(z_{it},z_{jt}) - (z'_{it},z'_{jt})\|}
\leq \frac{4}{\Delta m}.
\end{equation}

In addition, since $f \in \cF^\alpha$, 
$\|z_{rt}-z'_{rt}\| \le 2 \sup_{x \in \cX} \|f(x)\| \le 2 \alpha B$ for any $r \in \{i,j\}$, \eqref{eq:D_rt}, \eqref{eq:Ib1}, \eqref{eq:B-ave1}, and \eqref{eq:B-ave2} imply
which gives
\begin{equation}\label{eq:Kbound2}
K(B_\Delta) = \max_{t \in [m]} \sup_{\substack{\bz_i,\bz_j\in \cF(\cX)^{m} \\ (z'_{it},z'_{jt}) \neq (z_{it},z_{jt})}}
D^{t}_{(z'_{it},z'_{jt})} B_\Delta(\bz) \le \frac{8\alpha B}{\Delta m}.
\end{equation}

\paragraph{Bounding $J(B_{\Delta})$.} Consider $t_1\neq t_2 \in [m]$ and for $r \in \{i,j\}$, let 
$\bz_r = (z_{r1},\dots,z_{rm})$, $\bz'_r = (z'_{r1},\dots,z'_{rm})$ where $z'_{rb} = z_{rb}$ for all $b \neq t_1$, and $\bz''_{r} = (z''_{r1},\dots,z''_{rm})$ where $z''_{rb} = z'_{rb}$ for all $b \neq t_2$. Finally, let $\bar{\bz}_r = (\bar{z}_{r1},\dots,\bar{z}_{rm})$ where $\bar{z}_{rb} = z_{rb}$ for all $b \neq t_2$ and $\bar{z}_{r t_2}=z''_{r t_2}$.
Then it is easy to see that
\begin{align}
&D^{t_2}_{(z''_{i t_2},z''_{j t_2})}D^{t_1}_{(z'_{i t_2},z'_{j t_2})}B_\Delta(\bz_i,\bz_j) \nonumber \\
~&=~ B_\Delta(\bz_i,\bz_j)-B_\Delta(\bz'_i,\bz'_j) +B_\Delta(\bz''_i,\bz''_j) - B_\Delta(\bar{\bz}_i,\bar{\bz}_j) \nonumber \\
~&=~ \Avg_{b,I}\left[Q_{\Delta}(\bz;I,b) - Q_{\Delta}(\bz';I,b) 
+ Q_{\Delta}(\bz'';I,b) - Q_{\Delta}(\bar{\bz};I,b) \right]
\label{eq:D2B}
\end{align}
Now consider the terms inside the average.
First notice that the first and second, and respectively the third and fourth terms only differ in the $t_1$th coordinate of the variables, that is, the differences are nonzero only if $b=t_1$ or $t_1 \in I_i \cup I_j$. Similarly, the first and fourth, respectively the second and third terms differ in the $t_2$th coordinate, so again their difference is nonzero only if $b=t_2$ or $t_2 \in I_i \cup I_j$. Therefore,
\begin{align*}
&Q_{\Delta}(\bz;I,b) - Q_{\Delta}(\bz';I,b) 
+ Q_{\Delta}(\bz'';I,b) - Q_{\Delta}(\bar{\bz};I,b) \\
~&=~ 
\left(Q_{\Delta}(\bz;I,b) - Q_{\Delta}(\bz';I,b) 
+ Q_{\Delta}(\bz'';I,b) - Q_{\Delta}(\bar{\bz};I,b) \right) \\
& \qquad \times \left(\bI[b=t_1]+\bI[t_1 \in I_i \cup I_j]\right) \cdot \left(\bI[b=t_2]+\bI[t_2 \in I_i \cup I_j]\right)
\end{align*}
Similarly to \eqref{eq:Ib}, we obtain that
\begin{align*}
&Q_{\Delta}(\bz;I,b) - Q_{\Delta}(\bz';I,b) \\
~&\le~
\frac{\|(z_{i t_1},z_{j t_1}) - (z'_{i t_1},z'_{j t_1})\|}{\Delta}
\bigg(2 \bI[b=t_1] 
+\frac{\bI[t_1 \in I_i] }{n} +\frac{\bI[t_1 \in I_j]}{n}\bigg)
\Big(\bI[b=t_2]+\bI[t_2 \in I_i \cup I_j]\Big) \\
~&\le~ \frac{\|(z_{i t_1},z_{j t_1}) - (z'_{i t_1},z'_{j t_1})\|}{\Delta}\bigg(2 \bI[b=t_1] \big(\bI[t_2 \in I_i]+\bI[t_2 \in I_j]\big)
+ \bI[b=t_2]\frac{\bI[t_1 \in I_i]+\bI[t_1 \in I_j]}{n}\\
&\qquad\qquad\qquad\qquad\qquad\qquad\qquad\qquad\qquad+ \big(\bI[t_2 \in I_i]+\bI[t_2 \in I_j]\big)\frac{\bI[t_1 \in I_i]+\bI[t_1 \in I_j]}{n}\bigg),
\end{align*}
where in the last step we used that $t_1\neq t_2$ and hence
$\bI[b=t_1]\bI[b=t_2]=0$. As in \eqref{eq:B-ave1} and \eqref{eq:B-ave2}, for $v \in \{1,2\}$ and $r \in \{i,j\}$
\begin{align*}
\Avg_{b,I}\Big[\bI[b=t_v]\Big] = \frac{1}{m} 
\quad \text{and} \quad
\Avg_{b,I}\Big[\bI[t_v \in I_r]\Big] \le \frac{n}{m}.
\end{align*}
Finally,
\begin{align*}
&\Avg_{b,I}\Big[\bI[t_1 \in I_r]\bI[t_2 \in I_r]\Big]
= \Avg_{I_r}\Big[\bI[t_1,t_2 \in I_r]\Big]  \\
~&=~ 1 - \Avg_{I_r}\Big[\bI[t_1,t_2 \notin I_r]
+ \bI[t_1 \in I_r, t_2 \notin I_r] + \bI[t_1 \notin I_r, t_2 \in I_r]\Big] \\
~&=~ 1-\frac{(m-2)^n + 2 ((m-1)^n - (m-2)^n)}{m^n} 
\frac{m^n -2(m-1)^n + (m-2)^n}{m^n}
\le \frac{2n(n-1)}{m^2}
\end{align*}
where the last step holds since
\begin{align*}
m^n -2(m-1)^n + (m-2)^n &= \int_{m-1}^m n u^{n-1} \text{d}u - \int_{m-2}^{m-1}n u^{n-1} \text{d}u \le n\left(m^{n-1} - (m-2)^{n-1}\right) \\
&= n(n-1) \int_{m-2}^m u^{n-2} \text{d}u
\le 2n(n-1)m^{n-2}.
\end{align*}
Therefore, since we can independently average over the different variables ($b,I_1,I_2$), we obtain
\begin{align*}
\Avg_{b,I}\Big[
Q_{\Delta}(\bz;I,b) - Q_{\Delta}(\bz';I,b)\Big] 
&\le \frac{\|(z_{i t_1},z_{j t_1}) - (z'_{i t_1},z'_{j t_1})\|}{\Delta}
\left(\frac{2}{m^2} + \frac{6n}{m^2} +\frac{4n(n-1)}{m^2}\right)
\end{align*}
and the same bound holds for 
$\Avg_{b,I}\Big[
Q_{\Delta}(\bz'';I,b) - Q_{\Delta}(\bar{\bz};I,b)\Big]$.
Substituting into \eqref{eq:D2B}, we obtain
\begin{align*}
D^{t_2}_{(z''_{i t_2},z''_{j t_2})}D^{t_1}_{(z'_{i t_2},z'_{j t_2})}B_\Delta(\bz_i,\bz_j)& \le   \frac{\|(z_{i t_1},z_{j t_1}) - (z'_{i t_1},z'_{j t_1})\|}{\Delta}
\left(\frac{2}{m^2} + \frac{6n}{m^2} +\frac{4n(n-1)}{m^2}\right).
\end{align*}
Pairing the $Q_\Delta$ functions the other way around, we get the same bound but with the norm
$\|(z_{i t_2},z_{j t_2}) - (z''_{i t_1},z''_{j t_1})\|$,
\begin{align*}
D^{t_2}_{(z''_{i t_2},z''_{j t_2})}D^{t_1}_{(z'_{i t_2},z'_{j t_2})}B_\Delta(\bz_i,\bz_j)& \le   \frac{\|(z_{i t_2},z_{j t_2}) - (z''_{i t_1},z''_{j t_1})\|}{\Delta}
\left(\frac{2}{m^2} + \frac{6n}{m^2} +\frac{4n(n-1)}{m^2}\right).
\end{align*}
Averaging the above two bounds and using that
\[
\|(z_{i t_1},z_{j t_1}) - (z'_{i t_1},z'_{j t_1})\|+ \|(z_{i t_2},z_{j t_2}) - (z''_{i t_1},z''_{j t_1})\| \le \sqrt{2}\|\bz - \bz''\|,
\]
by the definition of $J(B_\Delta)$ we get
\begin{equation}
\label{eq:Jbound2}
J(B_\Delta) ~\le~ \frac{4n^2 + 2n + 2}{\sqrt{2}\Delta m} \le \frac{(2n+1)^2}{\sqrt{2}\Delta m}.
\end{equation}

\paragraph{Bounding $\E_{\tS_i,\tS_j}[R(\cF^{\alpha}(\tS_i\cup \tS_j)))]$.} Similarly to \eqref{eq:RFalpha}, 
\begin{equation}\label{eq:rad_bound_SiSj}
\E_{\tS_i,\tS_j}[R(\cF^{\alpha}(\tS_i\cup \tS_j)))] ~\le~ 2\alpha B \sqrt{m(q\log(2)+\log(p))}
\end{equation}
as $f(\tS_i\cup \tS_j)$ is a collection of $2m$ $p$-dimensional vectors with norm at most $B$.

Substituting the bounds
\eqref{eq:Mbound2}, \eqref{eq:Kbound2}, \eqref{eq:Jbound2} and \eqref{eq:rad_bound_SiSj} into
\eqref{eq:marginboundB-middle}, we conclude that with probability at least $1-\delta'$ over the selection of $\tS_i$ and $\tS_j$, for any $f \in \cF'$ with $\sup_{f \in \cF'} \cC(f) \le \alpha$,
\begin{equation}
\begin{aligned}
\label{eq:marginboundB-middle2}
\ell_{\Delta}(f;\tP_i,\tP_j) ~&\le~ 4\left(1+\frac{n}{m-n}\right) \cdot \ell_{\Delta}(f;\tS_i,\tS_j) 
\\
&\qquad+ 
16 \alpha B\frac{\sqrt{m}(2n+4)^2}{(m-n)\Delta}
\sqrt{\pi(q\log(2)+\log(p))\log(2m)} \\
&\qquad+ 
32 \alpha B \frac{\sqrt{m}}{(m-n)\Delta}\sqrt{2 \log(\fr{1}{\delta'})}
\end{aligned}
\end{equation}

To finish the proof, we extend the above bound to all $\Delta>0$ and all $f \in \cF$ with $\alpha$ replaced with the complexity $\cC(f)$. We proceed exactly as in the proof of Lemma~\ref{prop:marginBoundA}: We apply \eqref{eq:marginboundB-middle2} to the function class $\cF'=\cF^\alpha := \{f \in \cF: \cC(f) \le \alpha\}$ with positive integers $\alpha$ and for an exponential grid for $\Delta$ defined as $\Delta_t = 2^{-t}B$ for any integer $t$ and $\delta'=\delta_{t,\alpha}$ where 
$\delta_{t,\alpha}:=\fr{\delta}{3\alpha(\alpha+1)\vert t\vert (\vert t \vert +1)}$ for all $t \neq 0$ and 
$\delta_{0,\alpha}:=\fr{\delta}{3\alpha(\alpha+1)}$. Note that by definition, 
$\sum_{t=-\infty}^\infty\sum_{\alpha=1}^\infty \delta_{t,\alpha} = \delta$. Therefore, by the union bound, with probability at least $1-\delta$ over the selection of $\tS_i$ and $\tS_j$, for any $f \in \cF^{\alpha}$,
\eqref{eq:marginboundB-middle2} holds for $\Delta=\Delta_t$ and $\delta'=\delta_{t,\alpha}$ simultaneously for all integers $t$ and positive integers $\alpha$.

The final bound of the lemma follows for any $\Delta>0$ and $f\in\cF$ by instantiating this bound with $\alpha=\lceil \cC(f) \rceil$ and $t$ such that $\Delta \le \Delta_t \le 2\Delta$ (which always exists by the definition of the grid $\{\Delta_t\}$), and then 
applying the bound $\ell_{\Delta}(f;\tP_i,\tP_j) \le \ell_{\Delta_t}(f;\tP_i,\tP_j) \le \ell_{2\Delta}(f;\tP_i,\tP_j)$ which follows from the monotonicity of $\ell_\Delta$, bounding $1/\Delta_t$ by $1/\Delta$, and bounding $\delta_{t,\alpha}$ using that $|t| \le |\log (\fr{\Delta}{B})| + 1$ (which again holds by the definition of $\Delta_t$).
\end{proof}

%%%%%%%%%%%%%%%%%%%%%%%%%%%%%%%%%%%%%%%%%%%%%%%%%%%%%%%%%%%%%%%
\section{Proofs of Lemmas~\ref{prop:cdnv_error}-\ref{prop:cdnv_errorG}}\label{sec:error_analysis}
%%%%%%%%%%%%%%%%%%%%%%%%%%%%%%%%%%%%%%%%%%%%%%%%%%%%%%%%%%%%%%%
\def\dij{\Lambda_{ij}}
\def\dijhat{\hat \Lambda_{ij}}

\cdnvError*

\begin{proof} For $c \in \{i,j\}$, let $\hS_c \sim Q^n_c$ and $x_c \sim Q_c$,
$u_c = f(x_c)$, $\mu_c = \mu_f(Q_c)$, $\hat\mu_c = \mu_f(\hS_c)$ and $\sigma^2_c = \Var(\hat u_c)$ and let $s$ be a variable that we can choose to be $p$ if $\{f \circ Q_i, f\circ Q_j\}$ are spherically symmetric distributions and $s=1$ otherwise. Let $\dij:=\|\mu_i-\mu_j\|$ denote the considered centers' true distance, which sets the overall scale,
and $\dijhat:=\|\hat\mu_i-\hat\mu_j\|$ its estimate. Let $\gamma_1,\gamma_2,\gamma_3,\gamma_4 \in (0,1)$ be constants that satisfy $\gamma_4 - 2\gamma_3 - 4\gamma_1 - 2\gamma^2_2/s > 0$ (to be selected by the end of the proof). Moreover, let $\alpha_{ij}:=\dij^2\gamma_1/\Delta$, where $\Delta\leq \dij\gamma_2/\sqrt{s}$. Since $\bI[r < \Delta]\leq \ell_{\Delta}(-r)$, we have
\begin{align*}
\ell_{\Delta}(f;Q_i,Q_j) ~&=~ \E_{\hS_i,\hS_j}\E_{x_i}[\ell_{\Delta}(\|u_i - \hat\mu_i \| \le \|u_i - \hat\mu_j \|)]  \\
~&\le~ \Pr\left[\|u_i - \hat\mu_j \| \le \|u_i - \hat\mu_i \| + \Delta\right].
\end{align*}

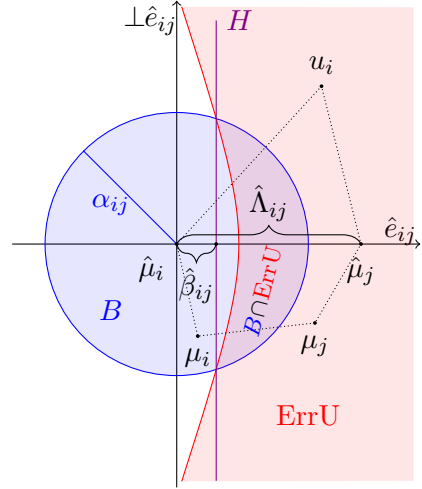
\begin{wrapfigure}{r}{0.4\textwidth}
\usetikzlibrary{decorations.pathreplacing}
\usetikzlibrary{decorations.pathreplacing}
\begin{tikzpicture}[scale=1.75,dot/.style={circle,inner sep=0.5pt,fill,label={#1},name=#1},
extended line/.style={shorten >=-#1,shorten <=-#1},
extended line/.default=1cm]
\pgfmathsetmacro{\a}{1}  
\pgfmathsetmacro{\b}{0.3}
\pgfmathsetmacro{\d}{1.4}
\pgfmathsetmacro{\D}{sqrt(\a*\a-2*\b*\d+\d*\d)-\a} 
% Axes
\draw[->] (-1.25,0) -- (1.85,0) coordinate (x axis);
\draw[->] (0,-1.85) -- (0,1.85) coordinate (y axis);
\node at (1.7,0.1) {$\hat{e}_{ij}$};
\node at (-0.2,1.7) {$\perp\!\!\hat{e}_{ij}$};
% Ball B
\draw[blue, fill=blue, fill opacity=0.1] (0,0) circle (\a);
\node at (-0.5*\a,-0.5*\a) {\color{blue}$B$};
\draw[blue,-] (0,0) -- (-0.71*\a,0.71*\a);
\node at (-0.5*\a,0.3*\a) {\color{blue}$\alpha_{ij}$};
% ErrU region
\draw[draw opacity=0,fill=red,fill opacity=0.1] plot[domain=-88.8:88.8] ({(\d*\d-\D*\D)/2/(\d*cos(\x)+\D)*cos(\x)},{(\d*\d-\D*\D)/2/(\d*cos(\x)+\D)*sin(\x)}) -- plot(1.8,1.8) -- plot(1.8,-1.8);
\draw[draw=red] plot[domain=-88.8:88.8] ({(\d*\d-\D*\D)/2/(\d*cos(\x)+\D)*cos(\x)},{(\d*\d-\D*\D)/2/(\d*cos(\x)+\D)*sin(\x)});
\node at (1,-1.3) {\color{red}ErrU};
% Points and Distances
\node[rotate=75] at (0.65,-0.36) {\footnotesize ${\color{blue}B}\!\cap\! {\color{red}\textnormal{ErrU}}$};
\draw[violet] (\b,-1.8) -- (\b,1.7) node[right] {\color{violet}$H$};
\draw[densely dotted] (0,0) node[dot,label=below left:$\hat{\mu}_i$]{} 
             -- (0.16,-0.7) node[dot,label=below:$\mu_i$]{} 
             -- (1.05,-0.6) node[dot,label=below:$\mu_j$]{} 
             -- (\d,0) node[dot,label=below:$\hat\mu_j$]{}
             -- (1.1,1.2) node[dot,label=above:$u_i$]{} -- (0,0);
\draw[decorate,decoration={brace,amplitude=7pt,mirror}] (\d,0) -- (0,0) node[pos=0.5,above=5pt]{$\dijhat$};
\draw[decorate,decoration={brace,amplitude=7pt,mirror}] (0,0) -- (\b,0) node[pos=0.5,below=5pt]{$\hat\beta_{ij}$};
\node[dot] at (\b,0) {};
\end{tikzpicture}
\caption{A geometric illustration of the various components in the proof of Lemma \ref{prop:cdnv_error}. The blue section denotes the ball $B$ of radius $\alpha_{ij}$ and the red section denotes the set $\textnormal{ErrU}$.}
\label{fig:illustration}
\vspace{-1em}
\end{wrapfigure}
The proof idea is graphically illustrated in Figure \ref{fig:illustration}: Consider the error event 
$\text{Err}:=[\|u_i-\hat\mu_j\| \le \|u_i-\hat\mu_i\| + \Delta]$. For $\Delta=0$, $\text{Err}$ is a $p$-dimensional half-space orthogonal to the unit vector $\hat e_{ij} := (\hat\mu_j-\hat\mu_i)/\dijhat$, and to determine its probability,
we can simply project all quantities onto the line $\mathbb{R}\cdot \hat e_{ij}$. This makes the problem one-dimensional and improves the bound by a factor of $p^{-1}$ in the spherically symmetric case. 

However, one complication is that $\hat e_{ij}$ itself is random, so in order to achieve a scale of $p^{-1}$, 
one has to carefully avoid using triangle inequalities involving $u_i$, which complicates the analysis, and also gives an extra $1/n$ term.

Another complicated arises from assuming $\Delta>0$: 
Consider the error set $\text{ErrU}(\hat\mu_i,\hat\mu_i):=\{u_i:\|u_i-\hat\mu_j\| \le \|u_i-\hat\mu_i\| + \Delta\}$.
Its boundary $\partial\text{ErrU}$ is now (one sheet of an) elliptic hyperboloid, and direct projection onto $\hat e_{ij}$ does not work anymore. If we confine $u_i$ to the hyperball $B:=\{u_i:\|u_i-\hat\mu_i\| \le \alpha_{ij}\}$, 
then $B$ cuts $\text{ErrU}$ into a finite moon-shaped set, and the projection method works again,
while the probability of $u_i\not\in B$ scales with $1/p$.

We derive a sequence of bounds (A)-(H) which we chain all together at the end.

\paragraph{(A)} 
As a first step we determine the halfspace (given $\hat \mu_i$ and $\hat \mu_j$) that covers the moon $B\cap\text{ErrU}$.
Assume that $\|u_i-\hat\mu_j\| \le \|u_i-\hat\mu_i\| + \Delta$ and $\|u_i-\hat\mu_i\| \le \alpha_{ij}$. Then,
\begin{align*}
\|u_i-\hat\mu_i\|^2 + 2\alpha_{ij} \Delta + \Delta^2 ~&\geq~ \|u_i-\hat\mu_i\|^2 + 2\|u_i-\hat\mu_i\| \cdot \Delta + \Delta^2 \\
~&=~ (\|u_i-\hat\mu_i\|+\Delta)^2\\
~&\geq~ \|u_i-\hat\mu_j\|^2 \\
~&=~ \|(u_i-\hat\mu_i) - (\hat\mu_j-\hat\mu_i)\|^2 \\
~&=~ \|u_i-\hat\mu_i\|^2 - 2(u_i-\hat\mu_i)^\top (\hat\mu_j-\hat\mu_i) + \dijhat^2.
\end{align*}
Therefore, we have $2 (u_i-\hat\mu_i)^\top (\hat\mu_j-\hat\mu_i) ~\geq~ \dijhat^2 - 2\alpha_{ij} \Delta - \Delta^2$, 
which can also be written as follows
\begin{align*}
(u_i-\hat\mu_i)^\top \hat e_{ij} ~\geq~ \hat\beta_{ij} ~:=~ \fr12\dijhat - \frac{2\alpha_{ij} \Delta + \Delta^2}{2\dijhat}.
\end{align*}
This is the halfspace $H$ that covers $B\cap\text{ErrU}$ most tightly.

\paragraph{(B)} 
We can use this to bound the probability of the event $\text{Err}$ in terms of $\alpha_{ij}$ and $\Delta$ by conditioning:
\begin{align*}
  \Pr[\text{Err}] ~&=~ \Pr[\|u_i-\hat\mu_j\| \le \|u_i-\hat\mu_i\| + \Delta] \\
  ~&=~ \Pr[\|u_i-\hat\mu_j\| \le \|u_i-\hat\mu_i\| + \Delta ~\land~ \|u_i-\hat\mu_i\| \le \alpha_{ij}] \\
  ~&~~~ + \Pr[\|u_i-\hat\mu_j\| \le \|u_i-\hat\mu_i\| + \Delta ~\land~ \|u_i-\hat\mu_i\| > \alpha_{ij}] \\
  ~&\smash{\stackrel{(A)}\le}~ \Pr[(u_i-\hat\mu_i)^\top \hat e_{ij} \geq \hat\beta_{ij}] ~+~ \Pr[\|u_i-\hat\mu_i\| > \alpha_{ij}] ~=:~ (C) ~+~ (F).
\end{align*}

\paragraph{(C)} 
Next we upper bound the term $\Pr[(u_i-\hat\mu_i)^\top \hat e_{ij} \geq \hat\beta_{ij}]$. Namely, 
\begin{align*}
~&\Pr[(u_i-\hat\mu_i)^\top \hat e_{ij} \geq \hat\beta_{ij}] \\
 &=~    \Pr[(u_i-\mu_i)^\top \hat e_{ij} + (\mu_i-\hat\mu_i)^\top \hat e_{ij} \geq \hat\beta_{ij}] \\
~&\le~ \Pr[(u_i-\mu_i)^\top \hat e_{ij} + (\mu_i-\hat\mu_i)^\top \hat e_{ij} \geq \gamma_3\dij ~\lor~ \hat\beta_{ij}\le \gamma_3\dij ] \\
~&\le~ \Pr[(u_i-\mu_i)^\top \hat e_{ij} + (\mu_i-\hat\mu_i)^\top \hat e_{ij} \geq \gamma_3\dij  ] 
        ~+~ \Pr[\hat\beta_{ij}\le \gamma_3\dij ] ~=:~ (D) ~+~ (H),
\end{align*}
for $\gamma_3 \in (0,1)$.

\paragraph{(D)}
By applying the union bound and use that a projection onto $\hat e_{ij}$ decreases the length of the vector $\mu_i-\hat\mu_i$, we get
\begin{align*}
(D) ~&=~ \Pr\left[ (u_i-\mu_i)^\top \hat e_{ij} + (\mu_i-\hat\mu_i)^\top \hat e_{ij} \geq \gamma_3\dij \right] \\
~&\le~ \Pr\left[ (u_i-\mu_i)^\top \hat e_{ij} \geq \gamma_3\dij /2 ~\lor~ (\mu_i-\hat\mu_i)^\top \hat e_{ij} \geq \gamma_3\dij/2 \right] \\
~&\le~ \Pr\left[(u_i-\mu_i)^\top \hat e_{ij} \geq \gamma_3\dij/2\right] + \Pr\left[ \|\mu_i-\hat\mu_i\| \geq \gamma_3\dij/2 \right] ~=:~ (E) ~+~ (G).
\end{align*}

\paragraph{(E)} Next we analyze each term separately. For a fixed pair $\hat\mu_i$ and $\hat\mu_j$, the variance of the random variable $(u_i-\mu_i)^\top \hat e_{ij}$ is at most $\Var(u_i)/s$, where we have exploited the projection in the spherically symmetric case. Therefore, by Markov's inequality,
\begin{align*}
\Pr\left[(u_i-\mu_i)^\top \hat e_{ij} \geq \dij\gamma_3/2\right] 
~&\le~ \E_{\hat\mu_i,\hat\mu_j}\left[\textstyle{\Pr_{u_i}}\left[(u_i-\mu_i)^\top \hat e_{ij} 
\geq \dij\gamma_3/2\right]\right]\\
~&\le~ \E_{\hat\mu_i,\hat\mu_j}\left[\frac{4\gamma_3^{-2}\Var(u_i)}{s \cdot \dij^2}\right] ~=~ \frac{4\Var(u_i)}{s \cdot \gamma^{2}_3\dij^2}~.
\end{align*}

\paragraph{(F)} Again, by Markov's inequality,
and using $\Delta \le \dij\gamma_2/\sqrt{s}$, hence $\alpha_{ij} = \dij^2\gamma_1/\Delta \geq \dij\sqrt{s}(\gamma_1/\gamma_2)$, we have
\begin{align*} 
\Pr\left[\|u_i-\hat\mu_i\| \geq \alpha_{ij}\right] ~&\le~ \Pr\left[\|\hat\mu_i-\mu_i\| \geq \alpha_{ij}/2 \right] + \Pr\left[\|u_i-\mu_i\| \geq \alpha_{ij}/2 \right] \\
~&\le~ \frac{4\Var(\hat\mu_i)}{\alpha^2_{ij}} + \frac{4\Var(u_i)}{\alpha^2_{ij}}\\
~&\le~ \frac{4\gamma^2_2\Var(\hat\mu_i)}{\gamma^2_1 \cdot s \cdot \dij^2} + \frac{4\gamma^2_2\Var(u_i)}{\gamma^2_1\cdot s \cdot \dij^2}~,
\end{align*}

\paragraph{(G)} and
\begin{equation*} 
\Pr\left[\|\hat\mu_i-\mu_i\| \geq \gamma_3\dij/2\right] \leq \frac{4\Var(\hat\mu_i)}{\gamma^2_3\dij^2} = \frac{4\Var(u_i)}{n\gamma^2_3\dij^2}~.
\end{equation*}

\paragraph{(H)} 
Next, by conditioning,
\begin{align*}
\Pr[\hat\beta_{ij}\le \gamma_3\dij] 
~&=~ \Pr[\hat\beta_{ij}\le \gamma_3\dij ~\land~ \dijhat \le \gamma_4\dij ]
+ \Pr[\hat\beta_{ij}\le \gamma_3\dij ~\land~ \dijhat > \gamma_4\dij ] \\
~&\le~ \Pr[\dijhat \le \gamma_4\dij]
+ \Pr[\hat\beta_{ij}\le \gamma_3\dij ~\land~ \dijhat > \gamma_4\dij ] \\
~&\le~ \Pr[\|\mu_j-\hat\mu_j\| + \|\mu_i-\hat\mu_i\| \geq (1-\gamma_4)\dij ]
+ \Pr[\hat\beta_{ij}\le \gamma_3\dij ~\land~ \dijhat > \gamma_4\dij ],
\end{align*}
where the last inequality follows from solving the triangle inequality
\begin{align*}
  \dij \leq \|\mu_i-\hat\mu_i\| + \dijhat + \|\hat\mu_j-\mu_j\|.
\end{align*}
We can bound the first term by
\begin{align*}
&\Pr[\|\mu_j-\hat\mu_j\| + \|\mu_i-\hat\mu_i\| \geq (1-\gamma_4)\dij]  \\
~&\le~ \Pr[\|\mu_j-\hat\mu_j\| \geq \gamma_4\dij/2] 
+ \Pr[\|\mu_i-\hat\mu_i\| \geq (1-\gamma_4)\dij/2] \\
~&\le~ \frac{4(\Var(\hat\mu_j)+\Var(\hat\mu_i))}{(1-\gamma_4)^2\dij^2}. 
\end{align*}
We notice that if $\hat\beta_{ij}\le \gamma_3\dij$ and $\dijhat>\gamma_4\dij$, then
\begin{align*}
\gamma_3\dij ~\geq~ \hat\beta_{ij} ~&>~ \frac{\gamma_4\dij}{2} - \frac{2\alpha_{ij} \Delta + \Delta^2}{\dij} \\
~&\geq~ \frac{\gamma_4\dij}{2} - \frac{2\gamma_1\dij^2 + \gamma^2_2 \dij^2/p}{\dij} \\
~&\geq~ \dij (\gamma_4/2 - 2\gamma_1 - \gamma^2_2/p),
\end{align*}
which is a contradiction to $\gamma_4 - 2\gamma_3 - 4\gamma_1 - 2\gamma^2_2/s > 0$ for appropriate choices of $\gamma_1,\dots,\gamma_4 \in (0,1)$. Hence 
\begin{align*}
\Pr[\hat\beta_{ij}\le \gamma_3\dij] ~&\le~ \frac{4(\Var(\hat\mu_j)+\Var(\hat\mu_i))}{(1-\gamma_4)^2\dij^2}
\end{align*}

\paragraph{Combining everything together} 
$(B)\le(C)+(F)\le(D)+(H)+(F)\le(G)+(E)+(F)+(H)$
and using $\Var(\hat\mu_i) = \Var(u_i)/n$ we finally obtain
\begin{align*}
\Pr[\text{Err}] ~&\le~ \frac{4\Var(u_i)}{n\gamma^2_3\dij^2} + \frac{4\Var(u_i)}{s \cdot \gamma^{2}_3\dij^2}
+ \frac{4\gamma^2_2\Var(u_i)}{n\gamma^2_1\cdot s \cdot \dij^2} + \frac{4\gamma^2_2\Var(u_i)}{\gamma^2_1\cdot s \cdot \dij^2} +\frac{4(\Var(u_j)+\Var(u_i))}{n(1-\gamma_4)^2\dij^2}.
\end{align*}
Finally, by choosing $\gamma_1=\gamma_2=0.1$, $\gamma_3=0.18$ and $\gamma_4=0.8$ (satisfying $\gamma_4 - 2\gamma_3 - 4\gamma_1 - 2\gamma^2_2/s > 0$), we obtain
\begin{align*}
\Pr[\text{Err}] ~&\le~ \frac{124\Var(u_i)}{n\dij^2} + \frac{124\Var(u_i)}{s \dij^2}
+ \frac{4\Var(u_i)}{n\dij^2} + \frac{4\Var(u_i)}{s \dij^2} +\frac{100(\Var(u_j)+\Var(u_i))}{n\dij^2} \\
~&=~ \frac{128\Var(u_i)}{n\dij^2} + \frac{128\Var(u_i)}{s \dij^2}
+\frac{100(\Var(u_j)+\Var(u_i))}{n\dij^2}. \\
\end{align*}
\end{proof} 

\cdnvErrorE*

\begin{proof} For $c \in \{i,j\}$, let $\hS_c \sim Q^n_c$ and $x_c \sim Q_c$, $u_c = f(x_c)$, $\mu_c = \mu_f(Q_c)$, $\hat\mu_c = \mu_f(\hS_c)$ and $\dij = \|\mu_i-\mu_j\|$. Since $\bI[r < \Delta]\leq \ell_{\Delta}(-r)$, we have
\begin{equation*}
\begin{aligned}
\ell_{\Delta}(f;Q_i,Q_j) = \E_{\hS_i,\hS_j}\E_{x_i}[\ell_{\Delta}(\|u_i - \hat\mu_i \| - \|u_i - \hat\mu_j \|)] \leq \Pr\left[\|u_i - \hat\mu_j \| \le \|u_i - \hat\mu_i \| + \Delta\right].
\end{aligned}
\end{equation*}
Next, by the triangle inequality and the union bound,
\begin{equation*}
\begin{aligned}
&\Pr\left[\|u_i - \hat\mu_j \| \le \|u_i - \hat\mu_i \| + \Delta\right] \\
~&\le~ \Pr\left[\|u_i-\mu_j\| \le \|u_i-\mu_i\| + \|\hat\mu_i-\mu_i\| + \|\hat\mu_j-\mu_j\| + \Delta\right] \\
~&\le~ \Pr\left[\|u_i-\mu_j\| \le \|u_i-\mu_i\| + 2\Delta ~\lor~ \|\hat\mu_i-\mu_i\| \ge 0.5\Delta ~\lor~ \|\hat\mu_j-\mu_j\| \ge 0.5\Delta\right]\\ 
~&\le~ \Pr\left[\|u_i-\mu_j\| \le \|u_i-\mu_i\| + 2\Delta\right] + \Pr\left[\|\hat\mu_i-\mu_i\| \ge 0.5\Delta\right] + \Pr\left[\|\hat\mu_j-\mu_j\| \ge 0.5\Delta\right] \\
~&=~ E_{2\Delta}(f;Q_i,Q_j) + \Pr\left[\|\hat\mu_i-\mu_i\| \ge 0.5\Delta\right] + \Pr\left[\|\hat\mu_j-\mu_j\| \ge 0.5\Delta\right].
\end{aligned}
\end{equation*}
By Markov's inequality,
\begin{equation*}
 \Pr\left[\|\hat\mu_i-\mu_i\| \ge 0.5\Delta\right] \leq \frac{4\Var(\hat\mu_i)}{\Delta^2} = \frac{4\Var(u_i)}{n^2\Delta^2} = \frac{\|\mu_i-\mu_j\|^2}{\Delta^2} \cdot \frac{4\Var(u_i)}{n^2\|\mu_i-\mu_j\|^2}~,
\end{equation*}
and similarly
\begin{equation*}
\begin{aligned}
\Pr\left[\|\hat\mu_j-\mu_j\| \ge 0.5\Delta\right] \leq \frac{\|\mu_i-\mu_j\|^2}{\Delta^2} \cdot \frac{4\Var(u_j)}{n^2\|\mu_i-\mu_j\|^2}~.
\end{aligned}
\end{equation*}
In particular, we obtain that
\begin{equation*}
\ell_{\Delta}(f;Q_i,Q_j) \leq E_{2\Delta}(f;Q_i,Q_j) +\fr{4\|\mu_i-\mu_j\|^2}{n^2 \Delta^2}\cdot (V_f(Q_i,Q_j)+V_f(Q_j,Q_i)).
\end{equation*}

Finally, we want to bound $E_{2\Delta}(f;Q_i,Q_j)$:
\paragraph{(A)} 
Assume that $\|u_i-\mu_j\| \le \|u_i-\mu_i\| + 2\Delta$ and $\|u_i-\mu_i\| \le \alpha_{ij}$. Then,
\begin{align*}
\|u_i-\mu_i\|^2 + 4\alpha_{ij} \Delta + 4\Delta^2 ~&\geq~ \|u_i-\mu_i\|^2 + 2\|u_i-\mu_i\| \cdot 2\Delta + 4\Delta^2 \\
~&=~ (\|u_i-\mu_i\|+2\Delta)^2\\
~&\geq~ \|u_i-\mu_j\|^2 \\
~&=~ \|(u_i-\mu_i) - (\mu_j-\mu_i)\|^2 \\
~&=~ \|u_i-\mu_i\|^2 - 2(u_i-\mu_i)^\top (\mu_j-\mu_i) + \dij^2.
\end{align*}
Therefore, we have $2 (u_i-\mu_i)^\top (\mu_j-\mu_i) ~\geq~ \dij^2 - 4\alpha_{ij} \Delta - 4\Delta^2$, 
which can also be written as follows
\begin{align*}
(u_i-\mu_i)^\top e_{ij} ~\geq~ \beta_{ij} ~:=~ \fr12\dij - \frac{2(\alpha_{ij} \Delta + \Delta^2)}{\dij},
\end{align*}
where $e_{ij} := (\mu_j-\mu_i)/\|\mu_j-\mu_i\|$.

\paragraph{(B)} 
We can use this to bound the probability of the event $\text{Err}$ in terms of $\alpha_{ij}$ and $\Delta$ by conditioning:
\begin{align*}
E_{2\Delta}(f;Q_i,Q_j) ~&=~ \Pr[\|u_i-\mu_j\| \le \|u_i-\mu_i\| + 2\Delta] \\
  ~&=~ \Pr[\|u_i-\mu_j\| \le \|u_i-\mu_i\| + 2\Delta ~\land~ \|u_i-\mu_i\| \le \alpha_{ij}] \\
  ~&\quad + \Pr[\|u_i-\mu_j\| \le \|u_i-\mu_i\| + 2\Delta ~\land~ \|u_i-\mu_i\| > \alpha_{ij}] \\
  ~&\smash{\stackrel{(A)}\le}~ \Pr[(u_i-\mu_i)^\top e_{ij} \geq \beta_{ij}] + \Pr[\|u_i-\mu_i\| > \alpha_{ij}] ~=:~ (C) ~+~ (D).
\end{align*}

\paragraph{(C,D)} 
Next we upper bound the terms $\Pr[(u_i-\mu_i)^\top e_{ij} \geq \beta_{ij}]$ and $\Pr[\|u_i-\mu_i\| > \alpha_{ij}]$. Namely, by Markov's inequality,
\begin{align*}
\Pr[(u_i-\mu_i)^\top e_{ij} \geq \beta_{ij}] \leq \frac{\Var(u_i)}{s \cdot \beta^2_{ij}}~,
\end{align*}
and
\begin{align*}
\Pr[\|u_i-\mu_i\| > \alpha_{ij}] \leq \frac{\Var(u_i)}{\alpha^2_{ij}}~.
\end{align*}
We choose $\alpha_{ij} = \fr{1}{8}\dij^2/\Delta$, denote $\gamma=\Delta/\dij$ and assume $\Delta \leq \dij/4$ ($\gamma\leq 1/4$). In particular, $\alpha_{ij}=\dij/(8\gamma)$ and $\beta_{ij}\geq (\fr{1}{4}-2\gamma^2)\dij \geq \dij/8$. Hence,
\begin{equation*}
\begin{aligned}
E_{2\Delta}(f;Q_i,Q_j) ~&\leq~ \frac{\Var(u_i)}{s \cdot (\fr{1}{4}-2\gamma^2)^2 \dij^2} + \frac{64\gamma^2\Var(u_i)}{\dij^2} \\
~&\leq~ \frac{64\Var(u_i)}{s \dij^2} + \frac{64\gamma^2\Var(u_i)}{\dij^2} \\
~&\leq~ 64(\fr{1}{s}+\gamma^2) \cdot V_f(Q_i,Q_j) \\
~&\leq~ 64(\fr{1}{s}+\fr{\Delta^2}{\dij^2}) \cdot V_f(Q_i,Q_j).
\end{aligned}
\end{equation*}
\paragraph{Combining everything together.} We obtain that
\begin{equation*}
\begin{aligned}
\ell_{\Delta}(f;Q_i,Q_j) ~&\le~ E_{2\Delta}(f;Q_i,Q_j) +\fr{4\|\mu_i-\mu_j\|^2}{n^2 \Delta^2}\cdot (V_f(Q_i,Q_j)+V_f(Q_j,Q_i)) \\
~&\le~ 64(\fr{1}{s}+\fr{\Delta^2}{\|\mu_i-\mu_j\|^2}) \cdot V_f(Q_i,Q_j) +\fr{4\|\mu_i-\mu_j\|^2}{n^2 \Delta^2}\cdot (V_f(Q_i,Q_j)+V_f(Q_j,Q_i)).
\end{aligned}
\end{equation*}
\end{proof}

\cdnvErrorG*

\begin{proof}
Let $\hS_c \sim Q^n_c$ and $x_c \sim Q_c$, $u_c = f(x_c)$, $\mu_c = \mu_f(Q_c)$, $\hat\mu_c = \mu_f(\hS_c)$ and $\sigma^2_c = \Var(u_c)$. Let $\dij:=\|\mu_i-\mu_j\|$ be the considered centers' distance, which sets the overall scale. By the triangle inequality,
\begin{equation}\label{eq:exp1}
\begin{aligned}
\Pr\left[\|u_i - \hat\mu_j \| \le \|u_i - \hat\mu_i \| + \Delta\right] ~&\le~ \Pr\left[\dij \le 2\|u_i-\mu_i\| + \|\hat\mu_i-\hat\mu_i\| + \|\hat\mu_j-\mu_j\| + \Delta\right] \\
~&\le~ \Pr\left[(\dij-\Delta)/4 \le \|u_i-\mu_i\|\right]\\ &\quad+ \Pr\left[(\dij-\Delta)/4 \le \|\hat\mu_i-\hat\mu_i\|\right] \\
&\quad+ \Pr\left[(\dij-\Delta)/4 \le \|\hat\mu_j-\mu_j\|\right].
\end{aligned}
\end{equation}
Since $u_c\sim\cN(\mu_c,\sigma_c^2\cdot \mathbb{I}/p)$ and $\hat\mu_c\sim\cN(\mu_c,\sigma_c^2\cdot \mathbb{I}/(n p))$, the random variables $E_i = \frac{\|u_i-\mu_i\|^2 p}{\sigma^2_i}$ and $U_i = \frac{\|\hat\mu_i-\mu_i\|^2 p}{n\sigma^2_i}$ are chi-squared random variables with $p$ components. We can write
\begin{small}
\begin{equation*}
\begin{aligned}
\Pr\left[\|u_i-\mu_i\|\geq \fr{\dij-\Delta}{4}\right] = \Pr\left[E_i \geq \fr{(\dij-\Delta)^2 p}{16\sigma^2_i} \right] \leq 1-F\left(\fr{(\dij-\Delta)^2p}{16\sigma^2_i},p\right) 
= 1-F\left(zp,p\right), 
\end{aligned}
\end{equation*}
\end{small}
where $F(x,p)$ is the CDF of the chi-squared distribution with $p$ components and $z=\left(\frac{\dij-\Delta}{4\sigma_c}\right)^2$. By Chernoff's bound, if $z > 1$, we have
\begin{equation}\label{eq:exp2}
1-F\left(zp,~ p\right) \leq \left(z \exp(1-z)\right)^{p/2} \leq \frac{\exp\left(-p/(32V^{ij}_f)\right)}{(\textnormal{e} \cdot V^{ij}_{f})^{p/2}} 
\end{equation}
Similarly, we have
\begin{equation}\label{eq:exp3}
\Pr\left[\|\hat\mu_i-\mu_i\|\geq \fr{\dij-\Delta}{4}\right] \leq \frac{\exp\left(-np/(32V^{ij}_f)\right)}{(\textnormal{e} \cdot V^{ij}_{f})^{p/2}} \leq \frac{\exp\left(-p/(32V^{ij}_f)\right)}{(\textnormal{e} \cdot V^{ij}_{f})^{p/2}}.
\end{equation}
Finally, by combining \eqref{eq:exp1}-\eqref{eq:exp3} we obtain the desired inequality.
\end{proof}

\newpage
%%%%%%%%%%%%%%%%%%%%%%%%%%%%%%%%%%%%%%%%%%%%%%%%%%%%%%%%%%%%%%%
\section{List of Notation}\label{app:Notation}
%%%%%%%%%%%%%%%%%%%%%%%%%%%%%%%%%%%%%%%%%%%%%%%%%%%%%%%%%%%%%%%

\begin{center}
\begin{tabular}{r c p{10cm} }
\toprule
$\cX$, $\mathcal{Y}_k$ & & instances and labels spaces \\
$\tP$, $P$ & & source and target distributions \\
$\tP_c$, $P_c$ & & source and target class-conditional distributions \\
$\bI$ & & indicator function \\
$L_{\tP}(h)$, $L_P(h)$ & & source and target test errors \\
$L_{\cD}(f)$ & & transfer risk \\
$c \in[l]|[k]$ & & class index \\
$\mathcal{E}$ & & set of class-conditional distributions \\
$\cD$ & & distribution over class-conditional distributions \\
$l \in \mathbb{N}$ & & number of classes in the train=source task \\
$k \in \mathbb{N}$ &  & number of classes in the test=target task \\
$m,n$ & & number of samples per source/target class \\
$p\in\mathbb{N}$ & & features dimension \\
$d \in \mathbb{N}$ & & input dimension \\
$q \in \mathbb{N}$ & & depth of a neural network \\
$f\in \cF$ & & feature map $\mathbb{R}^{d}\to \mathbb{R}^{p}$. Penultimate layer of a neural network \\
$\mathcal{C}(f)$ & & complexity of a neural network \\
$g \in \mathcal{G}$ & & mapping $\mathbb{R}^{p}\to \mathbb{R}^{k}$ from features to soft test labels \\
$\tilde{g} \in \tilde{\mathcal{G}}$ & & mapping $\mathbb{R}^{p}\to \mathbb{R}^{l}$ from features to soft train labels \\
$\lambda_n$ & & regularization parameter $\lambda_n = \alpha \sqrt{n}$\\
$\Delta>0$ & & margin \\
$B>0$ & & maximal norm of $x\in \mathcal{X}$ \\
$\ell_{\Delta}$ & & the soft-margin loss function \\
$S,S_c$ & & target training data and its subset consisting of all samples from class $c$ \\
$\tS, \tS_c$ & & source training data and its subset consisting of all samples from class $c$\\
$\mu_f(P)$ & & mean of $f(x)$ for $x\sim P$ \\
$\Var_f(P)$ & & variance of $f(x)$ for $x\sim P$ \\
$V_f(Q_1,Q_2)$ & & class-distance normalized variance (CDNV)\\
$\Lambda>0$ & & minimal distance between the embeddings means of two different source classes \\
$E_\Delta(f;Q_1,Q_2)$ & & the expected (one sided) $\Delta$-margin NCC classification error of $f$ \\
$\ell_\Delta(f;Q_1,Q_2)$ & & the expected (one sided) soft-margin NCC classification error of $f$ \\
$R(A)$ & & Rademacher complexity of $A\subset\R^d$ \\
$\Avg^{k}_{i=1}[a_i]$ & & average \\
$u(A)$ & & the projection of $u:B\to R$ over $A\subset B$ \\
$\cU(A)$ & & the set $\{u(A) : u \in \cU\}$ \\
\bottomrule
\end{tabular}
\end{center}
\label{tab:TableOfNotationForMyResearch}

\section*{Acknowledgements}

The authors would like to thank Ilja Kuzborskij, Razvan Pascanu, Tomaso Poggio, and Csaba Szepesv\'ari  for illuminating discussions during the preparation of this manuscript, and Miruna P\^islar for her priceless technical support. During parts of this work, Tomer Galanti was a Research Scientist Intern at Deepmind, and then was supported by the Center for Minds, Brains and Machines (CBMM), funded by NSF STC award CCF-1231216.

\vskip 0.2in
\bibliography{journal_version_arxiv}

\end{document}